\documentclass{article}

\usepackage{microtype}
\usepackage{graphicx}
\usepackage{subfigure}
\usepackage{booktabs} % for professional tables

\usepackage{hyperref}

\usepackage[accepted]{icml2020}

\icmltitlerunning{Preselection Bandits}

\usepackage{psfrag,epsf}
\usepackage{url} % not crucial - just used below for the URL 
\usepackage{amssymb,amsmath,amsthm,dsfont,tikz,ifthen,subfigure,booktabs,slashbox
}
\usepackage[mathscr]{euscript}

\usepackage{nicefrac}
\usepackage{units} 

\newcommand\R{\mathbb R} % K?rper der reellen Zahlen
\newcommand\N{\mathbb N} % nat?rliche Zahlen	
\newcommand\E{\mathbb E} % Erwartungswert
\newcommand{\argmax}[1]{\underset{#1}{\operatorname{arg}\,\operatorname{max}}\;} % argmax
\newcommand{\argmin}[1]{\underset{#1}{\operatorname{arg}\,\operatorname{min}}\;} % argmin
\newcommand\F{\mathcal{F}} %Filtrierung F
\newcommand\pvec[1]{\mathbf{#1}}

\newcommand\inv{^{-1}} % inverse

\newcommand\Kl[1]{\mathrm{KL} \big(#1\big)} % Kullback Leibler DIstanz

\newcommand{\IGNORE}[1]{}

\newcommand{\argmaxcode}{{\operatorname{arg}\,\operatorname{max}}} % argmax

\newcommand{\rank}{\mathbf r}
\renewcommand{\L}{\mathscr L}

\newcommand{\items}{n}

\newcommand{\pullnumber}{l}

\newcommand{\SK}{\mathbb S_{\items}}
\newcommand{\order}{\rank\inv}

\newcommand{\qioneetc}{q_{i_1,\ldots,i_{\pullnumber}}}
\newcommand{\oddij}{O_{i,j}}
\newcommand{\subij}{S_{i,j}}
\newcommand{\subminusi}{S_{-i,j}}
\newcommand{\subminusj}{S_{i,-j}}

\renewcommand{\P}{\mathbb P}

\newcommand{\timehorizon}{T}
\newcommand{\algo}{\varphi}

\newcommand{\oddiJ}{O_{i,J}}

\newcommand{\skill}{\theta}
\newcommand{\paramspace}{\Theta}

\newcommand{\estoddij}{\hat O_{i,j}}
\newcommand{\estoddiJ}{\hat O_{i,J}}

\newcommand{\oddiJtrcb}{\hat O^{TRCB}_{i,J}}

\newcommand{\Cshrink}{C_{shrink}}

\newcommand{\actionspace}{\mathbb A}
\newcommand{\subsets}{\mathbb A_\pullnumber}
\newcommand{\powerset}{\mathbb A_{full}}

\newcommand{\reward}{\mathrm{U}}
\newcommand{\rewardtilde}{\widetilde \reward}
\newcommand{\optsubset}{S^*}
\newcommand{\regret}{\mathcal{R}}

\newcommand{\skillmin}{\skill_{min}}
\newcommand{\skillmax}{\skill_{max}}

\newcommand{\rpreselect}{r}
\newcommand{\regpertime}{\tilde r}

\newtheorem{theorem}{Theorem}[section]

\newtheorem{lemma}[theorem]{Lemma}

\newtheorem{example}{Example}
\theoremstyle{remark}
\newtheorem{remark}{Remark}
\newcommand*{\defeq}{\mathrel{\vcenter{\baselineskip0.5ex \lineskiplimit0pt
			\hbox{\footnotesize.}\hbox{\footnotesize.}}}%
	=}
\newcommand{\defi}{\defeq}

\begin{document}
	
	\twocolumn[
	\icmltitle{Preselection Bandits }
	
	
	
	\icmlsetsymbol{equal}{*}
	
	\begin{icmlauthorlist}
		\icmlauthor{Viktor Bengs}{UPB}
		\icmlauthor{Eyke H\"ullermeier}{UPB}
	\end{icmlauthorlist}

	\icmlcorrespondingauthor{Viktor Bengs}{viktor.bengs@upb.de}
	
	\icmlaffiliation{UPB}{Heinz Nixdorf Institute and Department of Computer Science, Paderborn University, Germany}

	\icmlkeywords{Preference Learning, Exploration, Exploitation, Online Learning, Plackett-Luce, MNL model}
	
	\vskip 0.3in
	]
	
	
	
	\printAffiliationsAndNotice{}  % leave blank if no need to mention equal contribution
%	\printAffiliationsAndNotice{\icmlEqualContribution} % otherwise use the standard text.

\begin{abstract}
	In this paper, we introduce the Preselection Bandit problem, in which the learner preselects a subset of arms (choice alternatives) for a user, which then chooses the final arm from this subset.
	The learner is not aware of the user's preferences, but can learn them from observed choices. 
	In our concrete setting, we allow these choices to be stochastic and model the user's actions by means of the Plackett-Luce model.
	The learner's main task is to preselect subsets that eventually lead to highly preferred choices. 
	To formalize this goal, we introduce a reasonable notion of regret and derive lower bounds on the expected regret. 
	Moreover, we propose algorithms for which the upper bound on expected regret matches the lower bound up to a logarithmic term of the time horizon. 
\end{abstract}

\section{Introduction}
\label{sec_introduction}

The setting of preference-based multi-armed bandits or \emph{dueling bandits} \citep{busa2018preference} is a generalization of the standard stochastic multi-armed bandit (MAB) problem \citep{lattimore2018bandit}. Instead of assuming numerical rewards of individual arms (choice alternatives), the former is based on pairwise preferences between arms. In this paper, we introduce the Preselection Bandit (or simply Pre-Bandit) problem, which is closely related to the preference-based setting, especially to the recent variant of \emph{battling bandits} \citep{sahabattle}. 

Our setting involves an \emph{agent} (learner), which preselects a subset of arms, and a \emph{selector} (a human user or another algorithm), which then chooses the final arm from this subset.  
This setting is motivated by various practical applications. In information retrieval, for example, the role of the agent is played by a search engine, and the selector is the user who seeks a certain information. Another example is online advertising, where advertisements recommended to users can be seen as a preselection. As a concrete application, we are currently working on the problem of algorithm (pre-) selection \cite{kerschke2019automated}, where the (presumably) best-performing algorithm needs to be chosen from a pool of candidates.

In the beginning, the agent is not aware of the selector's preferences. However, the choices made by the latter reveal information about these preferences, from which the agent can learn. Due to time constraints, information asymmetry, or other reasons, we do not assume the selector to act perfectly, which means that it may miss the actually best among the preselected arms. In web search, for example, a user clicks on links based on limited information such as snippets, but without knowing the full content behind. Likewise, in algorithm selection, the final choice might be made on the basis of a cross-validation study, i.e., estimated performances that not guarantee the identification of the truly best algorithm. By modeling the selector's actions by means of the Plackett-Luce (PL) model \citep{luceindividual,plackett1975analysis}, we allow some randomness in the process of decision making. 
The agent's main task is to preselect subsets that eventually lead to highly preferred choices. To formalize this goal, we introduce a reasonable notion of regret based on utilities of preselected subsets, where an arm's (latent) utility is weighted with the probability of choosing this arm from the subset. 
In particular, this allows for capturing decision-making biases of users as studied intensively in behavioral economics or psychology.

%The agent's main task is to preselect subsets that eventually lead to highly preferred choices. To formalize this goal, we introduce a reasonable notion of regret and study two variants of the problem. 
We study two variants of the problem.  
In the first variant, which we call \emph{restricted} Pre-Bandit problem, the size of the preselection is predefined and fixed throughout. In the second variant, the \emph{flexible} Pre-Bandit problem, the agent is allowed to adjust the size of the preselection in every round. For these settings, we derive lower bounds on the expected regret. Moreover, for both scenarios, we propose active learning algorithms for which the upper bound on expected regret matches the lower bound (possibly) up to 
%a multiplicative constant and 
a logarithmic term of the time horizon. 

We discuss related work in Section \ref{sec_related_work}. In Section \ref{sec:prel}, we introduce the notation used throughout the paper, and also give a concise review of the PL model and some of its properties. In Section \ref{sec:pbp}, the Pre-Bandit problem is formally introduced, together with a reasonable notion of regret, for which  lower bounds with respect to the time horizon are verified.  Near-optimal algorithms for the two variants of the Pre-Bandit problem are provided in Section \ref{sec_algorithms}. 
We devote Section \ref{sec_experiments} to a simulation study demonstrating the usefulness and efficiency of our algorithms.
Finally, Section \ref{sec_conclusion} summarizes our results and discusses directions for future work.
All proofs of the theoretical results  are deferred to the supplements.
%supplementary material. 

\section{Related Work} \label{sec_related_work}

Bandit problems with the possibility of using more than one arm at a time have been considered in various ways in the literature. 
However, as will be detailed in the following, none of the previous works encompasses the  problem considered in this paper. Due to the specific meaning of the subset choice as a preselection in practical applications, this also justifies a new name for the setting.

A variant of the MAB problem is the combinatorial bandit problem \citep{cesa2012combinatorial,kveton2015tight}, in which the learner chooses a subset of arms in each time step, and then observes quantitative feedback, either in the form of rewards of each single arm (semi-bandit feedback) or the total sum of the rewards (bandit feedback) for the arms in the chosen subset. 
This differs fundamentally from our setting, in which no quantitative feedback is ever observed; instead, only qualitative feedback is provided, i.e., which arm is picked in a subset.

Qualitative forms of feedback for multiple arm choices at a time is considered in the realm of dueling, multi-dueling \citep{sui2017multi,busa2018preference} or battling bandits \citep{sahabattle}.
The flexible Pre-Bandit problem has obvious connections to the latter settings, with the freedom of adjusting the size of comparison for each time instance and can be interpreted as a combinatorial bandit problem with qualitative feedback.
\citet{saha2019pac} investigate  the effect of this flexibility in an active PAC-framework for finding the best arm under the PL model, while the active top-k-arm identification problem in this model is studied by \citet{chen2018nearly}. 
Recently, this scenario was considered in terms of a regret minimization problem with top-$m$-ranking feedback by \citet{saha2019regret}
for a straightforward extension of the dueling bandit notion of regret and not regarding the ``value'' of a subset in its entirety as we do (cf.\ Section \ref{subsec:regret_definition}).

Moreover, the algorithms suggested in \citet{sahabattle} are not applicable within our scenario for the restricted Pre-Bandit problem, as either the algorithms are focusing on the linear-subset choice model, which is fundamentally different from our choice model (cf.\ Section \ref{sec:prel}) or the algorithms allow replicates of the same arm within a chosen subset, which is inadmissible within our scenario (and in many practical applications as well).

The Pre-Bandit problem also reveals parallels to the \emph{Dynamic Assortment Selection} (DAS) problem  \citep{caro2007dynamic}, where a retailer seeks to find an optimal subset of his/her available items (or products) in an online manner, so as to maximize the expected revenue (or equivalently minimize the expected regret).
The DAS problem under the multinomial logit model, which is also known as the \emph{MNL-Bandits} problem \citep{rusmevichientong2010dynamic,saure2013optimal,agrawal2016near,agrawal2017thompson,wang2018near,chen2018dynamiccontext}  is especially close to our framework, as the corresponding concept of regret shares similarities with our definition of regret. 

However, our problem can rather be seen as complementary, since we do not assume a priori known revenues for each item or any revenues at all. While this might be natural for the retail management problem, it is arguably less so for applications we have in mind, such as recommendation systems or algorithm (pre-)selection.
In addition, we introduce a parameter in our setting that allows the learner to adjust its preselections with respect to the preciseness of the user's selections. This might also be an interesting direction for future work for the DAS problem.
Finally, to demonstrate the inappropriateness of the DAS algorithms for the restricted Pre-Bandit problem, we employ some of the algorithms in our experimental study. 

Another quite related branch of research is the so-called \emph{stochastic click model} \citep{zoghi2017online,lattimore2018toprank}, where a list of $\pullnumber$ items is presented to the selector in each iteration. Scanning the list from the top to the bottom, there is a certain probability that the selector chooses the item at the current position, or otherwise continues searching (eventually perhaps not choosing any item).
Thus, in contrast to our setting, the explicit order of the arms within a subset (or list) is a relevant aspect.
Further, the resulting learning task boils down to finding the $\pullnumber$ most attractive items, as these provably constitute the optimal list in this scenario (which is not necessarily the case for our setting).

\section{Preliminaries}\label{sec:prel}

\subsection{Basic Setting and Notation} \label{sec_notation}

We formalize our problem in the setting of preference-based multi-armed bandits \citep{busa2018preference}, which proceeds from a set of $\items$ arms, each of which is considered as a choice alternative (item, option).  We identify the arms by the index set $[\items] \defi \{1,\ldots,\items\}$, where $\items \in \N$ is arbitrary but fixed. Moreover, we assume a total preference order $\succ$, where $i \succ j$ means that the $i$th is preferred to the $j$th arm.  

Let $\subsets$ be the set of all $\pullnumber$-sized subsets of $[\items]$ and $\powerset \defi \cup_{\pullnumber=1}^\items  \subsets $. % the power set of $[\items]$. 
Moreover, let $\SK$ be the symmetric group on $[\items]$, the elements of which we refer to as \emph{rankings}: each $\rank \in \SK$ defines a ranking in the form of a total order of the arms $[n]$, with $\rank(i)$ the position of arm $i$.  We assume that $\SK$ is equipped with a probability distribution $\P:\, \SK \to [0,1]$. For an integer $\pullnumber > 1$ and a set of arms $\{{i_1}, {i_2},\ldots,{i_l} \} \subseteq [\items]$, the probability that $i_1$ is the most preferred among this set is given by
\begin{equation} \label{def_pairwise_marginals}
\qioneetc  \defi 
%\P({i_1} \succ \{{i_2},\ldots,{i_l}\}) \\
%				
\sum\nolimits_{\rank \in \SK: \rank(i_1) = \min\left(\rank(i_1), \ldots, \rank(i_l)\right)  }  \P(\rank) \enspace .
\end{equation}
%\\

\subsection{The Plackett-Luce Model} \label{sec_pl_model}

The Plackett-Luce (PL) model \citep{plackett1975analysis,luceindividual} is a parametric distribution on the symmetric group $\SK$ with parameter $\skill=(\skill_1,\ldots,\skill_{\items})^T \in \R_+^{\items},$ where each component $\skill_{k}$ corresponds to the strength of an arm $k,$ which we will refer to as \emph{score parameter.} 
The probability of a ranking $\rank \in \SK$ under the PL model is 
%\\
\begin{equation} \label{def_pl_model}
\P_\skill(\rank)= \prod\nolimits_{i=1}^{\items}  \frac{\skill_{\order(i)}}{ \skill_{\order(i)} + \ldots + \skill_{\order(\items)}   } \enspace ,
\end{equation}
where $\order(i)$ denotes the index of the arm on position $i$. 
According to (\ref{def_pl_model}), PL models a stage-wise construction of a ranking, where in each round, the item to be put on the next position is chosen with a probability proportional to its strength. 
As a model of discrete choice, the PL distribution has a strong theoretical motivation.  For example, it is the only model that satisfies the Luce axiom of choice \citep{luceindividual}, including independence from irrelevant alternatives (\emph{ILA property}, see \cite{alvo2014statistical}). Besides, it has a  number of appealing mathematical properties. For instance, there is a simple expression for the $\pullnumber$-wise marginals in (\ref{def_pairwise_marginals}): 
%\\
\begin{equation} \label{def_lwise_marginals_pl_model}
q_{i_1,\ldots,i_\pullnumber} = \frac{\skill_{i_1}}{\skill_{i_1} + \skill_{i_2} + \ldots + \skill_{i_l}}
\end{equation}
This probability is identical to the popular Multinomial Logit (MNL) model, which is a discrete choice probability model considered in various frameworks \cite{train2009discrete}.
For our purposes, the use of the relative scores
\begin{equation} \label{defi_relative_score}
\oddij  \defi \frac{\skill_i}{\skill_j}, \qquad i,j \in [\items], 
\end{equation}
will turn out to be advantageous, as they are directly affected  by the ILA property of the PL model.
Indeed, for $i,j \in [\items]$, let $\subij \in \subsets$ be such that $i,j \in \subij.$
Furthermore, define $\subminusi \defi \subij \backslash \{i\}$ and similarly $\subminusj \defi \subij \backslash \{j\}$ for  $i,j \in [\items].$ 
Then, for any such a set $S_{i,j}$, (\ref{def_lwise_marginals_pl_model}) and (\ref{defi_relative_score}) imply 
\begin{align*}
%	$
	\oddij = 	\frac{\skill_i}{\skill_j} 
	= \frac{\skill_i}{\skill_j} \cdot \frac{\sum_{ t \in \subij}  \skill_t  }{\sum_{ t \in \subij}  \skill_t} 
	= \frac{ q_{i,  \subminusi }	}{  q_{j, \subminusj} } \enspace .	
	%
%	$
\end{align*}
Without restricting the parameter space $ \paramspace = \{ \skill \in \R_+^{\items} \}$, the PL model in (\ref{def_pl_model}) is not (statistically) identifiable, as $\skill \in \R_+^{\items}$ and $\tilde \skill = C \, \skill$ for any constant $C>0$  lead to the same models, i.e.\ $\P_{\skill}  = \P_{\tilde \skill}.$
Restricting the parameter space by assuming some normalization condition on the score parameters fixes this issue. 
Thus, we consider as parameter space the (restricted) unit square w.r.t.\ the infinity norm,
\begin{align*}
\paramspace = \Big\{	\skill=(\skill_1,\ldots,\skill_{\items})^T \in &[\skillmin,1]^\items \ |  \ \skillmin\in(0,1), \ \\ &\skillmax\defi\max_i \skill_i =1			\Big\},
\end{align*}
which leads to an identifiable statistical model $(\P_{\skill})_{\skill \in \paramspace}$ and naturally yields a normalization of each individual score parameter easing the fast grading of an arm's utility.

For technical reasons, we additionally exclude models that allow scores below a certain threshold $\skillmin$ (which will be a small constant), as the relative scores in (\ref{defi_relative_score}) are then well-defined for any pair $(i,j)\in [\items]^2.$

\subsection{Degree of Preciseness}

In our setting, we model the score parameter $\skill$ as 
\begin{equation} \label{def_modeling_scores}
	\skill_i = v_i^\gamma, \quad \forall i \in [\items],
\end{equation}
 where $v_i \in \R_+$ represents the (latent) utility of arm $i,$ while $\gamma \in (0,\infty)$ represents the degree of preciseness of the user's selections:
The higher $\gamma,$ the more (\ref{def_pl_model}) resembles a point-mass distribution on the ranking modeling a precise selector that is always able to identify the best arm, while the lower $\gamma$, the more (\ref{def_pl_model}) resembles a uniform distribution modeling a  selector acting purely at random. 
The effect of $\gamma$ on the $\pullnumber$-marginals in (\ref{def_lwise_marginals_pl_model}) is quite similar.

Note that $v_1,\ldots,v_n,\gamma$ are not separately identifiable (c.f.\ Section 3.2 in \cite{train2009discrete}), but $\skill_1,\ldots,\skill_n$ are identifiable under our assumptions on the parameter space $\paramspace$ above.
Hence, by fixing $\gamma$, the latent utilities $v_1,\ldots,v_n$ are guaranteed to be identifiable.

\section{The Pre-Bandit Problem}\label{sec:pbp}

The considered online learning problem proceeds over a finite time horizon $\timehorizon$.
For each time instance $t \in [\timehorizon]$, the agent (i.e., the learner) suggests a subset $S_t \in \actionspace $, where $\actionspace$ is the action space.
The agent's action $S_t$ is based on its observations so far. As a new piece of information, it observes the selector's choice (i.e., the user or the environment) of an arm $i_t$ among the offered subset $S_t$ (with probability $q_{i_t,S_t\backslash\{i_t\}}$ given by (\ref{def_lwise_marginals_pl_model})).

Suppose $r: \actionspace \to \R_+$ is a suitable regret function (to be defined in the next section below). The goal of the learner resp.\ the agent is to preselect the available arms by means of subsets $S_t$ in every time instance $t$ such that the  \emph{expected cumulative regret} over the time horizon, that is $	\E_{\skill} \sum_{t=1}^\timehorizon r(S_t)	$ with $\skill \in \paramspace$,  is minimized.
The problem is analyzed for two possible characteristics of the action space:
\begin{itemize}
	\item {(Restricted Preselection)} $\actionspace=\subsets,$ i.e., a preselection consists of exactly $\pullnumber$ many arms, where $\pullnumber$ is a fixed integer strictly greater than one.
	\item {(Flexible Preselection)} $\actionspace=\powerset,$ i.e., a preselection can be any non-empty subset of $[\items].$
\end{itemize}
In the following, we introduce sensible notions of regret for the considered problem setting. The key question we then address is the following: What is a good preselection to present the selector?
Moreover, we provide a lower bound on the related expected cumulative regret.

\subsection{Regret Definition} \label{subsec:regret_definition}

%
%Assuming the selector to behave according to the PL model with score parameter (\ref{def_modeling_scores}), the expected utility of suggesting $S$, from the point of view of the agent, is given by 
Assuming the selector to behave according to the PL model with score parameter (\ref{def_modeling_scores}), the expected utility of suggesting $S$ is given by
\begin{align} \label{defi_expected_utility}
\begin{split}
\reward(S) := \reward(S;v,\gamma)  &= \sum_{ i \in S} v_i \cdot q_{i,S\backslash\{i\}} \\
& = \frac{\sum_{ i \in S} v_i^{1+\gamma} }{\sum_{ i \in S} v_i^\gamma }.
\end{split}
\end{align}
Indeed, if $i_t \in [\items]$ is the chosen arm at time $t,$ then 
\begin{center}
	$	\E(v_{i_t} \, | \, S) = \sum_{ i \in S} v_i \cdot  \P(\mbox{$i$ is chosen} \, | \, S) = \reward(S).	$
\end{center}
Hence, the corresponding optimal preselection is  
\begin{align} \label{definition_optimal_subset}
\optsubset \in 
\begin{cases}
\argmaxcode_{ S \subseteq [\items], \,  |S|=\pullnumber  } \ \reward(S), & \mbox{if } \actionspace=\subsets, \\
% \quad (\mbox{restricted Pre-Bandit}), \\
%	
\argmaxcode_{  i  } \ \ \skill_i, & \mbox{if } \actionspace=\powerset.
%\quad (\mbox{flexible Pre-Bandit}).
%	
\end{cases}
\end{align}
The (instantaneous) regret suffered by the selector is anticipated by the agent through
\begin{align} \label{defi_regret_preselect}
\rpreselect(S) := \reward(\optsubset) - \reward(S), \quad S \in \actionspace.
\end{align}
Thus, if $S_1,\ldots,S_\timehorizon$ are suggested for times $1$ to $\timehorizon$, respectively, the corresponding cumulative regret over $\timehorizon$ is
% over the time horizon is  given by
%		
\begin{align} \label{defi_regret}
\regret(\timehorizon) :=   \sum\nolimits_{ t =1}^{\timehorizon} \rpreselect(S_t) = \sum\nolimits_{ t =1}^{\timehorizon} \big(\reward(\optsubset) - \reward(S_t)\big).
\end{align}	

\begin{remark} [Relations to dueling bandits and battling bandits]
	Note that the optimal subset for $\actionspace=\powerset,$ i.e., for the flexible Pre-Bandit problem, always consists of the items whose score parameters equal the overall highest score  $\skill_{max}.$
	Thus, like for the dueling bandits and battling bandits problem, the goal is to find the best arm(s).
	However, whilst in the latter settings only pairwise resp.\ fixed $\pullnumber$-wise comparisons of arms are observed, we allow to draw comparisons of arbitrary size.  
In addition, the restricted Pre-Bandit problem can be interpreted as a dueling resp.\ battling bandit problem.
%	In addition, the restricted Pre-Bandit problem for $\pullnumber = 2$ can be interpreted as a dueling bandit problem. %, though, the notion of regret seems more natural.
	Compared to the latter, however, the notion of regret has a more natural meaning in our setting. This is due to the different semantics of a selection of a pair (or any subset) of arms, which is a preselection that eventually leads to a concrete choice. In other words, the regret in \eqref{defi_regret_preselect} focuses on the perceived preference of a subset by  regarding the ``value'' of a subset in its entirety.
%	(and hence ''reward'').
	%	
	%	 
	%
\end{remark}

\begin{remark} [No-choice option]
	In the related branches of literature (cf.\ Section \ref{sec_related_work}), it is common to assume an additional choice alternative which represents the possibility of the user choosing none of the alternatives in the preselection.
	Formally, this can be expressed in our setting by extending the $n$-dimensional parameter space $\paramspace$ to $n+1$ dimensions by augmenting it with a dummy score parameter $\skill_0 \in (\skillmin,1]$ representative for this no-choice option.
	As this option is always available, this dummy item is part of every preselection, and consequently its latent utility affects only the choice probabilities in (\ref{defi_expected_utility}).
	However, although we refrain from incorporating the no-choice alternative in this paper, it is straightforward to show similar lower bounds as below for this problem and to modify the suggested algorithms to this scenario.
\end{remark}

%\paragraph{Most-preferred subsets}
\subsection{Most preferred Subsets} \label{subsec_most_preferred_subsets}

One tempting question is how the most preferred subsets look like, given our definition of regret.
As already mentioned, the optimal preselection $\optsubset$  for the flexible Pre-Bandit variant consists of the items with the same highest score parameter.
However, in the restricted Pre-Bandit variant, the optimal preselection does not necessarily consist of the $\pullnumber$ items with the highest scores in general, as the following examples demonstrate.
\begin{example} \label{example_optimizing_sets}
	In Table \ref{table_example_reward}, we provide three problem instances with fixed degree of preciseness $\gamma=1$ for $\items=5$ and the corresponding expected scores of (the relevant) 3-sized subsets of $[\items]$.
	In the first instance, where one arm has a much higher utility than the remaining ones, it is favorable to suggest this high utility arm together with the arms having smallest utility.
	This is due to the large differences between the utilities, so that the selector will take the best arm with a sufficiently high probability. Roughly speaking, the best strategy for the agent is to make the problem for the selector as easy as possible. 
	\begin{table}[ht]
	\centering		
	\caption{Problem instances with different optimal subsets (indicated in bold font) for the regret in (\ref{defi_regret_preselect}) with $\items=5$ and $\pullnumber=3$ (omitted subsets had smaller utilities throughout). }
	\label{table_example_reward}
	\resizebox{0.46\textwidth}{!}{
		\begin{tabular}{r|lllll}
			S & \{1, 2, 3\} & \{1, 2, 5\} &   \{1, 3, 5\} & \{1, 4, 5\} & \{2, 3, 4\}  \\ 
			\hline
			\hline
			\multicolumn{6}{c}{$v= (1\,,0.122\,,0.044\,,0.037\,,0.017), \gamma =1$}\\
			\hline
			$   \reward(S)  $ & 0.872 & 0.891 & 0.945 & \textbf{0.951} & 0.0896 \\
			\hline
			\multicolumn{6}{c}{$v= (1\,,0.681\,,0.572\,,0.543\,,0.399), \gamma =1$}\\
			\hline
			$   \reward(S) $ & \textbf{0.795} & 0.780 & 0.754 & 0.749 & 0.604 \\ 
			\hline
			\multicolumn{6}{c}{$v= (1\,,0.681\,,0.572\,,0.543\,,0.171), \gamma =1$}\\
			\hline
			$   \reward(S) $ & 0.795 & \textbf{0.806} & 0.778 & 0.773 & 0.604 \\ 
			\hline
			\hline
		\end{tabular}
	}
\end{table}
	The second instance is different, as the optimal preselection for the agent now consists of the top-3 arms with the highest scores. This comes with a non-negligible probability of missing the optimal arm, however, since the runner-up arms are sufficiently strong, the regret can be tolerated. On the other hand, adding a poor arm would be suboptimal, as one cannot be certain enough that it will not be taken. But by reducing the score for the worst arm notably as in the third instance, the worst arm substitutes the third best, as then the best item can again be better distinguished from the suboptimal ones inside the optimal subset.
\end{example}
As suggested by this example, a reasonable strategy is to compose the preselection of subsets of best and worst arms, respectively. 
In fact, we show in the supplementary material (Section \ref{subsec:algo_reward_maximization}) that the optimal subsets for the restricted Pre-Bandit problem are always composed of best and worst arms with the overall best arm(s) mandatory inside the optimal subset.

The obvious rationale of adding a strong arm is to guarantee a reasonably high utility, whereas a poor arm merely serves as a decoy to increase the probability of choosing the best arm. 
Such effects are known in the literature on decision theory as the \emph{attraction effect} or the \emph{decoy effect}, see \citep{dimara2017attraction} and references therein. 
In particular, our definition of regret is able to capture this effect and consequently emphasizes that our regret aims at penalizing difficult decisions for the selector in the restricted case. 

However, the manifestation of the decoy effect and consequently its demand for the application at hand, can be steered by the learner through fixing the degree of preciseness $\gamma$ as illustrated by the next example.

\begin{example}\label{example_optimizing_sets_second}
	In Table \ref{table_example_reward_sec}, we investigate the effect of the degree of preciseness $\gamma$ for the third instance of Example \ref{example_optimizing_sets}.
	In the first instance, where the degree of preciseness is moderate, i.e., $\gamma=1,$ the attraction effect is still present.
	But by increasing the degree of preciseness to $\gamma=20$ as in the second instance, the (rounded) utilities of all subsets containing the best item are the same.
	For such a problem instance, it is enough to provide the best item in the preselection, as the user's probability of choosing the best item is sufficiently high in these cases and the best arm will not be missed.
	In the last instance, the utility monotonically decreases with the sum of scores of the subset $S$ if the degree of preciseness is reduced to $\gamma=0.05.$
	This is due to a throughout non-negligible probability for choosing the worst item inside the preselected subset $S.$
	Thus, in order to dampen this effect and the resulting regret, it is best to preselect the top arms.
	\begin{table}[ht]
		\centering		
		\caption{Problem instances with different optimal subsets (indicated in bold font) for the regret in (\ref{defi_regret_preselect}) with $\items=5$ and $\pullnumber=3$ (omitted subsets had smaller utilities throughout). }
		\label{table_example_reward_sec}
		\resizebox{0.46\textwidth}{!}{
			\begin{tabular}{r|lllll}
				S & \{1, 2, 3\} & \{1, 2, 4\} &   \{1, 2, 5\} & \{1, 3, 5\} & \{2, 3, 4\}  \\ 
				\hline
				\hline
				\multicolumn{6}{c}{$v= (1\,,0.681\,,0.572\,,0.543\,,0.171), \gamma =1$}\\
				\hline
				$   \reward(S) $ & 0.795 & 0.791 & \textbf{0.806} & 0.778 &  0.605 \\ 	
				\hline
				\multicolumn{6}{c}{$v= (1\,,0.681\,,0.572\,,0.543\,,0.171), \gamma =20$}\\
				\hline
				$   \reward(S) $ & \textbf{1.000} & \textbf{1.000} & \textbf{1.000} & \textbf{1.000} & 0.676 \\ 
				\hline
				\multicolumn{6}{c}{$v= (1\,,0.681\,,0.572\,,0.543\,,0.171), \gamma =0.05$}\\
				\hline
				$   \reward(S) $ & \textbf{0.753} &  0.744 & 0.630 & 0.593 & 0.599 \\ 		
				\hline
				\hline
			\end{tabular}
		}
	\end{table}
\end{example}

In view of this example, the learner can interpolate between the two extreme cases of users by varying $\gamma:$
Taking a sufficiently large $\gamma,$ the learner can model a very precise user for whom it suffices to preselect a subset that just entails the best item.
Using a sufficiently small $\gamma$ instead, the learner is able to reproduce a random user for whom it is best to compose the preselection of the best items.
\begin{remark}
	Note that, in terms of regret, the case of a very precise user is related to the weak regret of the dueling bandits problem \citep{yue2012k,chen2017dueling}, where no regret occurs whenever the best arm is participating in the duel.
	Similarly, the semantics of the regret in the case of a random user corresponds to the strong regret of the dueling bandits problem, where the regret is zero only when the best arm is duelled with itself. 
\end{remark}

\subsection{Lower Bounds}

In this section, we prove lower bounds on the expected regret defined in (\ref{defi_regret}) for the two types of Pre-Bandit problems. 
\begin{theorem} \label{theorem_lower_bound_regret} [Restricted Preselection Bandits]
	Let $\items\in \N,$ $\pullnumber \leq \items/4 $, and $\timehorizon \geq \items $ be integers.
	Then, for any algorithm $\algo$ suggesting an $\pullnumber$-sized subset $S_t^{\algo}$ at time $t,$ 
%	
	%	\\
	\begin{align*}
	\E_\skill \big( \regret(\timehorizon) \big) 
	&=  \sum\nolimits_{ t =1}^{\timehorizon} \E_\skill \big(\reward(\optsubset) - \reward(S_t^{\algo})\big) \\ 
	&\geq  \min\{1,1/\gamma	\} \,  C \, \sqrt{ \items \, \timehorizon }  	
	\end{align*}	
	 holds for any 	$\skill \in \paramspace$ and any $\gamma \in(0,\infty),$	 
	where $C>0$ is some constant independent of $\items,\pullnumber$, and $\timehorizon.$
\end{theorem}
\begin{remark} \label{remark_on_lower_bound_restricted_pb}
	The order of the lower bound in Theorem \ref{theorem_lower_bound_regret} coincides with the lower bound on the expected regret derived by \citet{chen2018note} for the DAS problem under the MNL model with capacity constraints.
	In particular, the preselection size $\pullnumber$ does not affect the order, at least if it is smaller than $\items/4.$ 
	Although the lower bounds are theoretically of the same order, it is not directly possible to use the lower bound results of \citet{agrawal2016near} or \citet{chen2018note}, as in both proofs the probability of the no-choice option is assumed to be strictly positive, and the revenues all equal 1.
	Moreover, the results for the stochastic click model are quite different from ours (cf.\ Theorem 2 by \citet{lattimore2018toprank}), as there the size of the subset $\pullnumber$ is present in the lower bound.
	Therefore, we provide a proof in the supplementary material (Section \ref{sec:lower_bounds_proofs}).
\end{remark}

\begin{theorem} \label{theorem_lower_bound_regret_flexible_size}[Flexible Preselection Bandits]
	Let $\items\in \N$ and $\timehorizon \geq \items $ be integers.
	Then, for any algorithm $\algo$ suggesting subset $S_t^{\algo} \in \powerset$ at time $t,$ the following holds for any $\gamma \in(0,\infty)$:
	
	\noindent
	(i) [Gap-independent version] There exists a constant $C>0$ independent of $\items$ and $\timehorizon$, such that 
		$$	\sup\nolimits_{\skill \in \paramspace}	\E_\skill \big( \regret(\timehorizon) \big) 
		\geq C \, \min\{1,1/\gamma\} \, \sqrt{  \timehorizon  } \, .		$$
		%	 
%		where 
		%
	(ii) [Gap-dependent version]	%	
		If  $\algo$ is a no-regret algorithm (cf.\ Definition 2 in \cite{saha2019regret}), there exists a constant $C >0$ independent of $\items$ and $\timehorizon$, such that 
		\begin{align*}
				\sup\nolimits_{\skill \in \paramspace} \Big(\min_{i\notin \optsubset} \,  (&\skillmax - \skill_i ) \, \cdot \,	\E_\skill \big( \regret(\timehorizon) \big)  \Big)\\ &\geq C  \, \min\{1,1/\gamma\} \, (\items-1) \, \log(\timehorizon)  \, .	
		\end{align*}	
\end{theorem}
\begin{remark}
 	Note that the gap-independent lower bound is independent of the number of arms $\items.$ 
	This is in line with the enhancement for the DAS problem for the uncapacitated compared to the capacitated MNL model \citep{wang2018near}.
	On the other hand, the gap-dependent lower bound depends on the number of arms $\items$, and is of the same order as in the dueling bandit setting.
	In particular, compared to the dueling bandits setting, there is (theoretically) no improvement by offering subsets larger than two. 
	This is in accordance with the observations  made by \citet{sahabattle,saha2019regret}.
%	
%	Again, the proofs for the DAS problem  is not straightforwardly transferable to the considered problem, so that we give a proof in the supplementary material (Section \ref{sec:lower_bounds_proofs}). 
	%	
\end{remark}

\section{Algorithms} \label{sec_algorithms}

In this section, we propose the \emph{Thresholding-Random-Confidence-Bound} (TRCB) algorithm stated in
Algorithm \ref{algo_explore_exploit}. This algorithm returns subsets $S_1,\ldots,S_T$ for the restricted Pre-Bandit problem. As will be shown, it has a satisfactory upper bound for the expected cumulative regret in (\ref{defi_regret}).
For the flexible Pre-Bandit problem, we  further suggest the \emph{Confidence-Bound-Racing} (CBR) algorithm as stated in Algorithm \ref{algo_racing}.
It is inspired by the idea of \emph{racing algorithms,} initially introduced by \citet{maron1997racing}  to find the best model in the framework of model selection.

\subsection{The TRCB Algorithm} \label{sec_TRCB_algo}

%\begin{wrapfigure}[30]{r}{0.49\textwidth}
%	\begin{minipage}{0.5\textwidth}
		\begin{algorithm}  [h]
%			[H]
			%	
			\caption{TRCB algorithm}\label{algo_explore_exploit}
			\begin{algorithmic}[1]
				%		\small
				%	
				\INPUT{ Set of arms $[\items],$ preselection size $l\in [2,\items] \cap \N,$
%					 time horizon $T,$ 
				lower bound for score parameters $\skillmin,$ magnitude of uncertainty consideration $\Cshrink \in (0,1/2),$ degree of preciseness $\gamma$}
				\STATE{\textbf{initialization:} $ W= [w_{i,j}]_{i,j} \leftarrow \pvec(0)_{\items \times \items}$} 
%				 \COMMENT{win counts} 
				\STATE $ \hat O = [\estoddij]_{i,j} \leftarrow \pvec(1)_{\items \times \items}$  
%				\COMMENT{relative score estimates} 
				\REPEAT
				\STATE 	$t \ \ \leftarrow t+1$\; 
				
				\STATE $J	\leftarrow \argmaxcode_{i \in [\items]} \ \# \{w_{i,j} \geq w_{j,i}	 \ | \ j \neq i		\}$
				\STATE  \COMMENT{Break ties arbitrarily} 
				
				\FOR{$i \in \{1,\ldots,\items\}\backslash\{J\}$}

				\STATE Sample $\beta_i \sim \text{Unif}[ \pm \sqrt{\frac{  32 \log( \pullnumber t^{3/2})}{ \skillmin^4 (w_{i,J} + w_{J,i}) }}	 	]  $
				
				\STATE $\oddiJtrcb \leftarrow \min \big(  \skillmin\inv  , \max\big( \estoddiJ + \Cshrink \, \beta_i				, \skillmin \big)					\big)$
				
				\ENDFOR
				
%				\STATE $\oddJJtrcb \leftarrow 1$
				
				\STATE		Compute $\hat S \leftarrow \argmaxcode_{ S \in  \subsets  }  \rewardtilde(S; \hat O_J^{TRCB},\gamma) $  
%				\COMMENT{Choose randomly if not unique}  

				\STATE		Suggest $S_t = \hat S$ and obtain choice $i_t \in S_t$
				
				\STATE		Update $ w_{i_t,j} \leftarrow w_{i_t,j} +1, j\in S_t\backslash\{i_t\}$   \\
				and for $i,j\in S_t$ \\

				 $ \hat O_{i,j} \leftarrow \begin{cases}
				\nicefrac{w_{i,j}}{w_{j,i}} , & w_{j,i} \neq 0, \\
				\skillmin, & \mbox{else.}			
				\end{cases}  $
				
				\UNTIL{$t==\timehorizon$}
				%		
				%	
				%		\normalsize	
			\end{algorithmic}
		\end{algorithm}
%		%
%	\end{minipage}
%\end{wrapfigure}

First of all, note that an estimation of the score parameter $\skill$ is not necessary for the goal of regret minimization. Instead, a proper estimation of the relative scores in (\ref{defi_relative_score}) is sufficient.
Indeed, maximizing the expected utility (\ref{defi_expected_utility}) is equivalent to maximizing the expected utility with respect to some reference arm $J,$ that is
%\\
\begin{align} \label{defi_expected_utility_with_rel_score}
\rewardtilde(S) = \rewardtilde(S; O_J , \gamma ) \defi \frac{\sum_{ i \in S} O_{i,J}^{\nicefrac{(1+\gamma)}{\gamma}}}{\sum_{ i \in S} O_{i,J}} \, ,
\end{align}
where $O_J=(O_{1,J},\ldots,O_{\items,J}   ),$
simply because $\rewardtilde(S)  = v_{  J}^{-1} \cdot \reward(S).$

Thanks to Lemma 1 by \citet{saha2019pac}, one can derive appropriate confidence region bounds based on a similar exponential inequality for the relative score estimates, so that one might be tempted to use a UCB-like policy for the restricted Pre-Bandit problem.
However, the main problem of such an approach is UCB's principle of ``optimism in the face of uncertainty'', which tends to exclude arms with low score from a preselection. As we have seen in Example~\ref{example_optimizing_sets}, such arms could indeed be part of the optimal subset $\optsubset$ depending on the specific value of $\gamma.$

The core idea of the TRCB algorithm is to solve this issue with a certain portion of pessimism. 
Instead of using the upper confidence bound estimates for the relative scores, a random value inside the confidence region of the relative score estimate is drawn (lines 7--8), so that pessimistic guesses for the relative scores are considered as well, which in turn ensures sufficient  exploration of the algorithm.
%\\
This sampling idea can be interpreted as a frequentist statistical version of Thompson Sampling.
To exclude inconsistencies with the score parameter space (cf.\ Section \ref{sec_pl_model}), these random confidence values are appropriately thresholded.

Until the a priori unknown time horizon is reached (lines 3, 4, 13), the TRCB algorithm repeatedly does the following.
Primarily, the arm with the highest total number of wins for the pairwise comparisons is determined as the reference arm $J$ (line 5).
Next, for every other arm, a random value inside its confidence region for its relative score with respect to the reference arm $J$ is drawn with uniform distribution and appropriately thresholded (lines 7--10).
These thresholded random values correspond to the current belief on the actual relative scores with respect to $J$ and are used to determine the preselection with the highest utility in (\ref{defi_expected_utility_with_rel_score}) (line 11).
After offering this preselection to the selector and observing its choice (line 12), the pairwise winning counts are updated (by breaking down the $\pullnumber$-wise comparison into pairwise comparisons) as well as the estimates for the relative scores (line 13). 

The following theorem shows that the upper bound for the worst-case cumulative regret of the proposed TRCB algorithm matches the information-theoretic lower bound on the cumulative regret in Theorem \ref{theorem_lower_bound_regret} with regard to $\items$ and $T$ up to a logarithmic term of $\timehorizon$ (the proof is given in Section \ref{sec:proof_TRCB} of the supplement).
\begin{theorem} \label{theorem_upper_bound_regret}
	If $ \Cshrink \in (0,1/2),$ then for any $\gamma\in(0,\infty)$ and any $ \timehorizon > \items,$ 
	\begin{align*}
				&\sup\nolimits_{\skill \in \paramspace} \E^{\mathrm{TRCB}}_\skill \, \regret(\timehorizon)	
				\\ &\leq C \,  \frac{\max \{\skillmin^{(\gamma-1)/(\gamma)},\skillmin^{(1-\gamma)/(\gamma)} \} }{ \gamma \, \skillmin^{2(3+\gamma)}}   \,  \sqrt{ \items \, \timehorizon \log(\timehorizon) },  
	\end{align*}	
	where $C>0$ is some constant independent of $\items,\pullnumber,$ $\timehorizon$ as well as  $\skillmin$ and $\gamma.$ 
\end{theorem}
\begin{remark}		
	The maximization over $\subsets$ in Algorithm \ref{algo_explore_exploit} (line 11) can be realized by Algorithm \ref{algo_reward_maxi} provided in the supplementary material. It keeps the computational cost low by exploiting structural properties of the utility function $\reward$ and the most preferred subsets (see Section \ref{subsec_most_preferred_subsets}).
\end{remark}
%\begin{remark}
%	Apparently the regret bound of TRCB increases with decreasing $\skillmin.$
%\end{remark}

%
\subsection{The CBR Algorithm} \label{sec_CBR_algo}

The CBR algorithm is structurally similar to the TRCB algorithm. However, it uses estimates of the pairwise winning probabilities and the corresponding confidence intervals instead of the relative scores.  

%\begin{wrapfigure}[26]{r}{0.47\textwidth}
%	\begin{minipage}{0.5\textwidth}
\begin{algorithm} [h]
	%			 [H]
	%	
	\begin{algorithmic}[1]
%		\small
		\caption{CBR-algorithm}\label{algo_racing}
		\INPUT{ \footnotesize Set of arms $[\items],$  sigmoidal function $\sigma:\mathbb R \to [0,1]$ \normalsize}
		
		\STATE{\textbf{Initialization:} $ W= [w_{i,j}] \leftarrow \pvec(0)_{\items \times \items}$} 
		%				 \COMMENT{$\items\times \items$-matrix of win counts} 
		
		\STATE $ \hat Q= [\hat q_{i,j}] \leftarrow \pvec(\nicefrac{1}{2})_{\items \times \items}$, \ $ A \leftarrow  [\items]$  
		%				 \COMMENT{$\items\times \items$-matrix of pairwise winning probabilities} 
		
		%				\STATE  
		%				\COMMENT{Set of active arms} 
		
		\REPEAT
		
		\STATE 	$t \leftarrow t+1$\; 
		
		\STATE $J	\leftarrow \argmaxcode_{i \in [\items]} \# \{w_{i,j} \geq w_{j,i}	 \ | \ j \neq i		\}$ 
		\STATE \COMMENT{Break ties arbitrarily} 
		\STATE $S  \leftarrow \{J\}$
		%\COMMENT{Break ties arbitrary} 
		
		\FOR{$i \in A$} 
		
		\STATE  $c_i \leftarrow \sqrt{\nicefrac{   2 \log(\items t^{3/2})}{   (w_{i,J} + w_{J,i}) }}$ 
		\STATE  $\hat t_{i,J} \leftarrow  \sigma\big( \nicefrac{ (\hat q_{i,J} + c_i -1/2)}{2 c_i } \big)$
		
		\STATE $S \leftarrow \begin{cases}
		S \cup  \{i\}, & \mbox{with probability \ } \hat t_{i,J}  \\
		S, & \mbox{with probability \ } 1 - \hat t_{i,J}
		\end{cases}$

		%				\STATE $S \rightarrow S \cup B$
		
		\IF{ $\hat t_{i,J} = 0$ }
		
		\STATE $A  \leftarrow A \backslash \{i\}$ 
		
		\ENDIF
		
		\ENDFOR
		
		\STATE		Suggest $S_t = S$ and obtain choice $i_t \in S_t$
		
		\STATE		Update $ w_{i_t,j} \leftarrow w_{i_t,j} +1, j\in S_t\backslash\{i_t\}$  \\ and  $ \hat q_{i,j} \leftarrow
		\nicefrac{w_{i,j}}{(w_{i,j} + w_{j,i})} $ for $i,j\in S_t$
		
		\UNTIL{$t==\timehorizon$}
		\normalsize
	\end{algorithmic}
\end{algorithm}
%
%	\end{minipage}
%\end{wrapfigure}
%
In particular, the CBR algorithm maintains a pool of candidates $A \in [\items]$ and admits an arm $i\in A$ to be part of the preselection with a certain probability determined by the rate of uncertainty that $i$ could beat the current arm $J$ with the most winning counts.
This uncertainty is expressed through the ratio between the length of the confidence interval for $q_{i,J}$ (cf.\ the definition in (\ref{def_pairwise_marginals})) exceeding $1/2$ and the overall confidence interval's length.
More specifically, if $[l_i(t),u_i(t)]$ is the confidence interval for $q_{i,J}$ in time instance $t,$ then arm $i$ is included into the preselection with probability 
$
\sigma \big(\nicefrac{(u_i(t)-1/2)}{(u_i(t)-l_i(t))} \big),
$
where $\sigma:\mathbb R \to [0,1]$ is a sigmoidal function, i.e., a surjective monotone function with $\sigma(1/2)=1/2$ and $\sigma(x)> 0$ iff $x> 0.$
Note that the degree of preciseness $\gamma$ can be taken into account by the learner through the shape of $\sigma.$

Hence, if the confidence interval lies mostly above $1/2,$ that is $l_i(t) \approx 1/2,$ the chance is high that this particular arm could possibly beat the current best arm and consequently has a large probability of being included in the preselection.
In contrast, if the upper bound of the confidence interval is beneath $1/2,$ that is $u_i(t)\leq 1/2,$ the arm is discarded from the pool of candidates (lines 12--14), as one can be sure that this arm is already beaten by another. 

At the beginning, the major part of the arms have a high chance to be part of the preselection, which however decreases over the course of time until finally the preselection consists of only the best arm(s).
In the repetition phase, the preselection is successively built starting from the current arm with the most total number of wins for the pairwise comparisons and adding arms from the active set depending on the outcome of a Bernoulli experiment (lines 5--11), whose success probability depends on the length of the confidence interval (of the arm's pairwise winning probability against $J$) above $1/2.$
After offering this preselection to the selector and observing its choice (line 16), the pairwise winning counts and estimates on the pairwise winning probabilities are updated (line 17).

We have the following theorem for the upper bound on the cumulative regret for CBR, which matches the information-theoretic gap-dependent lower bound on the cumulative regret in Theorem \ref{theorem_lower_bound_regret_flexible_size} (the proof is given in Section \ref{sec:proof_CBR} of the supplement).
\begin{theorem} \label{theorem_upper_bound_regret_flexible}
	%	
%	Let $\algo$ denote the CBR algorithm.
	% 
	There are  universal constants $C_0,C_1>0,$ which do not dependent on $T$ or $n$, such that 	%	
		\begin{align*}
					\sup\nolimits_{\skill \in \paramspace} &\E^{\mathrm{CBR}}_\skill \, \regret(\timehorizon)	\leq C_0\, \items  \\ &\qquad +  C_1 \, \frac{\max \{\skillmin^{(\gamma-1)/(\gamma)},\skillmin^{(1-\gamma)/(\gamma)} \} }{ \gamma \, \skillmin } \, \times \\&\quad \qquad  \sum_{i \in [\items]\backslash \optsubset } \frac{\log(\timehorizon)\sum_{i\in [\items] \backslash \optsubset}(\skillmax - \skill_i)}{(\skillmax - \skill_i)^2} 	
		\end{align*}
for any $ \timehorizon > \items,$ $\gamma \in (0,\infty)$.
\end{theorem}
\section{Experiments} \label{sec_experiments}

In this section, we investigate the performance of TRCB (Algorithm \ref{algo_explore_exploit}) as well as CBR (Algorithm \ref{algo_racing}) on synthetic data for some specific scenarios, while providing further scenarios in the supplementary material.

\subsection{Restricted Pre-Bandit Problem}
First, we analyze the empirical regret growth with varying time horizon $\timehorizon$ for the restricted Pre-Bandit problem.
We consider the case $\items=10,$ $\pullnumber=3$, and time horizons $\timehorizon\in \{ i \cdot 2000\}_{i=1}^{5}.$ 
The degree of preciseness is $\gamma=1$ throughout, and the score parameters $\skill=(\skill_i)_{i\in [\items]}$ are drawn uniformly at random from the $\items$-simplex, i.e., without a restriction on their minimal value and thus allowing $\skillmin$ to be infinitesimal.
The left plot in Figure \ref{fig:regretanalysis} provides the performance of our algorithms together with some algorithms for the DAS problem (see Section \ref{sec:further_exp} in the supplement for more information on these).

\begin{figure}
	\centering
	\subfigure
	{	\includegraphics[width=0.48\linewidth]{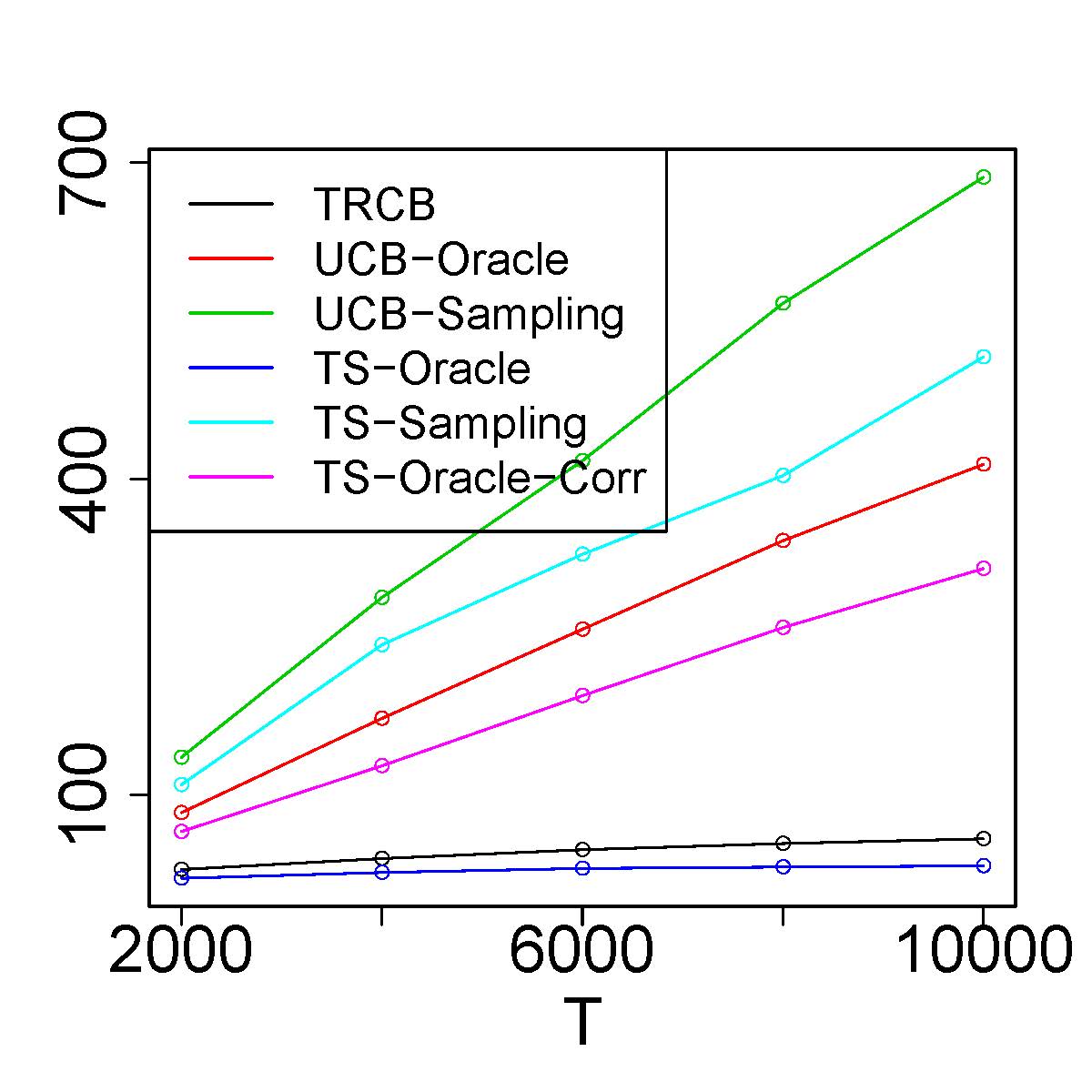}}
	\subfigure
	{  \includegraphics[width=0.48\linewidth]{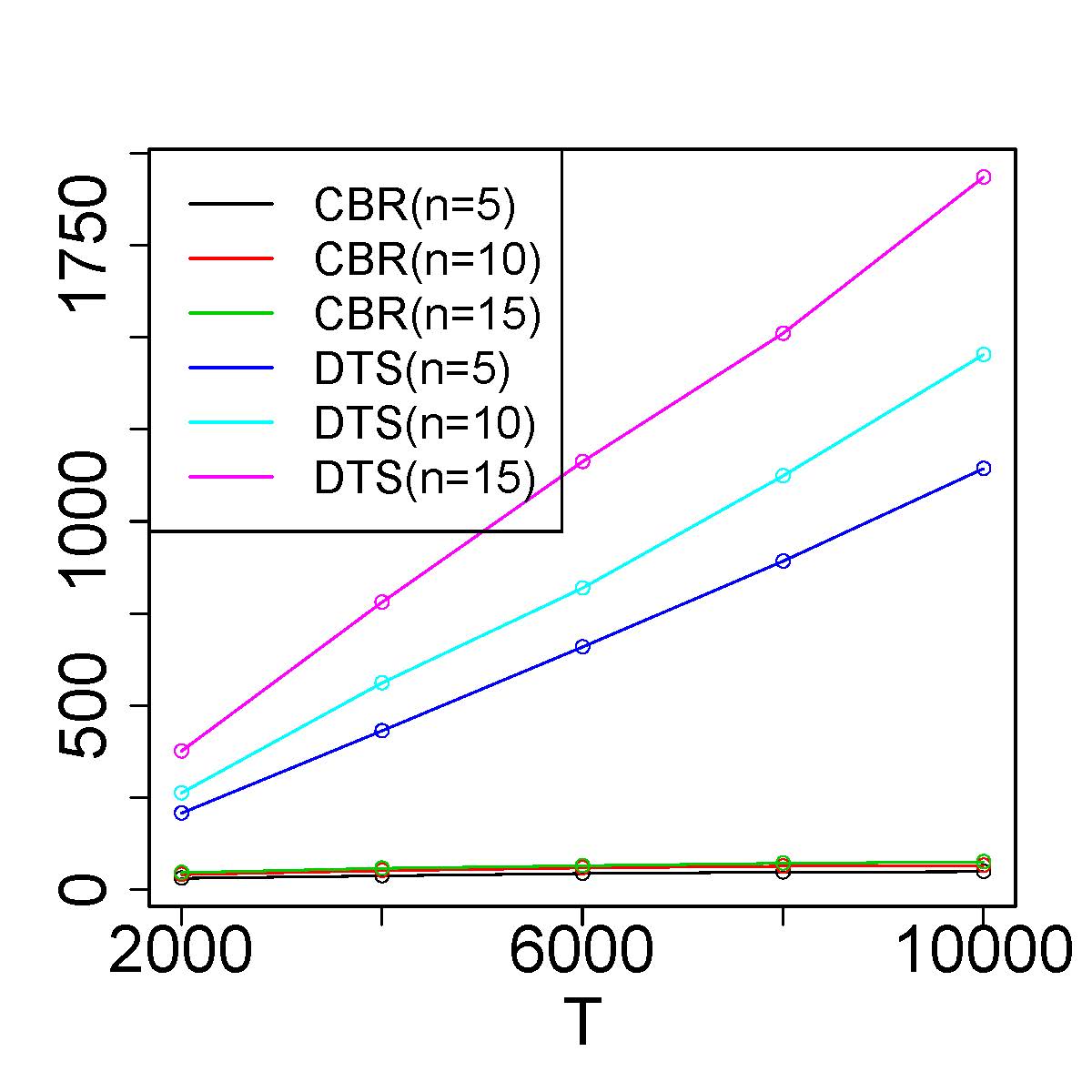}}
	\caption{
		{Left: Mean cumulative regret for 1000 runs of randomly generated restricted Pre-Bandit instances.
			%			for $\items=10$ and $\pullnumber=3.$ 
			Right: Mean cumulative regret for 1000 runs of randomly generated flexible Pre-Bandit instances.
			%			 for $\items\in\{5,10,15\}.$
	}}
	\label{fig:regretanalysis}
\end{figure}

For the algorithms of the DAS problem, the best arm is set to be the no-choice option, thereby putting (most of) them in the advantageous position of knowing a priori one element of the optimal subset.
Nevertheless, only TS-Oracle, with the advantage of knowing the best arm a priori, is able to slightly outperform TRCB in this scenario, whereas all other algorithms are distinctly outperformed by TRCB.

To explain this observation, recall our remark on UCB-like strategies in Section \ref{sec_TRCB_algo}. The UCB-based algorithms UCB-Oracle resp.\ UCB-Sampling as well as the UCB-like approximation of the variance of TS-Oracle-Corr tend to exclude arms with a low score from the suggested subset, even though they are contained in the optimal preselection.  
TS-Oracle and TS-Sampling, which do not use upper confidence bounds and include low score arms in the suggested subsets, are performing much better. The gap between these two TS algorithms shows how heavily the algorithms depend on the assumption that the  no-choice option corresponds to the highest scored arm, since we designed TS-Sampling such that, in each run, it samples once the best-arm from the top three arms according to an MNL model. 

In summary, this simulation confirms that the introduced (restricted) Pre-Bandit problem is indeed a new framework that differs from the DAS problem. A na\"ive  application of existing methods for the DAS problem is not suitable for this kind of problem. 

\subsection{Flexible Pre-Bandit Problem}
 Next, we investigate the empirical regret growth with varying time horizon $\timehorizon$ and varying  numbers of arms $\items$ for the flexible Pre-Bandit problem. In addition, we compared our algorithms with the Double Thompson Sampling (DTS) algorithm by \citet{wu2016double}, which is considered state-of-the-art for the dueling bandits problem with a small numbers of arms \citet{sui2017multi}. 

In the right picture of Figure \ref{fig:regretanalysis}, the results are displayed for the CBR resp.\ DTS algorithm on 1000 repetitions, respectively, with $n\in \{5,10,15\},$  $\timehorizon \in \{ i \cdot 2000  \}_{i=1}^5$, and $\sigma(x)= (1 \wedge x) 1_{[0,\infty)}(x).$ 
The score parameters are generated randomly as before.
%again drawn randomly from the $\items$-simplex, and then normalized such that $\skillmax=1,$ to ensure a fairer comparison between the regret of the algorithms.
%
It is clearly recognizable that  CBR  distinctly outperforms  DTS  in all scenarios, indicating that offering larger subsets is at least experimentally beneficial to find the best arm more quickly.

\section{Conclusion} \label{sec_conclusion}
In this paper, we have introduced the Pre-Bandit problem as a practically motivated and theoretically challenging variant of preference-based multi-armed bandits in a regret minimization setting. More specifically, we proposed two scenarios, one in which preselections are of fixed size and another one in which the size is under the control of the agent. For both scenarios, we derived lower bounds on the regret of algorithms solving these problems.
Moreover, we proposed concrete algorithms and analyzed  their performance theoretically and experimentally.

Our new framework suggests a multitude of conceivable paths for future work.
Most naturally, it would be interesting to analyze the Pre-Bandit problem under different assumptions on the user's choice behavior---despite being natural and theoretically justified \citep{mcfadden2001economic,train2009discrete}, the assumption of the PL model is relatively strong, and the question is to what extent it could be relaxed.
The main challenge surely lies in defining a sensible notion of regret, but an extension to the nested logit-model \citep{chen2018dynamicnested} or considering contextual information \citep{chen2018dynamiccontext} seems to be possible.
However, it is worth noting that our derived lower bounds on the regret based on expected utilities in the spirit of \eqref{defi_expected_utility} even hold for more general choice models, such as generalized random utility models \citep{walker2002generalized,train2009discrete}, which encompass the PL model.

Last but not least, like the related dynamic assortment selection problem studied in operational research, the motivation of our new framework stems from practical applications. Therefore, we are also interested in applying our algorithms to real-world problems, such as algorithm (pre-)selection already mentioned in the introduction.
In particular, the realm of algorithm selection seems to be a canonical candidate for our setting, as the decisions made are based on the noisy performance values of the algorithms, which justify a stochastic modeling of the process.

\subsubsection*{Acknowledgments}
The authors gratefully acknowledge financial support by the Germany Research Foundation (DFG). Moreover, the authors would like to thank the Paderborn Center for Parallel Computation (PC2) for the use of the OCuLUS cluster.
%Use unnumbered third level headings for the acknowledgments. All acknowledgments
%go at the end of the paper. Do not include acknowledgments in the anonymized
%submission, only in the final paper.

%\section*{References}


%\bibliography{example_paper}
\bibliographystyle{icml2020}
\bibliography{Literature}

\medskip

\appendix

 \newpage
 
 \onecolumn
 
\renewcommand{\L}{\mathscr L}

\renewcommand{\P}{\mathbb P}
 
\newcommand{\Oterm}{\mathcal{O}}

\icmltitle{Supplementary material to ''Preselection Bandits''}




%\begin{icmlauthorlist}
%		\icmlcorrespondingauthor{Cieua Vvvvv}{c.vvvvv@googol.com}
%\end{icmlauthorlist}

%\icmlaffiliation{abc}{Hogwarts}

%	\icmlcorrespondingauthor{}{}

%	\icmlcorrespondingauthor{Cieua Vvvvv}{c.vvvvv@googol.com}
%	\icmlcorrespondingauthor{Eee Pppp}{ep@eden.co.uk}

%\icmlkeywords{Machine Learning, ICML}
%
%\vskip 0.3in
%]

\section{Proofs of Theorems \ref{theorem_lower_bound_regret} and \ref{theorem_lower_bound_regret_flexible_size}} \label{sec:lower_bounds_proofs}

\newcommand{\subsetssmaller}{\mathbb A_{\pullnumber-1}}

\allowdisplaybreaks
For the proofs of Theorem \ref{theorem_lower_bound_regret} and Theorem \ref{theorem_lower_bound_regret_flexible_size} we need the following result on the Kullback-Leibler divergence of categorical probability distributions, which is Lemma 3 in \citet{chen2018note}.
Throughout the proofs we let $\gamma \in (0,\infty)$ be some arbitrary degree of preciseness.
\begin{lemma} \label{lemma_categrotical_kl_dist}
	Let $P \sim \mbox{Cat}(p_1,\ldots,p_m),$ i.e.\ $P(i)=p_i$ \, for $i=1,\ldots,m$ and $\sum_{i=1}^m p_i =1,$ as well as $Q\sim \mbox{Cat}(q_1,\ldots,q_m),$  such that $q_i = p_i + \varepsilon_i$  and $|\varepsilon_i|<1$ for any $i=1,\ldots,m.$
	Then,
	$$	\Kl{P,\, Q} \leq \sum_{i=1}^m \frac{\varepsilon_i^2}{q_i}.		$$
\end{lemma}
\noindent Moreover, we will need the following auxiliary result for all lower bound results.
\begin{lemma} \label{lemma_auxiliary_ineq_lower_bounds}
	For any $\delta\in(0,1)$ and any $\gamma \in (0,\infty)$ it holds that
	$$	1-(1-\delta)^{1/\gamma} \geq \min\{1,1/\gamma\} \, \delta.			$$
\end{lemma}
\begin{proof}
	First, consider the case $\gamma \in (0,1].$ 
	Then the assertion follows immediately as the left-hand side of the inequality is monotonically decreasing with $\gamma$ and for $\gamma=1$ the inequality is valid.

	Next, let us consider the case $\gamma \in (1,\infty).$
	The assertion is equivalent to showing that $f(x)=1-x\delta-(1-\delta)^x$ is non-negative for $x\in (0,1).$
	The first and second derivatives are respectively
	\begin{align*}
	f'(x) &= -\delta - \log(1-\delta) (1-\delta)^x, \\
	f''(x)&=  - \log(1-\delta)^2 (1-\delta)^x.
	\end{align*}
	By straightforward computations it can be shown that $f$ has a global maximum on $(0,1)$ at $x_{max}=\frac{\log\left(  \frac{-\delta}{\log(1-\delta)} \right)}{\log(1-\delta)}$ and $f$ is strictly increasing on $(0,x_{max})$ and strictly decreasing on $(x_{max},1).$
	As $\lim\limits_{x\to 0} f(x) = \lim\limits_{x\to 1} f(x) = 0,$ we can conclude the lemma.
\end{proof}

\begin{proof}[\sl Proof of Theorem \ref{theorem_lower_bound_regret}]
	We will use a similar proof technique as in \citet{chen2018note}.
	Let $\algo$ be some arbitrary algorithm suggesting the $\pullnumber$-sized subsets (preselections) $(S_t^\algo)_{t\in[\timehorizon]} \subset \subsets.$
	For a set $S \in \subsets$ we write $\skill_S=(\skill_S(1),\ldots,\skill_S(\items)) $ to denote the score parameter of the PL-model with components given by
	\begin{align*} 
	%	\label{parametrization_lower_restricted}
	%		
	\skill_S(i)	:= \begin{cases}
	1, & i\in S, \\
	1-\varepsilon, & i\notin S,
	\end{cases}		
	\end{align*}
	where $\varepsilon \in (0,1/2)$ is some hardness parameter specified below.
	%	
	%	In particular, $| S^{-}|=|S|=\pullnumber.$
	%	 
	Note that for any $S \in \subsets$ the score parameter $\skill_S$ is an element of the parameter space $\paramspace$.
	For sake of convenience, we will write $\P_S$ and $\E_S$ to express the law and expectation associated with the parameter $\skill_S,$ i.e., $\P_S = \P_{\skill_S}.$
	Recall the decomposition in (\ref{def_modeling_scores}) such that we have $\skill_S(i) = v_S(i)^\gamma$ for some suitable $v_S(i)$'s respectively and we define  $v_S$ in the same spirit as  $\skill_S.$

	First, for any $S,\tilde S \in \subsets$ with $S \neq \tilde S$ it holds that
	\begin{align} \label{help_ineq_lower_bound_restricted}
	\begin{split}
	\reward(S;v_S,\gamma)  - \reward(\tilde S;v_S,\gamma)	
	\geq 1 - \frac{(\pullnumber-1)+(1-\varepsilon)^{\frac{(1+\gamma)}{\gamma}}}{\pullnumber - \varepsilon}   
	&=  \frac{(1-\varepsilon) (1-( 1- \varepsilon)^{\frac{1}{\gamma}} ) }{  \pullnumber - \varepsilon}	  
	>  \frac{\min\{1,1/\gamma	\} \, \varepsilon}{2 \, \pullnumber} ,
	\end{split}
	\end{align}
	where we used for the last step $1-\varepsilon\geq 1/2$ and $\pullnumber - \varepsilon < \pullnumber$ as well as Lemma \ref{lemma_auxiliary_ineq_lower_bounds}.
	%		
	% Note that this result resembles Lemma 1 in \citet{chen2018note}.
	%	
	For $i\in [\items]$ let $N_i(t) = \sum_{s=1}^t 1_{ \{i \in S_s^\algo\}}$ denote the number of times an arm $i$ is part of a preselection  till time instance $t$ suggested by some algorithm $\algo.$
	In particular, write $N_i = N_i(\timehorizon),$ then (\ref{help_ineq_lower_bound_restricted}) implies
	\begin{align} \label{help_ineq_sec_lower_bound_restricted}
	\begin{split}
	\E_S \sum_{t=1}^\timehorizon \, \reward(S,\skill_S) - \reward(S_t^\algo,\skill_S) 
	&\geq  \frac{\min\{1,1/\gamma	\} \, \varepsilon}{2 \, \pullnumber}  \, \sum_{ i \notin S}  \E_S N_i.
	\end{split}
	\end{align}
	We can bound the expected regret  from below as follows
	\begin{align*}
	\sup_{\skill \in \paramspace} \, \E_{\skill} \regret(\timehorizon) 
	&\geq \sup_{S \in \subsets} \, \E_S \, \regret(\timehorizon) 
	=  \sup_{S \in \subsets} \E_S \sum_{t=1}^\timehorizon \, \reward(S,\skill_S) - \reward(S_t^\algo,\skill_S) \\
	&\geq \frac{1}{\binom{\items}{\pullnumber}} \sum_{S \in \subsets} \E_S \sum_{t=1}^\timehorizon \, \reward(S,\skill_S) - \reward(S_t^\algo,\skill_S) \\
	&\geq \frac{1}{\binom{\items}{\pullnumber}} \sum_{S \in \subsets} \sum_{i \notin S} \frac{\min\{1,1/\gamma	\} \, \varepsilon}{2 \, \pullnumber}  \, \E_S N_i  \\
	&=
	\frac{\min\{1,1/\gamma	\} \, \varepsilon}{2 }  \, \Big( \timehorizon - 
	\frac{1}{\pullnumber \, \binom{\items}{\pullnumber}} \sum_{S \in \subsets} \sum_{i \in S}  \, \E_S N_i \Big),  
	\end{align*}
	where we used for the last inequality (\ref{help_ineq_sec_lower_bound_restricted}) and for the last equality that
	$
	\timehorizon \, \pullnumber = 
	\sum_{i=1}^\items \E_S N_i = 
	\sum_{i \in S} \E_S N_i + \sum_{i \notin S} \E_S N_i.
	$
	Now,  using Formulas (5) -- (7) in \citet{chen2018note} and H\"older's resp.\ Jensen's inequality as in Section 3.4 of \citet{chen2018note} one obtains
	\begin{align*}
	\sup_{\skill \in \paramspace} \E_{\skill} \regret(\timehorizon)  
	\geq 	
	\frac{\min\{1,1/\gamma	\} \, \varepsilon\,
		\timehorizon}{2}  
 \Big( \frac{2}{3} - \sup_{S' \in \subsetssmaller}  \, \sqrt{  \sum_{i \in S'} \frac{\Kl{\P_{S'},\P_{S'\cup\{i\}}}}  {2(\items-\pullnumber+1)} } \ \Big).	
	\end{align*}
	The Kullback-Leibler divergence in the latter display can be dealt with by the following lemma which is proved below.
	\begin{lemma} \label{lemma_kl_distance_restricted_bandits}
		For each $ S' \in \subsetssmaller$ and $i \in  S'$ the following bound is true
		$$	\Kl{ \P_{ S'}, \P_{ S'  \cup \{i\}}		}	\leq \frac{22 \, \varepsilon^2 \, \E_{ S'} N_i}{\pullnumber} .	$$
	\end{lemma}
	\noindent With Lemma \ref{lemma_kl_distance_restricted_bandits} we have that for any $ S' \in \subsetssmaller$
	\begin{align*}
	\sqrt{  \sum_{i \in S'} \frac{\Kl{\P_{S'},\P_{S'\cup\{i\}}}}  {2(\items-\pullnumber+1)} }
	\leq \sqrt{ \frac{11\, \varepsilon^2 \, \timehorizon}{\items} },
	\end{align*}
	since $ \sum_{i \in S'} \E_{ S'} N_i \leq \timehorizon \pullnumber.$
	Thus, choosing $\varepsilon=\min( C
	\sqrt{\nicefrac{\items}{\timehorizon  }}, 1/2)$ for some appropriate small constant $C>0,$  independent of $\timehorizon,\items$ and $\pullnumber,$
	we obtain the assertion.
\end{proof}
%
%
%following the lines of proof in \citet{chen2018note} from the latter display (corresponding to Formula (4) in their paper) and using Lemma \ref{lemma_kl_distance_restricted_bandits} below instead of their Lemma 2 the proof is verified by setting 
%
%

%
\begin{proof} [\sl Proof of Lemma \ref{lemma_kl_distance_restricted_bandits}]
	Let $\tilde S \in \subsets$ be arbitrary.
	Then $\P_{ S'}( \cdot |  \tilde S)$ denotes the (categorical) probability distribution on the set $\tilde S$ parameterized by $\skill_{S'},$ i.e.,
	$$	\P_{ S'}( j |  \tilde S) = \begin{cases}
	\frac{\skill_{S'}(j)}{\sum_{k \in \tilde S} \skill_{S'}(k)}, & j \in \tilde S, \\
	0, & \mbox{else.}
	\end{cases}	$$
	If $i \notin \tilde S$ then $\Kl{ \P_{ S'}(\cdot | \tilde S), \P_{ S'  \cup \{i\}}(\cdot | \tilde S)		} = 0,$ as both distributions coincide in this case.
	Thus, we have the following bound
	\begin{align} \label{kl_ineq_help}
	\begin{split}
	\Kl{ &\P_{ S'}, \P_{ S'  \cup \{i\}}		} 
	\leq \Kl{ \P_{ S'}(\cdot | \tilde S, \, i \in \tilde S), \P_{ S'  \cup \{i\}}(\cdot | \tilde S, \, i \in \tilde S)		}    \E_{ S'} N_i  ,
	\end{split}	
	\end{align}
	as $i \in \tilde S$ happens  $\E_{ S'} N_i$ times in expectation.
	We proceed by bounding the Kullback-Leibler-divergence on the right-hand side of (\ref{kl_ineq_help}).
	Define $J_+ = |\tilde S \cap S'|,$ 
	and $J_- = |\tilde S \cap (S')^{\complement}|.$ 
	Since $\tilde S \in \subsets$ it holds that $J_+ + J_-  = \pullnumber.$
	With this, the categorical probabilities for $j \in \tilde S$ are given by
	\begin{align*}
	\begin{split}
	p_j &:= \P_{ S'}( j | \tilde S, \, i \in \tilde S) 	= \frac{\skill_{ S'   } (j)}{   J_+  +   (1-\varepsilon)   J_-   } , 
	\\
	q_j &:= \P_{ S'  \cup \{i\}}( j | \tilde S, \, i \in \tilde S) 	= 
	\frac{\skill_{ S'   } (j)}{   J_+ + 1  +   (1-\varepsilon)  ( J_-  -1 ) }.
	%			%	
	\end{split}
	\end{align*}
	For $j \neq i$ it holds that $\nicefrac{(p_j-q_j)^2}{q_j} \leq  \nicefrac{8 \varepsilon^2}{\pullnumber^3}.$
	We show this exemplary for the case, where $j\neq i$ and $j \in \tilde S \cap S',$ while the case $j\neq i$ and  $j \notin \tilde S \cap S',$  can be dealt with similarly.
	It holds that $		 J_+  +   (1-\varepsilon)   J_- = \pullnumber - \varepsilon J_- $ and $	 J_+ + 1  +   (1-\varepsilon)  ( J_-  -1 ) = \pullnumber + \varepsilon(1 - J_- ),$ so that 
	\begin{align*}
	p_j-q_j &=  \frac{ \varepsilon  }{  \big[\pullnumber - \varepsilon J_-  \big]\big[ \pullnumber + \varepsilon(1 - J_- )  \big]   }
	\end{align*}
	and with this
	\begin{align*}
	\frac{(p_j-q_j)^2}{q_j} &= \frac{ \varepsilon^2   }{   \big[\pullnumber - \varepsilon J_-  \big]^2\big[ \pullnumber + \varepsilon(1 - J_- )  \big]    }
	\leq   \frac{ 8 \varepsilon^2}{\pullnumber^3},
	\end{align*}
	as the terms inside the squared brackets are respectively greater than $\pullnumber/2,$ since $\varepsilon\in (0,1/2)$ and $|J_+|,| J_-|\leq \pullnumber.$ 
	If $j=i,$ then $\nicefrac{(p_j-q_j)^2}{q_j} \leq  \nicefrac{ 20 \varepsilon^2}{\pullnumber}.$
	%
	%	Again we show this exemplary for the case, where $i \in \tilde S \cap (S')^{-}.$
	%	
	Indeed, we have 
	\begin{align*}
	p_j-q_j &=  \frac{ \varepsilon \big(	1 - \pullnumber - \varepsilon(1  - J_-)		\big)    }{   \big[\pullnumber - \varepsilon J_-  \big]\big[ \pullnumber + \varepsilon(1 - J_- )  \big]   },
	\end{align*}  
	so that 
	\begin{align*}
	\frac{(p_j-q_j)^2}{q_j} &= \frac{ \varepsilon^2  \big(	1 - \pullnumber - \varepsilon(1  - J_-)		\big)^2 }{    \big[\pullnumber - \varepsilon J_-  \big]^2  \big[ \pullnumber + \varepsilon(1 - J_- )  \big]   }
	\leq  \frac{20  \varepsilon^2}{\pullnumber},
	\end{align*}
	since $\big(	1 - \pullnumber - \varepsilon(J_+ - J_-)		\big)^2 \leq 2 \pullnumber^2 + 2 \varepsilon^2 \pullnumber^2 \leq \nicefrac{5 \pullnumber^2}{2}. $
	Note that $|p_j-q_j|<1$ for each case, so that by using Lemma \ref{lemma_categrotical_kl_dist} and $\pullnumber\geq 2$ we obtain for Equation (\ref{kl_ineq_help}) that
	\begin{align*}
	\begin{split}
	\Kl{ \P_{ S'}, \P_{ S'  \cup \{i\}}		}  
	\leq \E_{ S'} N_i \cdot \Big(	 \frac{(\pullnumber-1) 8 \varepsilon^2}{\pullnumber^3} + \frac{20 \varepsilon^2}{\pullnumber}		\Big) \leq  \E_{ S'} N_i \cdot \frac{22 \varepsilon^2}{\pullnumber},
	\end{split}
	\end{align*}
	which completes the proof.
\end{proof}

%%%%%%%%%%%%%%%%%%%%%%%%%%%%%%%%%%%%%%%%%%%%%%%%%%%%%%%%%%%%%%%%%%%%%%%%%%%%%%%%%
%%%%%%%%%%%%%%%%%%%%%%%%%%%%%%%%%%%%%%%%%%%%%%%%%%%%%%%%%%%%%%%%%%%%%%%%%%%%%%%%%
%%%%%%%%%%%%%%%%%%%%%%%%%%%%%%%%%%%%%%%%%%%%%%%%%%%%%%%%%%%%%%%%%%%%%%%%%%%%%%%%%
%%%%%%%%%%%%%%%%%%%%%%%%%%%%%%%%%%%%%%%%%%%%%%%%%%%%%%%%%%%%%%%%%%%%%%%%%%%%%%%%%
%%%%%%%%%%%%%%%%%%%%%%%%%%%%%%%%%%%%%%%%%%%%%%%%%%%%%%%%%%%%%%%%%%%%%%%%%%%%%%%%%
%%%%%%%%%%%%%%%%%%%%%%%%%%%%%%%%%%%%%%%%%%%%%%%%%%%%%%%%%%%%%%%%%%%%%%%%%%%%%%%%%
%	Proof of flexible lower bound, gap-independent
%%%%%%%%%%%%%%%%%%%%%%%%%%%%%%%%%%%%%%%%%%%%%%%%%%%%%%%%%%%%%%%%%%%%%%%%%%%%%%%%%
%%%%%%%%%%%%%%%%%%%%%%%%%%%%%%%%%%%%%%%%%%%%%%%%%%%%%%%%%%%%%%%%%%%%%%%%%%%%%%%%%

\begin{proof} [\sl Proof of Theorem \ref{theorem_lower_bound_regret_flexible_size} (i)]
	Let $\algo$ be some arbitrary algorithm suggesting the subsets $(S_t^\algo)_{t\in[\timehorizon]} \subset \powerset.$
	%	For this proof we will use a different proof technique as for the proof before.
	%	
	In the following we define two problem instances characterized by score parameters $\skill^{(1)}, \skill^{(2)} \in \paramspace$ such that 
	\begin{align} \label{ineq_lower_bound_idea}
	\begin{split}
	\inf_{\algo} \, \Big\{\E_{\skill^{(1)}}^\algo\big(	\regret(\timehorizon) \big) + \E_{\skill^{(2)}}^\algo\big( \regret(\timehorizon)	\big) \Big\} 
	\geq 	\check C \sqrt{ \timehorizon},
	\end{split}
	\end{align}	
	where the infimum is taken over all terminating algorithms $\algo$ for the flexible Pre-bandit problem and $\check C>0$ is a constant similar to $C$ as in the assertion.
	The proof will be then complete due to 
	\begin{align*}
	\begin{split}
	\inf_{\algo} \sup_{\skill \in \paramspace} \E_\skill^\algo 	(\regret(\timehorizon)) 
	&\geq \frac12	\inf_{\algo} \, \Big\{ \E_{\skill^{(1)}}^\algo\big(	\regret(\timehorizon) \big) + \E_{\skill^{(2)}}^\algo\big( \regret(\timehorizon)	\big) \Big\}.	
	\end{split}
	\end{align*}
	%	  
	%
	%	In order to show \ref{ineq_lower_bound_idea}, let $\algo$ be an arbitrary algorithm for the flexible Pre-Bandit problem, which chooses $S_t^\algo \in \actionspace$ at time instance $t.$ 
	Thus, we proceed by showing (\ref{ineq_lower_bound_idea}).

	The observation at $t$ under the PL model assumption for the algorithm $\algo$ for an instance with score parameter $\skill$ is a random sample of $ P_{S_t^\algo,\skill} =	P_{S_t^\algo},$ where 
	\begin{align} \label{defi_categorical_distr}
	P_{S_t^\algo,\skill}(i) := \begin{cases}
	\frac{\skill_i}{\sum_{j \in S_t^\algo}  \skill_j}, & i\in S_t^\algo, \\
	0, &\mbox{else.}	
	\end{cases}
	\end{align}	
	The probability distribution with respect to $\algo$ and $\skill$ is denoted by $\P_{\skill}^\algo = \P_\skill$ and the corresponding expectation by $\E_\skill^\algo = \E_\skill.$
	The regret of $\algo$ for a PL model with parameter $\skill$ over the time horizon $\timehorizon$ is 
	%	denoted by $\regret(\timehorizon)$ and its expected value can be expressed as
	%
	\begin{align}  \label{repr_expected_regret_algo}
	\begin{split}
	\E_\skill^\algo  \big( \regret(\timehorizon) \big) 
	= \sum_{t=1}^\timehorizon \E_\skill^\algo \big(	\reward(\optsubset)  - \reward(S_t^\algo)		\big) 
	= \sum_{S \in \powerset}  \big( \reward(\optsubset)  - \reward(S)		\big) \E_\skill^\algo(N_S(T)),
	\end{split}
	\end{align}
	where $N_S(t) = \sum_{s=1}^t 1_{ \{S_s^\algo =	S\}} $ denotes the number of times the subset $S \in \powerset$ was suggested by $\algo$ till time $t \in [\timehorizon].$
	Note that we suppressed here the dependency of $\optsubset$ on $\skill$ in the notation for sake of brevity.
	Next, define
	\begin{align}
	\label{defi_score_param_lower_bound_flexible_size}
	\begin{split}
	\skill^{(1)} &:= \Big( 1, 1-\varepsilon, \skillmin, \ldots,\skillmin		\Big),  \qquad \mbox{and} \qquad
	\skill^{(2)} := \Big( 1-\varepsilon, 1,\skillmin, \ldots,\skillmin		\Big),
	\end{split}
	\end{align}
	where $\varepsilon \in (0,1-\skillmin)$ is a hardness parameter of the instances, which will be specified below.
	Note that both score parameters are elements of $\paramspace$ and only differ in two of the $\items$ components.
	It is easy to see that for any $S \in \powerset\backslash\{ 1  \}$ and $S' \in \powerset\backslash\{ 2  \}$ one has that
	\begin{align} \label{ineq_regret_lower_bounds_first_flex}
	\begin{split}
	  \reward(\{ 1  \},\skill^{(1)})  - \reward(S,\skill^{(1)})  	\geq \min\{1,1/\gamma\} \, \varepsilon, \qquad \mbox{and} \qquad
	& 	\reward(\{ 2  \},\skill^{(2)})  - \reward(S',\skill^{(2)}) 	\geq \min\{1,1/\gamma\} \, \varepsilon.
	\end{split}
	\end{align}
	Indeed, recall the decomposition of $\theta$ in (\ref{def_modeling_scores}) and obtain
	\begin{align*}
	\begin{split}
	\reward(\{ 1  \},\skill^{(1)})  - \reward(S,\skill^{(1)}) 
	\geq 1 - \frac{(1-\varepsilon)^{\frac{1+\gamma}{\gamma}}}{1-\varepsilon} 
	= 1- ( 1- \varepsilon)^{\frac{1}{\gamma}} 
	&\geq \min\{1,1/\gamma\} \, \varepsilon.
	\end{split}	
	\end{align*}
	The inequality $\reward(\{ 2  \},\skill^{(2)})  - \reward(S',\skill^{(2)}) \geq \min\{1,1/\gamma\} \, \varepsilon$ can be shown similarly.
	Clearly, the optimal subset to suggest for the problem instance characterized by $\skill^{(1)}$ is $\{ 1  \},$ while $\{ 2 \}$ is optimal for the other scenario associated with $\skill^{(2)}.$
	%\\
	Suggesting other subsets respectively results in an at least linear regret in the hardness parameter $\varepsilon.$
	By means of representation (\ref{repr_expected_regret_algo}) and (\ref{ineq_regret_lower_bounds_first_flex}) it follows that for $i=1,2$
	\begin{align*}
	\begin{split}
	\E_{\skill^{(i)}}^\algo  &\big( \regret(\timehorizon) \big) 
	> \P_{\skill^{(i)}}^\algo \big( N_{\{1\}}(\timehorizon) \leq \timehorizon/2 \big) \, \frac{ \min\{1,1/\gamma\} \, \varepsilon \timehorizon}{2  }.
	%	, \quad \mbox{and} \\
	%		
	%	&\E_{\skill^{(2)}}^\algo  \big( \regret(\timehorizon) \big) > \P_{\skill^{(2)}}^\algo \big( N_{\{1\}}(\timehorizon) > \timehorizon/2 \big) \,\frac{ \min\{1,1/\gamma\} \, \varepsilon \timehorizon}{2  }.
	%		
	\end{split}
	\end{align*} 
	%	
	%	where we abbreviated $N_{\{1\}}(\timehorizon) = N_{\{1,\ldots,\pullnumber\}} (\timehorizon;\algo).$
	%
	The inequalities are intuitive: if the optimal set $\{1\}$ for the parameter $\skill^{(1)}$ is suggested at most $\timehorizon/2$ times, then one obtains a regret of at least $ \varepsilon$  for the suggested sets in the remaining cases, which occur at least $\timehorizon/2$ times.
	Similarly, if the suboptimal set  $\{1\}$ for the problem instance with $\skill^{(2)}$ is suggested at least $\timehorizon/2$ times, then one obtains a regret of at least $ \varepsilon$ in all these timesteps.
	The latter display implies
	\begin{align*}
	\E_{\skill^{(1)}}^\algo\big(	\regret(\timehorizon) \big) + \E_{\skill^{(2)}}^\algo\big( \regret(\timehorizon)	\big)  
	&>  \frac{ \min\{1,1/\gamma\} \, \varepsilon \timehorizon}{2  } \big[ \P_{\skill^{(1)}}^\algo \big( N_{\{1\}}(\timehorizon) \leq \timehorizon/2 \big) + \P_{\skill^{(2)}}^\algo \big( N_{\{1\}}(\timehorizon) > \timehorizon/2 \big) \big] \\
	&\geq \frac{ \min\{1,1/\gamma\} \, \varepsilon \timehorizon}{2  } \exp\big[ -\Kl{ \P_{\skill^{(1)}}^\algo, \P_{\skill^{(2)}}}^\algo			\big],
	\end{align*}
	where we used in the last line a version of Pinkser's inequality, see Theorem 14.2 in \citet{lattimore2018bandit}. \\
	We proceed by analyzing the Kullback-Leibler distance in the latter display by means of Lemma \ref{lemma_categrotical_kl_dist} and the 
	following decomposition of the Kullback-Leibler divergence for the family of probability distributions $(\P_\skill^\algo)_{\skill \in \paramspace}$ which can be shown analogously to Lemma 15.1 in \citet{lattimore2018bandit}.
	\begin{lemma} \label{lemma_kl_decomp}
		Let $\skill,\skill' \in \paramspace,$ then
		$$	\Kl{\P_\skill^\algo, \, \P_{\skill'}^\algo} = \sum_{S \in \powerset} \, \E_\skill(N_S(T)) \, \Kl{P_{S,\skill}, P_{S,\skill'}}.		$$
	\end{lemma}
	Note that by definition of the score parameters in (\ref{defi_score_param_lower_bound_flexible_size}) it holds that $\Kl{P_{S,\skill^{(1)}}, P_{S,\skill^{(2)}}  } = 0 $ for any subset $S \in \powerset$ which does not contain $\{1\}$ and $\{2\},$ as both distributions are the same for such subsets. 
	%
	%	These are of order $\binom{\items -2}{\pullnumber }$ many subsets. 
	%
	For the remaining subsets $S'$, which are of order $\Oterm(2^{n-2})$ many, Lemma \ref{lemma_categrotical_kl_dist} yields $\Kl{P_{S',\skill^{(1)}}, P_{S',\skill^{(2)}}  } \leq   {2 \skillmin^{-1} \varepsilon^2}$ (cf.\ the proof of Lemma \ref{lemma_kl_distance_restricted_bandits}). 
	We distinguish two cases in the following.

	\emph{Case 1: $\timehorizon>2^\items-1.$}

	As $\sum_{S \in \powerset} \E_\skill (N_S(\timehorizon)) = \timehorizon$ for any $\skill \in \paramspace$ it is true that $\E_\skill (N_S(\timehorizon)) \leq \nicefrac{\timehorizon}{2^{\items}-1}$ for each $S \in \powerset$ by the pigeonhole principle.
	Thus, by means of Lemma \ref{lemma_kl_decomp} obtain 
	$	\Kl{\P_{\skill^{(1)}}, \P_{\skill^{(2)}}}	 \leq \widetilde C \, \timehorizon \varepsilon^2,		$
	where  $\widetilde C>0$ is some constant independent of $\items$ and $\timehorizon.$
	Hence,
	%\\
	\begin{align*}
	\E_{\skill^{(1)}}^\algo &\big(	\regret(\timehorizon) \big) + \E_{\skill^{(2)}}^\algo\big( \regret(\timehorizon)	\big) 
	\geq 	\frac{  \min\{1,1/\gamma\} \, \varepsilon \timehorizon}{2  }  \exp\Big(-\widetilde C \timehorizon \varepsilon^2 \Big).	
	\end{align*}

	\emph{Case 2: $\timehorizon\leq 2^\items-1.$}
	
	In this case, note that there are at least $2^\items-1 - \timehorizon$ many zero summands in $\sum_{S \in \powerset} \E_\skill (N_S(\timehorizon))$ as the sum equals $\timehorizon.$
	Therefore, similar to the case before obtain by means of Lemma \ref{lemma_kl_decomp}   that
	$	\Kl{\P_{\skill^{(1)}}, \P_{\skill^{(2)}}}	 \leq \widetilde C   \timehorizon  \varepsilon^2$ for some constant $\widetilde C>0$ independent of $\items$ and $\timehorizon.$ 
	Consequently,
	\begin{align*}
	\E_{\skill^{(1)}}^\algo &\big(	\regret(\timehorizon) \big) + \E_{\skill^{(2)}}^\algo\big( \regret(\timehorizon)	\big) 
	\geq 	\frac{  \min\{1,1/\gamma\} \, \varepsilon \timehorizon}{2  }  \exp\Big(-\widetilde C \timehorizon \varepsilon^2 \Big).	
	\end{align*}

	By choosing in both cases $\varepsilon = \min (\bar{ C} \,  \sqrt{ 1/\timehorizon}, \, 1- \skillmin)$ for some appropriate constant $\bar C >0$  we obtain the assertion with some constants $C,C'>0$ which are independent of $\timehorizon,\pullnumber$ and $\items.$
	%
	%	Note that $\varepsilon < 1- \skillmin$ is guaranteed by choosing $\bar C \in (0,1)$ small enough.
	%
	%	
\end{proof}

%%%%%%%%%%%%%%%%%%%%%%%%%%%%%%%%%%%%%%%%%%%%%%%%%%%%%%%%%%%%%%%%%%%%%%%%%%%%%%%%%
%%%%%%%%%%%%%%%%%%%%%%%%%%%%%%%%%%%%%%%%%%%%%%%%%%%%%%%%%%%%%%%%%%%%%%%%%%%%%%%%%
%%%%%%%%%%%%%%%%%%%%%%%%%%%%%%%%%%%%%%%%%%%%%%%%%%%%%%%%%%%%%%%%%%%%%%%%%%%%%%%%%
%%%%%%%%%%%%%%%%%%%%%%%%%%%%%%%%%%%%%%%%%%%%%%%%%%%%%%%%%%%%%%%%%%%%%%%%%%%%%%%%%
%%%%%%%%%%%%%%%%%%%%%%%%%%%%%%%%%%%%%%%%%%%%%%%%%%%%%%%%%%%%%%%%%%%%%%%%%%%%%%%%%
%%%%%%%%%%%%%%%%%%%%%%%%%%%%%%%%%%%%%%%%%%%%%%%%%%%%%%%%%%%%%%%%%%%%%%%%%%%%%%%%%
%	Proof of flexible lower bound, gap-dependent
%%%%%%%%%%%%%%%%%%%%%%%%%%%%%%%%%%%%%%%%%%%%%%%%%%%%%%%%%%%%%%%%%%%%%%%%%%%%%%%%%
%%%%%%%%%%%%%%%%%%%%%%%%%%%%%%%%%%%%%%%%%%%%%%%%%%%%%%%%%%%%%%%%%%%%%%%%%%%%%%%%%

\begin{proof} [\sl Proof of Theorem \ref{theorem_lower_bound_regret_flexible_size} (ii)]
	
	For the gap-dependent lower bound we will make use of the following result, which is Lemma 1 in \citet{kaufmann2016complexity}.
	
	\begin{lemma} \label{lemma_kaufmann_et_al}
		Let $\nu$ and $\nu'$ be two MAB models with $\items$ arms and $\nu_i$ resp.\ $\nu_i'$ denotes the reward distribution for arm $i\in [\items]$ respectively.
		Let $A_t$  denote the arm played at round $t$ and $R_t$ be the corresponding observed reward.
		Moreover, let $\F_t=\sigma(A_1,R_1,\ldots,A_t,R_t)$ be the sigma algebra generated by the observations till time instance $t.$
		Suppose that $\nu_i$ and $\nu_i'$ are mutually absolutely continuous for each $i\in [\items],$ then it holds that  
		$$		\sum_{i\in[\items]} \E_{\nu}[N_i(T)] \Kl{\nu_i,\nu_i'} \geq d(\E_{\nu}(\mathcal E) ,\E_{\nu'}(\mathcal E) )			$$
		for any $\F_T$-measurable random variable $\mathcal E.$ 	
		Here, $d(x,y)= x \log(\nicefrac{x}{y}) + (1-x) \log(\nicefrac{(1-x)}{(1-y)})$ and $N_i(t)=\sum_{s=1}^t 1_{i_s^\algo=i}$ is the number of times an algorithm $\algo$ plays arm $i$ till time instance $t.$
	\end{lemma}
	
	\newcommand{\skillone}{\skill^{(1)}}
	\newcommand{\skilli}{\skill^{(i)}}
	
	\noindent In the following, we will adapt the proof of Theorem 3 in \cite{saha2019regret} to our case, which boils down to incorporating our (different) notion of regret into their proof.
	\medskip
	
	\noindent To make use of Lemma \ref{lemma_kaufmann_et_al} we embed the flexible Pre-Bandit problem into a classical MAB problem by considering each subset $S \in \powerset$ as an arm.
	Moreover, we define the score parameters 
	\begin{align} \label{defi_score_para_gap_dep_lower}
	\begin{split}
	\skillone &=(1,1-\Delta,\ldots,1-\Delta), \\ 
	\skilli &= \big(1,1-\Delta,\ldots,1-\Delta,1+\varepsilon,1-\Delta,\ldots, 1-\Delta\big), \quad i=2,\ldots,\items,
	\end{split}
	\end{align}
	where $\Delta \in (0,1-\skillmin)$ and $\varepsilon>0$ and the $i$-th component of $\skilli$ is $1+\varepsilon.$ 
	%
	%Note that we 
	%
	For $\skill\in\paramspace$ and $S\in \powerset$ let $P_{S,\skill}$ denote the categorical distribution as in (\ref{defi_categorical_distr}).
	Using Lemma \ref{lemma_kaufmann_et_al} with  $\nu_{S} = P_{S,\skillone}$ and $\nu_{S}'=P_{S,\skilli}$ for $i\neq 1$ for any $S\in \powerset$ as the reward distributions of the arms and the $\F_T$-measurable random variable $\mathcal E = \nicefrac{N_{\{i\}}(\timehorizon)}{\timehorizon},$ one has that
	\begin{align} \label{eq_lower_bound_gap_dep_main}
	\begin{split}
	\sum_{S \in \powerset} \E_{\skillone}[N_S(\timehorizon)] \, \Kl{P_{S,\skillone},P_{S,\skilli}} 
	&=\sum_{S \in \powerset} \E_{\skillone}[N_S(\timehorizon)] \, \Kl{\nu_S,\nu_S'} \\
	&\geq d(\E_{\skillone}[\nicefrac{N_{\{i\}}(\timehorizon)}{\timehorizon}] ,\E_{\skilli}[\nicefrac{N_{\{i\}}(\timehorizon)}{\timehorizon}] ).	
	\end{split}
	\end{align}
	Now, since $d(x,y) \geq (1-x)\log(\nicefrac{1}{(1-y)}) - \log(2) $ derive that 
	\begin{align*}
	d(\E_{\skillone}[\nicefrac{N_{\{i\}}(\timehorizon)}{\timehorizon}] ,\E_{\skilli}[\nicefrac{N_{\{i\}}(\timehorizon)}{\timehorizon}] ) 
	&\ \geq \Big(	1 - \frac{\E_{\skillone}[N_{\{i\}}] }{\timehorizon}		\Big) \log\Big(	\frac{\timehorizon}{\timehorizon- \E_{\skilli}[N_{\{i\}}]}		\Big) - \log(2).
	\end{align*}
	As we assume that $\algo$ is a no-regret algorithm, we have that $\E_{\skillone}[N_{\{i\}}] = o(\timehorizon^\alpha)$ and $\timehorizon- \E_{\skilli}[N_{\{i\}}] = \E_{\skilli}[\sum_{S\in \powerset, S\neq \{i\} }  N_{\{i\}}] = o(\timehorizon^\alpha) $ for some $\alpha\in(0,1].$
	Hence, by dividing the latter display by $\log(\timehorizon)$ and by considering $\timehorizon\to\infty$ one obtains
	\begin{align*}
	\lim_{\timehorizon \to \infty} \frac{ d(\E_{\skillone}[\nicefrac{N_{\{i\}}(\timehorizon)}{\timehorizon}] ,\E_{\skilli}[\nicefrac{N_{\{i\}}(\timehorizon)}{\timehorizon}] )}{\log(\timehorizon)} 
	&\geq 	\lim_{\timehorizon \to \infty}  \frac{ 1}{\log(\timehorizon)} \Big(	1 -  o(\timehorizon^{\alpha-1}) 	\Big) \log\Big( \frac{T}{ o(\timehorizon^{\alpha})} 		\Big) - \frac{ \log(2)}{\log(\timehorizon)} \\
	&\geq (1-\alpha).
	\end{align*}
	Hence, dividing (\ref{eq_lower_bound_gap_dep_main}) by $\log(\timehorizon)$ and considering the limit case obtain 
	\begin{align} \label{eq_lower_bound_gap_dep_main_sec}
	\begin{split}
	\lim_{\timehorizon \to \infty}  \frac{ 1}{\log(\timehorizon)}  \sum_{S \in \powerset} \E_{\skillone}&[N_S(\timehorizon)] \, \Kl{P_{S,\skillone},P_{S,\skilli}} 
	\geq (1-\alpha).
	\end{split}
	\end{align}
	The Kullback-Leibler divergence in (\ref{eq_lower_bound_gap_dep_main_sec}) can be bounded by the following lemma, which first statement can be shown by following the lines of display (2) in \citet{saha2019regret}, while the second statement is straightforward from the choice of the score parameters in (\ref{defi_score_para_gap_dep_lower}).
	\begin{lemma} \label{lemma_KL_gap_dependent_aux}
		For each $i\neq 1$ it holds that
		$$	 \Kl{P_{S,\skillone},P_{S,\skilli}} \leq \frac{(\Delta+\varepsilon)^2}{(1-\Delta)|S|(1+\varepsilon)}.	$$
		Moreover, if $i \notin S$	or if $|S|=1,$ then $$ \Kl{P_{S,\skillone},P_{S,\skilli}} = 0.	$$
	\end{lemma}
	\noindent Using Lemma \ref{lemma_KL_gap_dependent_aux} we can derive from (\ref{eq_lower_bound_gap_dep_main_sec}) by multiplying with $\nicefrac{(1-\Delta)^2}{(\Delta+\varepsilon)} $ that
	\begin{align*} 
	%\label{eq_lower_bound_gap_dep_main_sec}
	\begin{split}
	\lim_{\timehorizon \to \infty}  \frac{ 1}{\log(\timehorizon)} \sum_{\stackrel{S \in \powerset\backslash\{i\},}{i \in S}  } \, \frac{ \E_{\skillone}[N_S(\timehorizon)] (1-\Delta)(\Delta+\varepsilon)}{|S|(1+\varepsilon)} 
	&\geq \frac{(1-\Delta)^2}{(\Delta+\varepsilon)} \, (1-\alpha).
	\end{split}
	\end{align*}
	Summing over $i\in\{2,\ldots,\items\}$ and taking the limit $\varepsilon \to 0$ in the latter display leads to 
	\begin{align} \label{eq_lower_bound_gap_dep_main_third}
	\begin{split}
	\lim_{\timehorizon \to \infty}  \frac{ 1}{\log(\timehorizon)}  \sum_{ i =2}^\items \sum_{  \underset{i \in S}{S \in \powerset\backslash\{i\},}  } \E_{\skillone}[N_S(\timehorizon)] \, \frac{(1-\Delta)\Delta}{|S|} 
	&\geq \frac{(1-\Delta)^2}{\Delta} \, (\items-1) \, (1-\alpha).
	\end{split}
	\end{align}
	Next, we bound the cumulative regret in (\ref{defi_regret}) for any algorithm $\algo$ for the flexible Pre-Bandit problem from below. For this purpose recall the decomposition in (\ref{def_modeling_scores}) and denote the $i$th component of $\skillone$ by $\skillone_i$ and let $v_i^{(1)}= (\skillone_i)^{1/\gamma}.$
	Hence, we get
	\begin{align*} 
	\E_{\skillone} \big( \regret(\timehorizon) \big) 
	&=  \sum_{ t =1}^{\timehorizon} \E_{\skillone}\big(\reward(\optsubset) - \reward(S_t^\algo)\big) \\
	&=  \sum_{ t =1}^{\timehorizon} \E_{\skillone}\Big(  v_1^{(1)}  - \frac{\sum_{ i \in S_t^\algo} \big(v_i^{(1)}\big)^{1+\gamma} }{\sum_{ i \in S_t^\algo} \big(v_i^{(1)}\big)^\gamma }\Big) \\
	&=   \E_{\skillone} \Big( \sum_{ t =1}^{\timehorizon} \sum_{S \in \powerset} 1_{S_t^\algo=S}  \,  \frac{\sum_{ i =2}^\items 1_{i\in S} \, \big(v_i^{(1)}\big)^\gamma  (v_1^{(1)}-v_i^{(1)})}{\sum_{ i =1}^\items 1_{i\in S} \big(v_i^{(1)}\big)^\gamma}\Big) \\
	&\geq   \min\{1,1/\gamma\}  \E_{\skillone} \Big( \sum_{ t =1}^{\timehorizon} \sum_{S \in \powerset} 1_{S_t^\algo=S}  \, \sum_{ i =2}^\items   \frac{ 1_{i\in S} \, (1-\Delta) \Delta }{|S|}\Big) \\
	&=   \min\{1,1/\gamma\}  \sum_{ i =2}^\items \sum_{S \in \powerset} \E_{\skillone} \Big( \sum_{ t =1}^{\timehorizon}  1_{S_t^\algo=S}  \Big) \,  1_{i\in S}  \,\frac{  (1-\Delta) \Delta }{|S|} \\
	&=  \min\{1,1/\gamma\}  \sum_{ i =2}^\items \sum_{S \in \powerset, \, i\in S} \E_{\skillone} (N_S[\timehorizon]) \,   \,\frac{  (1-\Delta) \Delta }{|S|},
	\end{align*}
%	\normalsize
	%
	where we used Lemma \ref{lemma_auxiliary_ineq_lower_bounds} for the inequality together with $\sum_{ i =1}^\items 1_{i\in S} \big(v_i^{(1)}\big)^\gamma \leq |S|.$
	With this obtain from (\ref{eq_lower_bound_gap_dep_main_third}) that if $\algo$ is a no-regret algorithm, then
	\begin{align*} 
	%\label{eq_lower_bound_gap_dep_main_fourth}
	%\begin{split}
	%
	\lim_{\timehorizon \to \infty}  \frac{ 1}{\log(\timehorizon)} &\E_{\skillone} \big( \regret(\timehorizon) \big) \geq \frac{ \min\{1,1/\gamma\} \cdot (1-\alpha)(1-\Delta)^2}{\Delta} \, (\items-1),
	%	
	%
	%\end{split}
	\end{align*}
	which concludes the proof as $\Delta$ corresponds to $\min_{i\notin \optsubset} \skillmax - \skill_i$ for $\skill=\skillone$ and $(1-\alpha)(1-\Delta)^2$ is some constant independent of $\timehorizon$ and $\items.$
\end{proof}

\section{Proof of Theorem \ref{theorem_upper_bound_regret}} \label{sec:proof_TRCB}

We start by introducing the notation  for the rest of the proof and recalling the main terms of the TRCB algorithm. 
Thereafter we give an outline of the proof, before deriving the details.

\subsection{Notation and relevant terms}

\newcommand{\wiA}{w_{i,A}}
\newcommand{\wi}{w_{i}}
\newcommand{\wjA}{w_{j,A}}
\newcommand{\wj}{w_{j}}
\newcommand{\wJA}{w_{J,A}}
\newcommand{\oddijhat}{\hat O_{i,j}}

\newcommand{\timesi}{T_i}
\newcommand{\timesj}{T_j}

\newcommand{\wiJ}{w_{i,J}}
\newcommand{\wJi}{w_{J,i}}
\newcommand{\wij}{w_{i,j}}
\newcommand{\wijbar}{\overline{w}_{i,j}}
\newcommand{\wiJbar}{\overline{w}_{i,J}} 
\newcommand{\cij}{c_{i,j}}
\newcommand{\ciJ}{c_{i,J}}

\newcommand{\Ft}{\F_t}

\newcommand{\wji}{w_{j,i}}

Throughout $(S_t)_{t=1,\ldots,\timehorizon}$ denotes the suggested subsets (the preselections) of the TRCB algorithm  at each  time instance respectively and $(i_t)_{t=1,\ldots,\timehorizon}$ the corresponding decisions of the selector, i.e., $i_t \in S_t.$
Furthermore, let  $\gamma \in (0,\infty)$ be some arbitrary degree of preciseness.
Next, we clarify the notation as well as recall the main terms emerging in the TRCB algorithm.
We define
\begin{align} \label{defi_number_of_picks}
\wij(t) := \begin{cases}
\sum_{s=1}^{t-1} 1_{	\{  i_s= i , \, \{i,j\} \in S_s	\}	}, & t > 1, \\
0, & t=1,
\end{cases}
\end{align}
to denote the number of times $i$ has been picked by the selector till time instance $t,$ when $i$ and $j$ were both part of the preselection, while	$\wijbar(t)	:= \wij(t)  + \wji(t)$ is the number of times either $i$ or $j$ was picked till time instance $t,$ when both were part of the preselection.
The relative scores in (\ref{defi_relative_score}) are estimated in time instance $t$ by
\begin{align} \label{defi_estoddij}
\estoddij(t) := \begin{cases}
\frac{\wijbar(t) }{ \wji(t) } -1, & \wji(t) \neq 0, \\
\skillmin, & \mbox{else,}			
\end{cases} \quad i,j \in [\items].
\end{align}
The arm with the most picks till time instance $t$ is 
\begin{align} \label{defi_most_picked_arm}
J := J(t)	=  \argmax{i \in [\items]} \# \{w_{i,j}(t) \geq w_{j,i}(t)	 \ | \ j \neq i		\}.		
\end{align}	 
Note that in the following we will suppress its dependency on the time instance $t$ in the notation.
The (thresholded) random value inside the confidence region of $\estoddiJ(t)$ is
%
%\begin{align}
%%	
%\oddiJtrcb(t) := \begin{cases} 
%1, & i= J, \\
%%		
%\min \big(  \skillmin\inv  , \max\big( \estoddiJ(t) + \Cshrink \, \beta_i(t)				, \skillmin \big)					\big), & \mbox{else,}			
%\end{cases}
%%	
%\end{align}
\begin{align*}
\oddiJtrcb(t) =    \skillmin\inv  \wedge \Big( \big( \estoddiJ + \Cshrink \, \beta_i(t)	\big)	\vee \skillmin \big)					\Big),
\end{align*}
for $i \neq J$ and $\oddiJtrcb(t)=1$ for $i=J,$
where 
\begin{align*}
\beta_i(t) &\sim \text{Unif} [ -\ciJ(t), \ciJ(t)	] , \qquad 
\ciJ(t)	= \sqrt{\frac{ 32 \,   \log(  \pullnumber \, t^{3/2})}{ \, \skillmin^4 \, \wiJbar(t) }},	 
\end{align*}
and $\Cshrink\in(0,1/2)$ is some finite constant.
Note that the $\beta_i$'s are mutually independent.
Recall the definition of regret for any time instance $t\in [\timehorizon]$ in (\ref{defi_regret}).
Due to (\ref{defi_expected_utility_with_rel_score}) we will consider the following scaled regret per time
\begin{align} \label{defi_regret_scaled}
\regpertime(t) := \rewardtilde(\optsubset;O_J,\gamma) - \rewardtilde(S_t;O_J,\gamma) = v_J\inv\, \rpreselect(S_t).
\end{align}
Finally, let
$\Ft$ denote the $\sigma$-algebra  generated by $S_1,i_1,\ldots,S_{t-1},i_{t-1}$ in time instance $t,$ with $\F_1$ being the trivial $\sigma$-algebra.
Note that $J(t)$ as well as $\wiJbar(t)$ resp.\ $\ciJ$ are $\Ft$-measurable for any $t\in [\timehorizon].$

\subsection{Outline of the proof} \label{subsec:outline_proof_TRCB}

\newcommand{\At}{A_t}
\newcommand{\Bt}{B_t}
\newcommand{\skillmintilde}{\tilde \skill_{min}}
\newcommand{\const}{\mathrm{const}}

%We will show the assertion for  an arbitrary score parameter $\skill \in[\skillmin,1]^\items.$
%%
%Thus, we show
%%
%\begin{align} \label{theorem_sufficient_assertion}
%%	
%\sup_{\skill \in [\skillmin,1]^\items}	\E_{\skill} \, \regret(\timehorizon) = O(\sqrt{\items \, \timehorizon\, \log(\timehorizon)}).
%%	
%\end{align}
%\\
%for some constant $\skillmin,$ which is \emph{independent} of $\items.$
% 

%
We introduce in the following the core lemmas to prove the main result, which will be gradually verified in the next subsection.
For  $t\in[\timehorizon]$  define
\begin{align}
\At := \{	\exists i \in S_t \cup \optsubset : \ 	| \oddiJtrcb(t) - \oddiJ		| > \ciJ(t)		\}.
%	
%	\Bt := \{ \forall i \in S_t \cup \optsubset : \	 \wJi(t) \geq 1		 	\}.
%	
\end{align}
Thus, $\At$ is the event on which the estimates for the relative scores  for arms in the chosen preselection and the optimal preselection with respect to the currently most winning arm $J$ are not close enough to their actual relative score, where the length of the confidence region $\ciJ(t)$ determines how closeness is to be understood in this case.

%
%The event $\Bt$ ensures that observations for a reliable computation of $\estoddiJ$ (c.f.\ (\ref{defi_estoddij})) are available.
%
As a consequence, one wishes that the probability that $\At$ happens is sufficiently small.
%, while the probability of $\Bt$ should be large.
%
The following lemma establishes this requirement.
\begin{lemma} \label{lemma_at_bt_event}
	It holds that
	$$		\E_\skill \big(	1_{\{\At \}} \, | \Ft		\big)	= \Oterm\Big(\sqrt{\frac{\log(t)}{t}}\Big),
	% \quad 	\P\big(	 \Bt^\complement \, 		\big)	= \Oterm\big(\nicefrac{1}{\sqrt{t}}\big), 
	$$
	where the constant in the $\Oterm$-term is independent of $\timehorizon,\pullnumber$ and $\items.$
	In particular,  for any $i\in S_t \cup \optsubset,$
	\begin{align*}
	\E_\skill  &\Big[ \E_\skill \big(  	| \oddiJtrcb(t) - \oddiJ		| \, 1_{\At^\complement} \, | \Ft  \big) \Big] 
	\leq  \E_\skill \big[   \ciJ(t) \big]  = \E_\skill \Big[\sqrt{\frac{ 32 \,   \log(  \pullnumber \, t^{3/2})}{ \, \skillmin^4 \, \wiJbar(t) }}\Big] .	
	\end{align*}
\end{lemma}
\noindent Next, we investigate the deviation between the scaled regret per time (cf.\ (\ref{defi_regret_scaled})) and its empirical counterpart.
For this purpose, note that
\begin{align} 
\label{motivation_regret_deviation}
\begin{split}
\regpertime(t) &= \rewardtilde(\optsubset;O_J,\gamma) - \rewardtilde(S_t;O_J,\gamma)  \\
%		\\
%	
%		&= \rewardtilde(\optsubset;O_J) - \rewardtilde(\optsubset;\hat O_J^{\mathrm{TRCB}})		
%		+ \rewardtilde(\optsubset;\hat O_J^{\mathrm{TRCB}})	- \rewardtilde(S_t;\hat O_J^{\mathrm{TRCB}})  
%	
%			+	 \rewardtilde(S_t;\hat O_J^{\mathrm{TRCB}})	 -	\rewardtilde(S_t;O_J) \\
%	
%		&
&\leq  \big[\rewardtilde(\optsubset;O_J,\gamma) - \rewardtilde(\optsubset;\hat O_J^{\mathrm{TRCB}},\gamma)\big]  +	 \big[\rewardtilde(S_t;\hat O_J^{\mathrm{TRCB}},\gamma)	 -	\rewardtilde(S_t;O_J,\gamma)\big],
\end{split}
\end{align} 
since $\rewardtilde(\optsubset;\hat O_J^{\mathrm{TRCB}})	- \rewardtilde(S_t;\hat O_J^{\mathrm{TRCB}}) \leq 0,$ by the definition of $S_t$ in line 11 of the TRCB algorithm.
Here, we abbreviated $\hat O_J^{\mathrm{TRCB}}=(\hat O_{1,J}^{\mathrm{TRCB}},\ldots,\hat O_{\items,J}^{\mathrm{TRCB}}).$
\medskip 

\noindent The following lemma gives a bound on the ratio between the two terms in squared brackets on the right-hand side of the latter display.
\begin{lemma} \label{lemma_bound_regret_deviation}
	Conditioned on $\Ft$ there exist constants $C_1,C_2>0$ depending if at all on $\skillmin$ and $\gamma$ (but independent of $\timehorizon,\pullnumber$ and $\items$) such that on $\At^\complement$ it holds with probability at least $1-\frac{C_1}{\sqrt{t}}$ that
	$$		\frac{ \rewardtilde(\optsubset;O_J,\gamma) - \rewardtilde(\optsubset;\hat O_J^{\mathrm{TRCB}},\gamma)}{\rewardtilde(S_t;\hat O_J^{\mathrm{TRCB}},\gamma)	 -	\rewardtilde(S_t;O_J,\gamma)}	\leq C_2.	$$
	Moreover, $C_2$ is of the form $\const \cdot \skillmin^{-2(3+\gamma)}.$
	In particular, with probability at least $1-\frac{C_1}{\sqrt{t}}$
	\begin{align*} 
	%	\label{conclusion_regret_deviation}
	%	
	&\E_\skill \big(\regpertime(t) 1_{\At^\complement} \, | \Ft \big) 
	\leq (C_2+1) \E_\skill \big( \big|\rewardtilde(S_t;\hat O_J^{\mathrm{TRCB}})	 -	\rewardtilde(S_t;O_J)\big| \, 1_{\At^\complement} \, | \Ft \big) .
	\end{align*}
\end{lemma}
\noindent The next pillar of the proof is to transfer the high concentration of $\hat O_J^{\mathrm{TRCB}}$ around $O_J$  to a high concentration of the corresponding utilities $\rewardtilde$ by exploiting its Lipschitz smoothness.

\begin{lemma} \label{lemma_lipschitz_cont_reward}
	%	
	%	There exists a constant $C>0$ depending if at all on $\skillmin$ and $\gamma$ (but independent of $\timehorizon,\pullnumber$ and $\items$) such that  
	For any $t\in[\timehorizon] $ 
	\begin{align*}
	&\big|\rewardtilde(S_t;\hat O_J^{\mathrm{TRCB}},\gamma)	 -	\rewardtilde(S_t;O_J,\gamma)\big| 
	\leq \frac{\max \{\skillmin^{(\gamma-1)/(\gamma)},\skillmin^{(1-\gamma)/(\gamma)} \} }{ \gamma } \sum_{i \in S_t} |\hat O_{i,J}^{\mathrm{TRCB}}(t)  - O_{i,J} | . 
	\end{align*}
\end{lemma}
\noindent Finally, an upper bound on the expected length of the confidence regions over time  (that is basically $(\wiJbar(t))^{-1/2}$) has to be verified.
\begin{lemma} \label{lemma_length_conf_region}
	The following statement is valid,
	$$	\sum_{t \in \timehorizon }  \E_{\skill} \big(\sum_{i \in S_t}	\nicefrac{1}{\sqrt{\wiJbar(t)}} 		\big) \leq 4 \sqrt{\timehorizon \items}.		$$

\end{lemma}

\paragraph{Conclusion: Proof of Theorem \ref{theorem_upper_bound_regret}}

Given these core lemmas, we are now in the position to verify Theorem \ref{theorem_upper_bound_regret}.
\medskip

\noindent Let $\skill \in \paramspace$ and $\timehorizon\in \N$ with $\timehorizon>n,$ then since $r(S_t) \leq  \regpertime(t),$ for any  $t \in [\timehorizon],$ we have
\begin{align*}
\E_{\skill}[ \regret(\timehorizon) ]
&\leq  \sum_{t=1}^\timehorizon \, \E_{\skill} \big( \E( \regpertime(t) | \Ft) \big),
\end{align*}
where we used the tower property of the conditional expected value.
Note that $\regpertime \leq 1/\skillmin $ such that by applying  Lemma \ref{lemma_bound_regret_deviation}, Lemma \ref{lemma_at_bt_event} and then Lemma \ref{lemma_lipschitz_cont_reward}, one can derive that
\begin{align*}
\E_{\skill} [\regret(\timehorizon)] 
&\leq \sum_{t=1}^\timehorizon \, \Big[ \E_{\skill} \big( \E( \regpertime(t) 1_{ \At } | \Ft) \big) +    \E_{\skill}  \sum_{ i \in S_t } \big( \E( \regpertime(t) 1_{\At^\complement} \, | \Ft ) \big)  \Big] \\
&\leq \sum_{t=1}^\timehorizon \, \Big[ \E_{\skill} \big( \E( \regpertime(t) 1_{ \At } | \Ft) \big) + C_0 \sum_{t=1}^\timehorizon \, \frac{1}{\sqrt{t}}
%	
%	+	 \E_{\skill} \big( \E( \regpertime(t) 1_{ \At \cap \Bt^\complement } | \Ft) \big) 
 + C_1    \E_{\skill} \big( \sum_{ i \in S_t } \E(  \, |\hat O_{i,J}^{\mathrm{TRCB}}(t)  - O_{i,J} | \, 1_{\At^\complement} \, | \Ft ) \big)  \Big] \\
%
%&\qquad  \\
%	
&\leq C_2 \,  \sum_{t=1}^\timehorizon \, \sqrt{\frac{\log(t)}{t}} 
+  C_1 \sum_{t=1}^\timehorizon    \E_{\skill}  \big(   \sum_{ i \in S_t } \E( \, |\hat O_{i,J}^{\mathrm{TRCB}}(t)  - O_{i,J} | \, 1_{\At^\complement} \, | \Ft ) \big) \\
&\leq C_2 \sum_{t=1}^\timehorizon \sqrt{\frac{\log(t)}{t}} +  C_3 \sum_{t=1}^\timehorizon   \E_{\skill}   \big( \sum_{ i \in S_t } \, \sqrt{\nicefrac{ \log( \pullnumber \cdot t)}{ \wiJbar(t) }} \big),
\end{align*}
where $C_i>0,$ for $i\in \{0,1,2,3\},$ are constants depending if at all only on $\skillmin$ and $\gamma,$  but are independent of $\timehorizon,\pullnumber$ and $\items$.
Next, since $\sum_{t=1}^\timehorizon t^{-1/2} \leq 2 \sqrt{\timehorizon}$ and $\log(\pullnumber\cdot  t) \leq  2 \log( \timehorizon),$ due to $\pullnumber\leq \items < \timehorizon,$ we can further estimate the right-hand side of the latter display to obtain
\begin{align*}
\E_{\skill} [\regret(\timehorizon) ]
&\leq	C_4 \sqrt{\timehorizon\, \log(\timehorizon)} 
 +  C_5 \sqrt{\log(\timehorizon)} \sum_{t=1}^\timehorizon   \E_{\skill}  \sum_{ i \in S_t } \, \sqrt{\nicefrac{1}{ \wiJbar(t) }} \\
&\leq C_4 \sqrt{\timehorizon \, \log(\timehorizon)}  +  C_6 \sqrt{\log(\timehorizon)\, \timehorizon \,\items} ,
\end{align*}
where we used Lemma \ref{lemma_length_conf_region} for the second last inequality.
Here, the constants $C_4,C_5,C_6>0$ are as before depending (if at all) on $\skillmin$  and $\gamma,$  but are independent of $\timehorizon,\pullnumber$ and $\items$.
In particular, we have $C_4$ is of the form $\const \cdot \skillmin^{-1},$ while $C_6$ is of the form $$\const \cdot \frac{\max \{\skillmin^{(\gamma-1)/(\gamma)},\skillmin^{(1-\gamma)/(\gamma)} \} }{ \gamma } \cdot \skillmin^{-2(3+\gamma)}.$$
This concludes the proof.

\subsection{Proofs of the core lemmas in Subsection \ref{subsec:outline_proof_TRCB}}

We start with the proof of Lemma \ref{lemma_at_bt_event}. For this we need the following result, which is Lemma 1 in \citet{saha2019pac}.
\begin{lemma} \label{lemma_saha_gopalan}
	It holds that for any $r \in \N, i,j \in [\items]$ and $\varepsilon>0$ that
	\begin{align*}
	\P\Big(	\, \Big| \frac{\wij(t)}{\wijbar(t)}	- \frac{\skill_i}{\skill_i+\skill_j}	\Big| \geq \varepsilon, \ \wijbar(t) = r	\Big) 
	&\leq 
	\P\Big(	\, \Big| \frac{\wij(t)}{\wijbar(t)}	- \frac{\skill_i}{\skill_i+\skill_j}	\Big| \geq \varepsilon, \ \wijbar(t) \geq r	\Big) \\
	&\leq 2 \exp(-2\, r \, \varepsilon^2).
	\end{align*}
\end{lemma}

\begin{proof}[\sl Proof of Lemma \ref{lemma_at_bt_event}]
	Define the function $\phi(x)=x\inv -1,$ then note that $\phi\big(\frac{\wji(t)}{\wijbar(t)}\big) = \estoddij(t)$ and $\phi\big(\frac{\skill_j}{\skill_i+\skill_j}\big) = \oddij.$ 
	Further, by the mean value theorem there exists for any pair of arms $(i,j)$ some $\tilde z_{i,j}$ between $\frac{\wji(t)}{\wijbar(t)}$ and $\frac{\skill_j}{\skill_i+\skill_j}$ such that
	\begin{align*}
	\estoddij(t) - \oddij =	\phi\Big(\frac{\wji(t)}{\wijbar(t)}\Big) - \phi\Big(\frac{\skill_j}{\skill_i+\skill_j}\Big) 
	=	\phi'(\tilde z_{i,j}) 	\,  \Big(	\frac{\wji(t)}{\wijbar(t)}  - \frac{\skill_j}{\skill_i+\skill_j} \Big) 
	&= - \frac{1}{\tilde z_{i,j}^2} \Big(	\frac{\wji(t)}{\wijbar(t)}  - \frac{\skill_j}{\skill_i+\skill_j} \Big).	
	\end{align*}
	%
	%	Note that on the event $\{  | \nicefrac{\wji(t)}{\wijbar(t)}	- \nicefrac{\skill_j}{\skill_i+\skill_j}	| < \eta  \}$ for some $\eta\in(0,\skillmin/4)$ it holds that 
	%	
	Note that 
	\begin{align*}
	\tilde z_{i,j} \geq \min (\nicefrac{\wji(t)}{\wijbar(t)},\nicefrac{\skill_j}{\skill_i+\skill_j} ) 
	\geq \min (\nicefrac{\wji(t)}{\wijbar(t)},\nicefrac{\skillmin}{2} ) 
	\end{align*}
	and in particular if $j=J$ then $$\tilde z_{i,J} \geq \min (\nicefrac12,\nicefrac{\skillmin}{2} ) = \nicefrac{\skillmin}{2},$$ as $\wiJbar \leq 2 \wJi$ by definition of $J$ and $\skillmin <1.$
	Let us write  $ E_{i,J}(t) = \big|  \estoddiJ(t)		- \oddiJ			\big|$ for sake of brevity, then we get with the deviation above for  $\varepsilon>0$  for any $t \in [2,\timehorizon] \cap \N$ that
	\begin{align*}
	\P\Big(	\big\{  E_{i,J}(t)	\geq \nicefrac{\varepsilon}{\sqrt{\wiJbar(t)}} \big\} 	\Big)	
	&\leq 	\sum_{r=1}^{t-1}	\P\Big(	\, \Big\{ \Big| \frac{\wJi(t)}{\wiJbar(t)}	- \frac{\skill_J}{\skill_i+\skill_J}	\Big| \geq \frac{\skillmin^2 \, \varepsilon}{4 \sqrt{\wiJbar(t)} }    \Big\}  	\Big)  \cap  \{\wiJbar(t) = r\}	\Big) \\
	%		
	%	&\quad + \sum_{r=1}^{t-1} \P\Big(	\,  \{  | \nicefrac{\wJi(t)}{\wiJbar(t)}	- \nicefrac{\skill_J}{\skill_i+\skill_J}	| \geq \eta  \}	\cap  \{\wiJbar(t) = r\} \Big)	\\
	%		
	%	
	&= \sum_{r=1}^{t-1} \P\Big(	\, \Big\{ \Big| \frac{\wJi(t)}{\wiJbar(t)}	- \frac{\skill_J}{\skill_i+\skill_J}	\Big| \geq \frac{\skillmin^2 \, \varepsilon}{4 \sqrt{r} }    \Big\}  \cap  \{\wiJbar(t) = r\}	\Big) \\
	%	
	%	&\quad + \sum_{r=1}^{t-1} \P\Big(	\,  \big\{ | \nicefrac{\wJi(t)}{\wiJbar(t)}	- \nicefrac{\skill_J}{\skill_i+\skill_J}	| \geq \eta \big\} \cap  \{\wiJbar(t) = r\}	\Big)	\\ 
	%	
	&\leq 2 (t-1) \exp\Big(	- \frac{  \skillmin^4 \varepsilon^2 }{8}			\Big),
	\end{align*}
	where  Lemma \ref{lemma_saha_gopalan} was used in the last step.
	%
	%	By choosing $\eta= \eta(r) = \min(\nicefrac{\skillmin}{4},\, \sqrt{\nicefrac{\log( \pullnumber t^2)}{r}})$ the second term on the right hand side is $\Oterm(\nicefrac{1}{\pullnumber \sqrt{t}}).$
	%	
	Setting  $\varepsilon = \sqrt{\nicefrac{ 8 \,   \log(  \pullnumber \, t^{3/2})}{  \skillmin^4 }} $ in the last display, we obtain in combination with the law of total expectation that conditioned on $\Ft$ that
	\begin{align*}
	\P\big(	\At 	\big) 
	&\leq  \sum_{i \in S_t \cup \optsubset} \int_{[-\ciJ(t),\ciJ(t)]} \, (2\ciJ(t))^{-1} 
	\P \Big(  \big\{ E_{i,J}(t)	\geq \ciJ(t) - \Cshrink\, y \big\}   \Big)   \, dy \\
	&\leq  \sum_{i \in S_t \cup \optsubset}  \,   \P \Big( \big\{ E_{i,J}(t)	\geq (1-\Cshrink) \, \ciJ(t)  \big\} \Big) \\
	&\leq  4 \, \pullnumber \, (t-1) \exp\big(	- \nicefrac{   \skillmin^4 \varepsilon^2 }{8}			\big) 
	= \Oterm(\nicefrac{ \sqrt{\log(t)} }{\sqrt{t}}),
	\end{align*}
	where we used that thresholding of the relative scores only makes the probability of the event smaller for the first inequality and in the second last step that, firstly, $\Cshrink\leq 1/2$ in combination with $1/2 \, \ciJ(t)=\nicefrac{\varepsilon}{\sqrt{\wiJbar(t)}}$ and secondly, that $|S_t \cup \optsubset|\leq 2 \pullnumber.$
	The constant in the $\Oterm$-term is independent of $\pullnumber,$ $\timehorizon$ and $\items.$
	This concludes the lemma.
	%
	%	For the second claim, note that if $t>2$ then
	%	%	
	%	$$	\E_{\skill} \wJi(t) = \sum_{s=1}^{t-1} \E_{\skill} \, 1_{	\{  i_s= J , \, \{i,J\} \in S_s	\}	} \geq (t-1) \, \frac{\skillmin}{2}.		$$
	%	%
	%	Thus, with the multiplicative Chernoff Bound we have for any $\delta\in(0,1)$
	%	%	
	%	\begin{align*}
	%	%		
	%	\P(	 \Bt^\complement	)
	%	%		
	%	&= \P (	\exists i \in S_t \cup \optsubset: \,	\wJi(t) =	0	)
	%	%		
	%	\leq 2 \, \pullnumber \, \P \big(	\wJi(t) \leq (1-\delta)	\E_{\skill} \wJi(t)		\big) \\
	%	%		
	%	&\leq 2 \, \pullnumber \, \exp\big(	- \nicefrac{\delta^2 \skillmin \, (t-1)}{4}			\big).
	%	%		
	%	\end{align*} 
	%	%	
	%	The right hand side of the latter inequality is $\Oterm(\nicefrac{1}{\sqrt{t}}),$ if $\delta \geq \sqrt{\frac{4 \log(2 \, \pullnumber \, \sqrt{t})}{ \skillmin \, (t-1)}},$ which is on the other hand an element of $(0,1)$ for $t$ sufficiently large.
	%
	%
\end{proof}

\newcommand{\Sstarsquare}{S^*_2(O)}
\newcommand{\Sstarplain}{S^*_1(O)}
\newcommand{\Sstarhatsquare}{S^*_2(\hat O)}
\newcommand{\Sstarhatplain}{S^*_1(\hat O)}
\newcommand{\Stsquare}{S^t_2(O)}
\newcommand{\Stplain}{S^t_1(O)}
\newcommand{\Shattsquare}{S^t_2(\hat O)}
\newcommand{\Shattplain}{S^t_1(\hat O)}

\begin{proof}[\sl Proof of Lemma \ref{lemma_bound_regret_deviation}]
	Let us write $\Sstarsquare = \sum_{i\in \optsubset} O_{i,J}^{\frac{1+\gamma}{\gamma}}$ and $\Sstarplain	= \sum_{i\in \optsubset} O_{i,J}.	$
	In the same spirit define $\Sstarhatsquare,$ $\Sstarhatplain$, $\Stsquare,$ $\Stplain,$ $\Shattsquare$ and $\Shattplain,$ where $\hat O$ is short for $\hat O_J^{\mathrm{TRCB}}.$
	Then,
	\begin{align} \label{eq_help_regret_ratio}
	\begin{split}
	&\frac{  \rewardtilde(\optsubset;O_J,\gamma) - \rewardtilde(\optsubset;\hat O_J^{\mathrm{TRCB}},\gamma)}{\rewardtilde(S_t;\hat O_J^{\mathrm{TRCB}},\gamma)	 -	\rewardtilde(S_t;O_J,\gamma)}
	= 	\frac{  \frac{\Sstarsquare}{\Sstarplain}   - \frac{\Sstarhatsquare}{\Sstarhatplain}  }{ 	   \frac{\Shattsquare}{\Shattplain}   - \frac{\Stsquare}{\Stplain}	 } 
	%	\\
	%			
	%			&= \frac{\Shattplain \Stplain}{\Sstarplain \Sstarhatplain}	\frac{  \Sstarsquare \Sstarhatplain  - \Sstarhatsquare \Sstarplain  }{ 	   \Shattsquare \Stplain   - \Stsquare \Shattplain	 } \\
	%			
	%	&
 =	\frac{  \frac{ [\Sstarsquare - \Sstarhatsquare		]}{\Sstarplain}   + \frac{\Sstarhatsquare[\Sstarhatplain - \Sstarplain]}{\Sstarplain\Sstarhatplain}  }			{ 	   \frac{[\Shattsquare-\Stsquare]}{\Shattplain}   + \frac{\Stsquare[\Stplain - \Shattplain]}{\Shattplain\Stplain}	 }.
	\end{split}
	\end{align}
	It holds that 
	\begin{align*}
	&\nicefrac{\skillmin}{ \pullnumber} \leq \frac{1}{\Sstarplain}\leq \nicefrac{1}{\skillmin \pullnumber},
	\\ 
	&\nicefrac{\skillmin^{3+\gamma}}{  \pullnumber} \leq \frac{\Sstarhatsquare}{\Sstarplain\Sstarhatplain} \leq 	\nicefrac{1}{ \skillmin^{3+\gamma} \pullnumber}, \\
	& \nicefrac{\skillmin}{\pullnumber} \leq 
	\frac{1}{\Shattplain} \leq \nicefrac{1}{\skillmin \pullnumber}, 
	\\
	&\nicefrac{\skillmin^{3+\gamma}}{\pullnumber } \leq
	\frac{\Stsquare}{\Shattplain\Stplain} \leq  	\nicefrac{1}{\skillmin^{3+\gamma} \pullnumber }.
	\end{align*}
	Hence, all of the latter terms can be bounded from below resp.\ above by $\tilde C_j/l$ for some suitable constants $C_j$ which depend if at all on $\skillmin.$
	%
	%	Thus, we can bound the right hand side of the first display of the proof to obtain
	%%	
	%	\begin{align*}
	%%		
	%\frac{ \rewardtilde(\optsubset;O_J) - \rewardtilde(\optsubset;\hat O_J^{\mathrm{TRCB}}) }{ \rewardtilde(S_t;\hat O_J^{\mathrm{TRCB}})	 -	\rewardtilde(S_t;O_J) }
	%%
	%&\leq
	%%
	%\frac{  \frac{ [\Sstarsquare - \Sstarhatsquare		]}{\skillmin}   + \frac{[\Sstarhatplain - \Sstarplain]}{\skillmin^4}  }			{ 	   [\Shattsquare-\Stsquare] \skillmin  + [\Stplain - \Shattplain] \skillmin^4	 }.	
	%%		
	%	\end{align*}
	%	
	Following the lines of proof of Lemma \ref{lemma_at_bt_event}, it can be shown that there exists a constant $C_1>0$  (depending on $\skillmin$ and $\gamma$) such that the ratios of the terms in the squared brackets  in (\ref{eq_help_regret_ratio}) are  bounded by some constant $C_2>0$ on the event $\At^\complement,$ with probability at least $1-\frac{C_1}{\sqrt{t}}.$
	%	 and in particular the enumerator and the denominator are of the same order in terms of $t$ and $\pullnumber.$
	%
	%
	Hence, the whole term in (\ref{eq_help_regret_ratio}) can be bounded with probability at least $1-\frac{C_1}{\sqrt{t}}$ by some constant $C_3>0$ which if at all depends only on $\skillmin.$
	This yields the first part of the lemma.
	The second part is just a consequence of the first part together with (\ref{motivation_regret_deviation}).
\end{proof}

\begin{proof} [\sl Proof of Lemma \ref{lemma_lipschitz_cont_reward}]
	Define the function $\phi(x_1,\ldots,x_\pullnumber) = \nicefrac{\sum_{i=1}^\pullnumber x_i^{\nicefrac{(1+\gamma)}{\gamma}} }{\sum_{i=1}^\pullnumber x_i}$ for $x_1,\ldots,x_\pullnumber\in[A,B]$ for $0< A <B.$
	Then, we have that for $i=1,\ldots,\pullnumber$	
	$$		\frac{\partial \phi(x_1,\ldots,x_\pullnumber)}{\partial x_i} 
	= \frac{ \frac{1+\gamma}{\gamma} x_i^{\nicefrac{1}{\gamma}} \sum_{j} x_j -  \sum_{j} x_j^{\nicefrac{(1+\gamma)}{\gamma}} }{(\sum_{j} x_j)^2}, 		$$
	It can be easily checked that 
	$$ \sup_i	\sup_{	x_i \in[A,B]	} 	\Big| \frac{\partial \phi(x_1,\ldots,x_\pullnumber)}{\partial x_i} \Big|	\leq  \begin{cases}
	\frac{B^{\frac{1-\gamma}{\gamma}}}{\gamma}, & \gamma \leq 1,\\
	\frac{A^{\frac{1-\gamma}{\gamma}}}{\gamma}, & \gamma>1.
	\end{cases},	$$
	Without loss of generality assume that $S_t=\{1,\ldots,\pullnumber\},$ then by setting  $x_i = O_{i,J}$ and $y_i = O_{i,J}^{\mathrm{TRCB}}(t)$ and noting that $\phi(x_1,\ldots,x_\pullnumber) = \rewardtilde(S_t;O_J)$ as well as $\phi(y_1,\ldots,y_\pullnumber) = \rewardtilde(S_t;\hat O_J^{\mathrm{TRCB}}),$ we obtain with the mean value theorem that
	\begin{align*}
	\big|\rewardtilde(S_t;\hat O_J^{\mathrm{TRCB}},\gamma)	& -	\rewardtilde(S_t;O_J,\gamma)\big|
	\leq  C	\sum_{i \in S_t} |\hat O_{i,J}^{\mathrm{TRCB}}(t)  - O_{i,J} |,		
	\end{align*}
	by choosing 
	$C=  \max\{  \nicefrac{\skillmin^{\frac{\gamma-1}{\gamma}}}{\gamma} , \nicefrac{\skillmin^{\frac{1-\gamma}{\gamma}}}{\gamma}, \},$
	since $\skillmin \leq  O_{i,J} \leq 1/\skillmin$ and $\skillmin \leq  O_{i,J}^{\mathrm{TRCB}}(t) \leq 1/\skillmin.$
\end{proof}

\begin{proof} [\sl Proof of Lemma \ref{lemma_length_conf_region}]
	Since $\sum_{t=1}^\timehorizon t^{-1/2} \leq 2 \sqrt{\timehorizon}$ one has $\sum_{\wiJbar(t)=1}^{\wiJbar(\timehorizon)} \nicefrac{1}{ \sqrt{\wiJbar(t)} } \leq 2 \sqrt{\wiJbar(\timehorizon)}.  $
	Due to $\sum_{ i \in [\items]}  \E \wiJbar(\timehorizon) \leq \timehorizon$ it follows 
	%	that $\E \wiJbar(\timehorizon) \leq \nicefrac{\timehorizon}{\items}$ for each $i \in [\items].$
	%	
	%	Thus, 
	by Jensen's inequality  that
	$\sum_{t=1}^\timehorizon    \E \, \sum_{ i \in S_t } \sqrt{\frac{1}{ \wiJbar(t) }} \leq 4 \sqrt{\timehorizon \, \items}.$
\end{proof}

\section{Proof of Theorem \ref{theorem_upper_bound_regret_flexible}} \label{sec:proof_CBR}

We start by introducing the notation  for the rest of the proof and recalling the main terms of the CBR algorithm. 
Thereafter we give an outline of the proof, before deriving the technical details.

\newcommand{\Ct}{B_t}
\newcommand{\Rt}{R_t}
\newcommand{\Et}{E_t}
\newcommand{\Lt}{L_t}
\newcommand{\imax}{i_{max}}
\newcommand{\Zmax}{Z_{\imax}}
\newcommand{\Ztildemax}{\tilde Z_{\imax}}
\newcommand{\Fs}{\mathcal F_s}
\newcommand{\Gap}{H}
\newcommand{\Gaptwo}{H'}

We break the proof down into two core lemmas, for which we first clarify the notation. We assume that without loss of generality $|\optsubset|=1,$ i.e., there is only one best arm, as this makes the learning problem only more difficult.
Indeed, having several arms with the same highest score extends the opportunities to identify one of these highest score arms.
%, while no penalization  in terms of regret occurs  by pulling the highest score arms. 
%
To ease the notation we denote the score of the highest scored arm  with $\skillmax,$ which is $1$ by definition of $\paramspace$ and its index by $\imax.$

\subsection{Notation and relevant terms}
We define the estimate for the pairwise winning probability $q_{i,j}$ (cf.\ (\ref{def_lwise_marginals_pl_model})) by
$$		\hat q_{i,j} = \hat q_{i,j}(t) = \begin{cases}
\frac{w_{i,j}(t)}{w_{i,j}(t) + w_{j,i}(t)}, & i,j\in [\items], i\neq j, \\
0, &  i= j,
\end{cases}$$
where $w_{i,j}$ are as in (\ref{defi_number_of_picks}) and with the convention that $\nicefrac{x}{0}=0.$
With $J(t)=J$ we again denote the arm (within the active set) with the most picks till time instance $t$ as in (\ref{defi_most_picked_arm}).
With $\Delta_i = \skillmax - \skill_i$ we define the gap between the score of the $i$th arm and the overall best arm.
The lengths of the confidence intervals are
$$		c_{i,j}^{\mathrm{CBR}}(t) = c_{i,j}  =  \begin{cases}
\sqrt{\frac{ 2 \log( \items \, t^{3/2})}{\wijbar(t)}}, & i,j\in[\items], i\neq j, \\
0, & i=j, \\
\end{cases}	$$
thereby implicitly setting  $\overline{w}_{ii}(t)=\infty$ for any $i \in [\items].$

\subsection{Outline of the proof} \label{subsec:outline_proof_CBR}

We define the following events
\begin{align*}
\Ct = \{\exists i \in [\items] \ | \ | \hat q_{i,J}(t) 	- q_{i,J} | > c_{i,J}(t)		\},	
\qquad 	\Rt &= \{	J(t) \neq \imax		\}, \qquad \mbox{and} \qquad 	\Et = \{  |S_t| > 1 \}. 
\end{align*}
Here, $\Ct$ is the event where an arm exists whose pairwise probability estimate for winning against $J$ is not close enough to its actual parameter, where closeness is understood by means of the confidence length $c_{i,J}(t).$
$R_t$ is the event when the  most winning arm $J$ is not the overall best arm and $\Et$ is the event, where the offered subset at time instance $t$ is not a singleton.
All these events are ''bad'' events and we will show that their probability of occurrence is sufficiently small.

We have the following key lemmas to prove the main result.

\begin{lemma} \label{lemma_first_aux_lemma_CBR}
	There exist constants $C_1,C_2,C_3>0$ independent of\, $T$ and $\items$  (depending if at all on the parameter space $\paramspace$), such that
	\begin{align*}
	&\sum_{t=1}^\timehorizon \P( \Ct  ) \leq C_1 \quad \mbox{and} \quad  	
	\sum_{t=1}^\timehorizon \P(\Rt \cap \Ct^\complement ) \leq C_2 \,  \log(\timehorizon) \sum_{i \in [\items]\backslash\{\imax\}} \frac{1}{\Delta_i^2}  + C_3 \items.
	\end{align*}	
\end{lemma}

\begin{lemma} \label{lemma_second_aux_lemma_CBR}
	There exist constants $C_1,C_2>0$ independent of\, $T$ and $\items$ (depending if at all on the parameter space $\paramspace$), such that
	$$ 	\sum_{t=1}^\timehorizon \P( \Ct^\complement \cap \Rt^\complement \cap \Et  ) \leq C_1 \,  \log(\timehorizon) \sum_{i \in [\items]\backslash\{\imax\}} \frac{1}{\Delta_i^2}  + C_2 \items.	$$
\end{lemma}

\begin{lemma} \label{lemma_third_aux_lemma_CBR}
	%	
%	On the event $\Rt^\complement$ 
	It holds that
	$$ \rpreselect(S_t) \leq  \frac{\max \{\skillmin^{(\gamma-1)/(\gamma)},\skillmin^{(1-\gamma)/(\gamma)} \} }{ \gamma \, \skillmin} \sum_{i\in [\items] \backslash \optsubset}\Delta_i.	$$
\end{lemma}

\paragraph{Putting all together.}

%Note that $\rpreselect(S_t) \leq 1$ and therefore
Recalling the cumulative regret in (\ref{defi_regret}), we obtain
\begin{align*}
\E \,\big( \regret(\timehorizon)\big)
&= \sum_{t=1}^\timehorizon \E \, \rpreselect(S_t) \\
&\leq \sum_{t=1}^\timehorizon \P( \Ct ) + \sum_{t=1}^\timehorizon \E \, \rpreselect(S_t) 1_{ R_t\cap \Ct^\complement}   
 +  \sum_{t=1}^\timehorizon \E \, \rpreselect(S_t)1_{ \Ct^\complement \cap R_t^\complement  }  \\
&\leq  \sum_{t=1}^\timehorizon \P( \Ct ) + \sum_{t=1}^\timehorizon  \E \, \rpreselect(S_t) 1_{ R_t\cap \Ct^\complement}   
+ \sum_{t=1}^\timehorizon \E \, \rpreselect(S_t) 1_{ \Ct^\complement \cap R_t^\complement \cap E_t  }    
+ \sum_{t=1}^\timehorizon \E \, \rpreselect(S_t) 1_{ \Ct^\complement \cap R_t^\complement \cap E_t^\complement  }     \\
&\leq C_0 \items + C_1 \, \frac{\max \{\skillmin^{(\gamma-1)/(\gamma)},\skillmin^{(1-\gamma)/(\gamma)} \} }{ \gamma \, \skillmin } \, \log(\timehorizon)  \sum_{i \in [\items]\backslash\{\imax\}} \frac{\sum_{i\in [\items] \backslash \optsubset}\Delta_i}{\Delta_i^2},
\end{align*}
where we used Lemma \ref{lemma_first_aux_lemma_CBR} and Lemma \ref{lemma_second_aux_lemma_CBR} to derive the constants $C_0,C_1>0,$ which are both independent of $\timehorizon$ and $\items,$ while Lemma \ref{lemma_third_aux_lemma_CBR} introduced the factor accompanying $C_1.$ 
Furthermore, we used that on $  R_t^\complement \cap E_t^\complement$ we have that $ S_t$ equals $ \{\imax\} = \optsubset$  and thus $\rpreselect(S_t)=0.$
%	
%	$$	\rpreselect(S_t) = \skillmax - \reward(S_t) = \frac{\sum_{i\in S_t}  \skill_i(\skillmax-\skill_i) }{\sum_{i \in S_t} \skill_i}	\leq c_{\imax,J}(t)  = \sqrt{\frac{\log(\sqrt{\items}t^{3/2})}{\overline{w}_{\imax,J}}}.	$$
%	
%	This together with Lemma \ref{lemma_third_aux_lemma_CBR} implies
%%
%	\begin{align*}
%%		 
%		  \sum_{t=1}^\timehorizon \E \rpreselect(S_t) 1_{ C_t^\complement \cap R_t^\complement \cap E_t^\complement  } 
%%		 
%			&\leq  C_0 \items^2 + C_1 \sqrt{\log(\timehorizon)}  \sum_{t=1}^\timehorizon \E ( \sqrt{\frac{1}{\overline{w}_{\imax,J}}} \cdot  1_{ C_t^\complement \cap R_t^\complement \cap E_t^\complement  }  )
%%
%			= \Oterm\Big(   \sqrt{\log(\timehorizon) \, \timehorizon}		\Big),
%%
%	\end{align*}
%%
%	where $C_0,C_1>0$ are constants depending if at all on $\skillmin.$

\subsection{Proofs of the core lemmas in Subsection \ref{subsec:outline_proof_CBR}}

\begin{proof} [\sl Proof of Lemma \ref{lemma_first_aux_lemma_CBR}]
	Using Lemma \ref{lemma_saha_gopalan} one obtains
	\begin{align*}
	\P( \Ct  ) 
	\leq \sum_{i \in [\items]}  \sum_{r=1}^{t-1}  \P\big(  | \hat q_{i,J}(t) 	- q_{i,J} | > c_{i,J}(t)	, \ \wiJbar(t)=r 		\big)
	&\leq 2 \items \sum_{r=1}^{t-1} \exp(-4 \log( \items t^{3/2})) \leq \nicefrac{2}{t^5}.
	\end{align*}
	By summing over $t$ till $\timehorizon,$ we get
	$	\sum_{t=1}^\timehorizon	\nicefrac{2}{t^5} < 	2 \sum_{t=1}^\infty	\nicefrac{1}{t^2} = \pi^2/3, $
	which yields the first claim.
	
	For the second claim, let $A_t$ denote the set of active arms at time instance $t,$ i.e.,
	\begin{align*}
	A_t = \Big\{	i\in[\items] \, \big| \,  \sigma\Big(\frac{\hat q_{i,J(s)}(s) + c_{i,J(s)}(s) -1/2}{2 c_{i,J(s)}(s)}  \Big)  > 0, \	\forall s \in [t]	\Big\}	
	\end{align*}
	It holds that conditioned on $\Ct^\complement$ we have that $ \imax \in A_t$ almost surely.
	Indeed, 
	\begin{align*}
	\P( \{\imax \notin A_t\} \cap \, \Ct^\complement) 
	= \P\Big(  \sigma\Big(\frac{\hat q_{\imax,J}(t) + c_{\imax,J}(t) -1/2}{2 c_{\imax,J}(t)}  \Big)  \leq 0, \, \Ct^\complement\Big) 
	&= \P\Big(   \hat q_{\imax,J}(t) + c_{\imax,J}(t) \leq 1/2 , \, \Ct^\complement\Big) \\
	&\leq \P\Big(    q_{\imax,J}(t)  \leq 1/2 ) = 0,
	\end{align*}
	where we used that $\sigma(x)\leq 0$ iff $x\leq 0$ and for the last inequality that $\hat q_{\imax,J}(t) + c_{\imax,J}(t) \geq  q_{\imax,J}(t)$  on $\Ct^\complement,$ while $q_{\imax,J}(t)  > 1/2 $ holds by definition of $\imax.$  \\
	Next, consider the counting process $M_t^{i,\imax}:= w_{i,\imax} - w_{\imax,i}$ for some $i\in A_t\backslash\{\imax\}$ and define for sake of brevity the event $ \tilde S_s^{i} = \{  \{i,\imax\} \in S_s	\}.$
	Note that $M_t^{i,\imax}$ can be written as $$  M_t^{i,\imax} = \sum_{s=1}^{t-1}  1_{	\{  \{i_s= i\} \cap \tilde S_s^{i} \}	}   - 1_{	\{  \{i_s= \imax ,\} \cap \tilde S_s^{i}	\}	}.$$
	It holds that the event $\{ \{i,\imax\} \in S_s\}$ has a strictly positive probability for any arm $i \in  A_t\backslash\{\imax\}$ and any $s \in [t],$ as otherwise the arm would not be active anymore.
	%		
	%		Thus, we can find a 
	Conditioned on some set $S_s$ we have that 
	\begin{align*}
	\P\big(\{  i_s= i 	\} \big) & - \P\big(\{  i_s= \imax  \}\big) 
	 = \frac{\skill_i}{\sum_{j\in S_s} \skill_j}  - \frac{\skillmax}{ \sum_{j\in S_s} \skill_j }  \leq   - \frac{\Delta_i}{\Gaptwo },
	\end{align*}
	where $\Gaptwo = \sum_{i\in [\items]} \skill_i.$  
	Thus, we can find a constant $C>0,$ which depends only on $\paramspace$ such that for each  $s \in [t]$
	\begin{align*}
	\P\big(\{  i_s= i 	\}  \cap \tilde S_s^{i} \big) & - \P\big(\{  i_s= \imax  \}  \cap \tilde S_s^{i} \big)\leq - \Delta_i \, C.
	\end{align*}
	Therefore, $	\E 	 M_t^{i,\imax} \leq   - (t-1) \, C \, \Delta_i 	$ 
	%		for some appropriate $\delta_i>0,$ as the probability of the event $\{  i_s= i , \, \{i,\imax\} \in S_s	\}$ is strictly smaller than of $\{  i_s= \imax , \, \{i,\imax\} \in S_s	\},$ as $\skill_{\imax}>\skill_i$ for any $i\in[\items]\backslash\{\imax\}.$
	%		
	and by Lemma \ref{lemma_bouning_inequality_ternary} it follows that
	\begin{align*}
	\P( w_{i,\imax}  \geq  w_{ \imax,i}  ) 
	=	\P(M_t^{i,\imax} \geq 0) \leq \P(M_t^{i,\imax} \geq    - 2 (t-1) \, C \, \Delta_i) 
	&\leq \exp\big(	-\frac{C^2 \, \Delta_i^2  (t-1) }{8  }	\big). 
	\end{align*}
	The event $\Rt$ is contained in the event that there exists an active arm $i$ such that the winning count of $\imax$ against $i$ is smaller than the winning count of $i$ against $\imax,$ that is $M_t^{i,\imax}\geq 0.$ 
	Hence, using the union bound in combination with the latter display we obtain 
	\begin{align*}
	\sum_{t=1}^\timehorizon \P(\Rt \cap \Ct^\complement)  
	&\leq 	\sum_{t=1}^\timehorizon \sum_{i\in [\items] \backslash\{\imax\}}   \exp\big(	-\frac{C^2 \, \Delta_i^2  (t-1) }{8  }	\big)\\
	& =    \sum_{i\in [\items] \backslash\{\imax\}}  	\sum_{t=1}^{ \lceil \nicefrac{8 \, \, \log(\timehorizon)}{ C^2 \Delta_i^2} \rceil }    \exp\big(	-\frac{C^2 \, \Delta_i^2  (t-1) }{8  }	\big) 
 + \sum_{t\geq \lceil \nicefrac{8 \,  \, \log(\timehorizon)}{ C^2 \Delta_i^2} \rceil}^{\timehorizon}  \  \exp\big(	-\frac{ C^2 \, \Delta_i^2  (t-1) }{8   }	\big) \\
	&\leq \frac{8 \, \log(\timehorizon)}{C^2} \sum_{i\in [\items] \backslash\{\imax\}} \frac{ 1  }{ \Delta_i^2}	+ 2\, \items \, \timehorizon \exp(-\log(\timehorizon)),
	\end{align*}		
	from which we can conclude the lemma.
	%	
	%		\begin{align*}
	%		%		
	%		\P(\Rt) 
	%		%		
	%		&= \P \big(	\frac{\skillmax}{\skillmax + \skill_J}	- \frac{\skill_J}{\skillmax + \skill_J}	> \frac{c_{\imax,J}(t)}{\skillmax + \skill_J}	\big)
	%		%		
	%		\leq \P \big( q_{\imax,J} - q_{J,\imax}	> \frac{c_{\imax,J}(t)}{2}	\big) \\
	%		%		
	%		&\leq \P \big( q_{\imax,J} - \hat q_{\imax,J}	> \frac{c_{\imax,J}(t)}{4}	\big) 
	%		+ \underbrace{\P \big( \hat q_{\imax,J} - \hat q_{J,\imax}	> 0	\big)}_{=0}  \\
	%		%		
	%		&\quad + \P \big(  \hat q_{J,\imax} - q_{J,\imax}	> \frac{c_{\imax,J}(t)}{4}	\big) \\
	%		%		
	%		&= \sum_{r=1}^t \P \big( q_{\imax,J} - \hat q_{\imax,J}	> \frac{c_{\imax,J}(t)}{4} , \ \overline{w}_{\imax,J}(t)=r  \big) \\
	%		&\quad + \quad  \P \big(  \hat q_{J,\imax} - q_{J,\imax}	> \frac{c_{\imax,J}(t)}{4}, \ 	\overline{w}_{\imax,J}(t)=r  \big) \\
	%		%		
	%		&=  \Oterm(\nicefrac{1}{\sqrt{t}}),
	%		%		
	%		\end{align*}
	%		%	
	%		where we used that $\hat q_{\imax,J} \leq \hat q_{J,\imax} $ almost surely in the first inequality and Lemma \ref{lemma_saha_gopalan}  for the last step.
	%	
\end{proof}

\begin{proof} [\sl Proof of Lemma \ref{lemma_second_aux_lemma_CBR}]
	For any $i \neq \imax$ we have that
	$$	\E (\overline{w}_{i,\imax}(t)) = \sum_{s=1}^{t-1} \P( i_s \in \{i,\imax\}  , \{i,\imax\} \in S_s  ).	$$ 
	Now, similar as in the proof of Lemma \ref{lemma_first_aux_lemma_CBR} before, we can find a constant $\tilde C>0$ which depends if at all on $\paramspace$ such that 
	$ \P( i_s \in \{i,\imax\}  , \{i,\imax\} \in S_s  ) \geq \skillmin \tilde C $ for any active arm $i$ and each $s\in[t].$
	%	\\
	With this, we obtain that $\E (\overline{w}_{i,\imax}(t)) \geq (t-1) \skillmin \tilde C.$
	%	\\
	Using Lemma \ref{lemma_counting_processes} with $\overline{w}_{i,\imax}$ as the counting process one can derive that there exists a constant $C>0$ depending on $\paramspace$ such that
	\begin{align} \label{ineq_help_second_aux_CBR}
	\P\Big(\overline{w}_{i,\imax}(t) \leq \frac{(t-1) C}{2} \Big) \leq \exp\big(	- \frac{(t-1)  C^2}{8}		\big).
	\end{align}
	Next, write for short $\delta_{i,\imax} = \hat q_{i,\imax}(t) + c_{i,\imax}(t) -1/2$ and note that 
	\begin{align*}
	\P( \Ct^\complement \cap \Rt^\complement \cap \Et  ) 
	&= \P( \exists i\neq \imax: \,  \{ i \in S_t  \}   , \, \Ct^\complement \cap \Rt^\complement   ) \\
	&\leq \sum_{i\in [\items] \backslash\{\imax\}} \P\Big(  \sigma\Big(\frac{ \delta_{i,\imax} }{2 c_{i,\imax}(t)}  \Big)  \geq 0,  \, \Ct^\complement \cap \Rt^\complement   \Big) \\
	&\leq \sum_{i\in [\items] \backslash\{\imax\}} \P\Big( \delta_{i,\imax}  \geq 0,  \, \Ct^\complement \cap \Rt^\complement   \Big) \\
	&\leq \sum_{i\in [\items] \backslash\{\imax\}} \P\Big(    2 c_{i,\imax}(t) \geq 1/2  - q_{i,\imax}   \Big) \\
	&= \sum_{i\in [\items] \backslash\{\imax\}} \P\Big(    \overline{w}_{i,\imax}(t) \leq \frac{8 \log(nt^{3/2})}{(1/2  - q_{i,\imax} )^2}  \Big) \\
	&\leq \sum_{i\in [\items] \backslash\{\imax\}} \P\Big(    \overline{w}_{i,\imax}(t) \leq \frac{20 \log(\timehorizon)}{(1/2  - q_{i,\imax} )^2}  \Big),
	\end{align*}
	where we used that $J(t) = \imax$ on $\Rt^\complement$ for the first inequality, $\sigma(x)\leq 0$ iff $x\leq 0$ for the second inequality, for the third inequality that $\hat q_{i,\imax}(t) - c_{i,\imax}(t) \leq  q_{i,\imax}(t)$  on $\Ct^\complement,$ while the last inequality is due to $\log(\items t^{3/2}) \leq 5/2 \log(\timehorizon),$ as $\max\{n,t\}\leq \timehorizon.$ 
	One can find constants $C_i \in [1/4,1/2]$ such that $1/2  - q_{i,\imax}  = C_i \Delta_i.$
	Indeed, note that $1/2  - q_{i,\imax} = \nicefrac{\Delta_i}{(2(\skill_i+\skillmax))}$ and it holds that
	$$	\frac{\Delta_i}{4} \leq 	\frac{\Delta_i}{2(\skill_i+\skillmax)}	\leq \frac{\Delta_i}{2}.	$$
	Hence, with these considerations one obtains
	\begin{align*}
	\sum_{t=1}^\timehorizon \P( \Ct^\complement \cap \Rt^\complement \cap \Et  ) 
	&\leq 	\sum_{t=1}^\timehorizon \sum_{i\in [\items] \backslash\{\imax\}} \P\Big(    \overline{w}_{i,\imax}(t) \leq \frac{20 \log(\timehorizon)}{C_i^2 \Delta_i^2}  \Big) \\
	%		\\
	&\leq  \sum_{i\in [\items] \backslash\{\imax\}}  \frac{40 \log(\timehorizon)}{ C C_i^2 \Delta_i^2}  
	+
	\sum_{i\in [\items] \backslash\{\imax\}} \sum_{t=\lceil \frac{40 \log(\timehorizon)}{ C C_i^2 \Delta_i^2} \rceil }^{\timehorizon} \P\Big(    \overline{w}_{i,\imax}(t) \leq \frac{20 \log(\timehorizon)}{C_i^2 \Delta_i^2}  \Big).
	%		\\
	%		&\leq  \frac{40 \log(\timehorizon)}{ C }  \sum_{i\in [\items] \backslash\{\imax\}} \frac{1}{C_i^2 \Delta_i^2}   +
	%	%		
	%		\sum_{i\in [\items] \backslash\{\imax\}} \sum_{t=\lceil \frac{40 \log(\timehorizon)}{\tilde C C_i^2 \Delta_i^2} \rceil }^{\timehorizon} \P\Big(    \overline{w}_{i,\imax}(t) \leq \frac{40 \log(\timehorizon)}{C_i^2 \Delta_i^2}  \Big).
	%		
	\end{align*}
	Now, the summation over $t$ on the right-hand side of the last display is such that 
	$ \nicefrac{20 \log(\timehorizon)}{C_i^2 \Delta_i^2}  \leq \nicefrac{(t-1) C}{2}.$
	Thus, we can use (\ref{ineq_help_second_aux_CBR}) to further estimate the last display by
	\begin{align*}
	\sum_{t=1}^\timehorizon \P( \Ct^\complement \cap \Rt^\complement \cap \Et  ) \leq  \frac{40 \log(\timehorizon)}{ C }  \sum_{i\in [\items] \backslash\{\imax\}} \frac{1}{C_i^2 \Delta_i^2}   +
	C_1 \items \timehorizon^{-C_2},
	\end{align*}
	for some constants $C_1,C_2>0.$
	From the latter display we can conclude the lemma.
\end{proof}

\begin{proof}[\sl Proof of Lemma \ref{lemma_third_aux_lemma_CBR}]
	%	\\
	Note that 
	\begin{align*}
	\rpreselect(S_t) 
	= 	\reward(\optsubset) - \reward(S_t) 
	= \frac{\sum_{i\in S_t} (v_{\max} - v_i) v_i^\gamma }{\sum_{i\in S_t} v_i^\gamma} 	 
	%		\\
	= \frac{\sum_{i\in S_t} (\skillmax^{1/\gamma} - \skill_i^{1/\gamma}) \skill_i }{\sum_{i\in S_t} \skill_i }. 	
	\end{align*}
	Since $\theta \in \paramspace$ it holds that  
	$	\sum_{i\in S_t} \skill_i \geq \skillmin.	$
	With this, and the fact that $\skill_i \leq \skillmax=1,$  we can infer that
	%\\
	\begin{align*}
	\rpreselect(S_t)  \leq \skillmin^{-1} \sum_{i\in S_t} (\skillmax^{1/\gamma} - \skill_i^{1/\gamma}) 
	\leq   \skillmin^{-1}  \sum_{i\in [\items]\backslash \optsubset} (\skillmax^{1/\gamma} - \skill_i^{1/\gamma}).
	\end{align*}
	Considering the function $f(x)=x^{1/\gamma}$ defined for $x\in[\skillmin,\skillmax]$ the assertion follows easily by the mean-value theorem as in the proof of Lemma \ref{lemma_lipschitz_cont_reward}.
\end{proof}

\subsection{Technical results}

In this subsection we collect the technical auxiliary results needed for the proofs of the core lemmas.
These technical results could  also be of independent interest.

The next two lemmas were of major importance for the proof of Lemma \ref{lemma_first_aux_lemma_CBR}.
%and are an extension of Lemma 12 and 13 of \citet{kocsis2006improved} to the ternary case.

\begin{lemma} \label{lemma_concentration_ternary}
	Let $M_t=\sum_{s=1}^t Z_s,$ where $(Z_s)_{s=1,\ldots,t}$ are random variables with values in $\{-1,0,1\},$ such that $\F_s$ is the canonical filtration generated by $\{Z_1,\ldots,Z_{s-1}\}$ and $Z_{s+1}$ is conditionally independent of $Z_{s+2},\ldots,Z_{t}$ given $\F_s.$
	We have that for any $z>0$
	$$	\P(M_t - \E(M_t) > z  ) \leq \exp\big(	-\frac{z^2}{8\,t}	\big).			$$
\end{lemma}
\begin{proof}  [\sl Proof of Lemma \ref{lemma_concentration_ternary}]
	The function $f(z_1,\ldots,z_t) = z_1+\ldots+z_t$ is Lipschitz-continuous with Lipschitz constant $L=2$ if $-1 \leq z_i \leq 1$ for each $i.$
	It is a well-known result that the sequence of random variables $(X_i)_{i=1,\ldots,t}$ with $X_i = \mathbb{E} [ f(Z_1,\ldots,Z_t) | \mathcal{F}_i  ]$ is a martingale (the so-called Doob martingale) with bounded differences $|X_{i+1}-X_i|\leq 2L =4$ (cf.\ Lemma 11 in \citet{kocsis2006improved}).
	%			
	%		The (classical) Azuma-Hoeffding inequality is applied with 
	Consider the martingale difference sequence $\tilde X_i = X_i - \mathbb{E} X_i =  X_i - \mathbb{E} M_t$  and note that $\tilde X_t =  X_t - \mathbb{E} X_t = M_t - \mathbb{E} M_t$ and $\tilde X_0 = 0$ by setting $\mathcal{F}_0=\{\varnothing,\Omega\}.$ 
	%		as well as noting that $|\tilde X_{i+1}- \tilde X_i|<2.$	
	%		Note that $X_i= \mathbb{E} [ M_t | \mathcal{F}_i  ]$ and that by the 
	Thus, the Azuma-Hoeffding inequality implies for any $z>0$ that  
	\begin{align*}
	\P(  M_t - \mathbb{E}(M_t)>  z)  &=  \P( \tilde X_t - \tilde X_0 >  z)    \leq \exp(-z^2/(8\, t)).	
	\end{align*}	
\end{proof}

\begin{lemma} \label{lemma_bouning_inequality_ternary}
	Consider the setting of Lemma \ref{lemma_concentration_ternary} and assume that there exists $\Delta_t$ such that $\E(M_t)\leq  \Delta_t/2.$
	Then,
	$$	\P(M_t \geq \Delta_t  ) \leq \exp\big(	-\frac{\Delta_t^2}{32\,t}	\big).			$$
\end{lemma}

\begin{proof} [\sl Proof of Lemma \ref{lemma_bouning_inequality_ternary}]
	\begin{align*}
	\P(  M_t \geq \Delta_t  )  
	=   	\P(  M_t \geq \E(M_t) + \Delta_t  -  \E(M_t))  
	\leq 	\P(  M_t \geq \E(M_t) + \Delta_t/2) 
	&\leq  \exp(-\Delta_t^2/(32\, t)),
	\end{align*}
	where we used Lemma \ref{lemma_concentration_ternary} in the last step.
\end{proof}

For the proof of Lemma \ref{lemma_second_aux_lemma_CBR} we use the following variant of Lemma 13 in \citet{kocsis2006improved}.
\begin{lemma} \label{lemma_counting_processes}
	Let $N_t=\sum_{s=1}^t Z_s,$ where $(Z_s)_{s=1,\ldots,t}$ are random variables with values in $\{0,1\},$ such that $\F_s$ is the canonical filtration generated by $\{Z_1,\ldots,Z_{s-1}\}$ and $Z_{s+1}$ is conditionally independent of $Z_{s+2},\ldots,Z_{t}$ given $\F_s.$
	If $\E N_t \geq 2 \Delta_t,$ for some $\Delta_t$ then
	$$	\P(N_t \leq \Delta_t) \leq \exp\big(	- \nicefrac{\Delta_t^2}{2 \, t}	\big).			$$
\end{lemma}
\begin{proof} [\sl Proof of Lemma \ref{lemma_counting_processes}]
	By using $\E N_t \geq 2 \Delta_t,$ we have
	\begin{align*}
	\P(N_t \leq \Delta_t) = \P(N_t \leq \E N_t + \Delta_t - \E N_t)
	\leq \P(N_t \leq \E N_t - \Delta_t ) 
	\leq \exp\big(	- \nicefrac{\Delta_t^2}{2 \, t}	\big),
	\end{align*}
	where we used Lemma 12 of \citet{kocsis2006improved} for the last inequality.
\end{proof}

\section{Optimal subsets for restricted Pre-Bandits and an efficient algorithm for utility maximization } \label{subsec:algo_reward_maximization}

In this section, we show that the best arm is always element of the optimal preselection for the restricted Pre-Bandit case. 
Following this, we present a sophisticated algorithm (Algorithm \ref{algo_reward_maxi}) to avoid highly computational costs for determining the maximizing set in line 11 of Algorithm \ref{algo_explore_exploit}.

The following lemma, which can be verified by simple techniques of curve sketching, is the foundation for Algorithm \ref{algo_reward_maxi} and the proof of Lemma \ref{lemma_highest_score_in_optsubset}.
\begin{lemma} \label{lemma_help_func_reward_max}
	Let $0\leq a<b$ be real values, $(\skill_1,\ldots,\skill_\items) \in [a,b]^\items$ and $S\subseteq [\items]$ be a nonempty subset.
	Further, define $f_\gamma: [a,b] \to \R^+$ by 
	$$	f_\gamma(\skill) = f_\gamma(\skill;S) = \frac{\skill^{1+\gamma} + \sum_{ i \in S }  \skill_i^{1+\gamma} }{\skill^\gamma + \sum_{i\in S} \skill_i^{\gamma}}.			$$
	The following statements are valid.
	\begin{enumerate}
		\item [(i)] For $\tilde \skill= \nicefrac{\sum_{ i \in S }  \skill_i^{1+\gamma} }{\sum_{i\in S} \skill_i^\gamma} $ we have that $f_\gamma(\tilde \skill) = f_\gamma(0) = \tilde \skill. $
		\item [(ii)] $f_\gamma$ has a unique global minimum in $\bar \skill,$ which is the (unique) real-valued solution of the following equation in $x$
		\begin{align*}
		x^{1+\gamma} + (1+\gamma)\big(\sum_{ i \in S } v_i^\gamma		\big)x - \gamma \sum_{ i \in S }v_i^{1+\gamma} = 0.
		\end{align*}
		%		 \sqrt{ (\sum_{i\in S} \skill_i)^2 + (\sum_{i\in S} \skill_i^2)    }  - (\sum_{i\in S} \skill_i)$ and 
		It holds that $f_\gamma$ is strictly decreasing in $[a,\bar \skill]$ and strictly increasing in $[\bar \skill,b].$
		Moreover, $\bar \skill \leq \tilde \skill.$
		%		
		%		\item [(iii)]  $f(a) \leq f(b).$
		%		
		%		
	\end{enumerate}
\end{lemma}

\begin{lemma} \label{lemma_highest_score_in_optsubset}
	Let $\skill \in \paramspace$ be such that $|\argmax{i \in [\items]} \skill_i|=1$
	%	
	%	 all individual scores are different, i.e.\, $\skill_i \neq \skill_j$ for any $i\neq j.$
	%	
	and let  $J  =  \argmax{i \in [\items]} \skill_i.$ 
	Then,  for any $l\in \N$, one has $J \in \optsubset,$ where each $\optsubset$ is a maximizing subset as in (\ref{definition_optimal_subset}) for $\actionspace=\subsets$.
	Furthermore, if $|\argmax{i \in [\items]} \skill_i|>1$ then
	$\reward(  \{  J \}  ) \geq \reward (  \{J \} \cup \{ i \} )
	$
	for any $i\in[\items]$, with an equality if and only if $\skill_i = \skill_J.$
	The same holds true for $\rewardtilde.$
	%	if there exists an item $k \in [\items]$ such that $\skill_k=0,$ then 
	%	
\end{lemma}

\begin{proof} [\sl Proof of Lemma \ref{lemma_highest_score_in_optsubset}]
	%	
	%	Without loss of generality, suppose that the maximizing subset $\optsubset$ in (\ref{definition_optimal_subset}) is unique, such that $\optsubset_j=\optsubset.$
	%	
	We prove the first assertion by contradiction.
	Hence, suppose that $J \notin \optsubset.$
	Let $\tilde J \in \optsubset$ be such that $\skill_{\tilde J} < \skill_{J}$ and define $\tilde S = \optsubset \backslash \{ \tilde J\} \cup \{J\}. $
	%
	%	If one of the items, say $k$, in $\optsubset$ has a score parameter of zero, we are finished by the induction hypothesis, as
	%	
	%	Since 
	Thus, by assumption it should hold that
	\begin{align*}
	\reward(\optsubset) &=	\frac{\skill_{\tilde J}^{1+\gamma} + \sum_{i \in \optsubset \backslash \{\tilde J \}} \skill_i^{1+\gamma} }{\skill_{\tilde J}^\gamma + \sum_{i \in \optsubset \backslash \{\tilde J \}} \skill_i^\gamma }  =	\frac{\sum_{ i \in \optsubset} \skill_i^{1+\gamma}}{\sum_{ i \in \optsubset} \skill_i^\gamma} \\
	&> \frac{\sum_{ i \in \tilde S  } \skill_i^{1+\gamma}}{\sum_{ i \in \tilde S} \skill_i^\gamma } =  	\frac{\skill_{ J}^{1+\gamma} + \sum_{i \in \optsubset \backslash \{\tilde J \}} \skill_i^{1+\gamma}}{\skill_{ J}^\gamma + \sum_{i \in \optsubset \backslash \{\tilde J \}} \skill_i^\gamma }  \\ &= \reward(\tilde S).
	\end{align*}	
	In terms of Lemma \ref{lemma_help_func_reward_max} this means that $f_\gamma(\skill_{\tilde J}, \optsubset \backslash \{ \tilde J\})  > f_\gamma(\skill_{J}, \optsubset \backslash \{ \tilde J\}),$ but this is a contradiction due to (i) and (ii) of Lemma \ref{lemma_help_func_reward_max}, as $\skill_J > \tilde \skill = \frac{  \sum_{i \in \optsubset \backslash \{\tilde J \}} \skill_i^{1+\gamma}}{ \sum_{i \in \optsubset \backslash \{\tilde J \}} \skill_i^\gamma}$ and $\bar \skill \in [0,\tilde \skill].$
	%	Since
	%	%	
	%	$$	 \skill_{\tilde J}^2  + \sum_{ i \in \optsubset \backslash \{\tilde J \} } \skill_i^2   	< \skill_{ J}^2  + \sum_{ i \in \optsubset \backslash \{\tilde J\}} \skill_i^2$$ 
	%	%	
	%	and
	%	%	
	%	$$\skill_{ \tilde J}  + \sum_{ i \in \optsubset \backslash \{\tilde J\}} \skill_i \geq \skill_{  J}  + \sum_{ i \in \optsubset \backslash \{\tilde J\}} \skill_i,
	%	$$
	%	%	
	%	it follows that
	%	
	%
	%	which is a contradiction.
	%	
	The second claim follows immediately by the strict monotonic behavior of $f_\gamma$ and the claims for $\rewardtilde$ can be shown similarly.
\end{proof}

%\begin{wrapfigure}[26]{r}{0.97\textwidth}
%	\begin{minipage}{0.5\textwidth}	
\begin{algorithm}[H]
	\begin{algorithmic}[1]
		\caption{Utility-maximization} \label{algo_reward_maxi}
		\INPUT { $\items$ many paramters $\skill_1,\ldots,\skill_\items,$ preciseness parameter $\gamma,$ preselection size $\pullnumber$\;}
		\STATE {\textbf{initialization:} $\tau \leftarrow Sort(\skill_1,\ldots,\skill_\items)$ \COMMENT{determine permutation which sorts the scores in decreasing order}\;}
		\STATE $S \leftarrow  \argmax{ i \in \tau([\items] ) } \skill_{\tau(i)} $ \COMMENT{select all high-score items}\;

		\IF{ $|S|\geq \pullnumber$ }
		\STATE {  \textbf{return:} {  randomly selected $\pullnumber$ elements of $S$\; } }
		\ELSE 	
		\STATE		$A \leftarrow [\items] \backslash [|S|] $ \COMMENT{ set of active arms }\;
		\REPEAT 
		
		\STATE $\tilde \skill \leftarrow  \nicefrac{\sum_{ i \in S }  \skill_i^{1+\gamma} }{\sum_{i\in S} \skill_i^\gamma }$\;
		
		%		$A_{min} \leftarrow \min A, \ A_{max} \leftarrow \max A$
		
		\STATE $A_{next} \leftarrow \argmin{ i \in \{ \min A , \max A \} } \{ |\tilde \skill - f_\gamma(\skill_{\tau(i)};S)|     \}$ \\
		\COMMENT{$f_\gamma$ as in Lemma \ref{lemma_help_func_reward_max}, break ties arbitrarily}\;
		
		%		$S_{next}  \leftarrow  \argmax{ i\in \{ \tau(A_{min}), \tau(A_{max}) \} } \{ |\tilde \skill - f(\skill_i;S)|     \}  $ \tcc*[f]{$f$ as in Lemma \ref{lemma_help_func_reward_max}}\;

		\STATE	$S \leftarrow S \cup \tau(A_{next}) $\;
		
		%		$\skill \leftarrow remove(S_{next}, \skill)$  \tcc*[f]{remove $\skill_i \ i \in S_{next}$ from $\skill$ }\;

		\STATE $A \leftarrow A \backslash A_{next}$\;
		\UNTIL{$|S| == \pullnumber$}
		
		\STATE { \textbf{return:}   $S$\; }
		
		\ENDIF
	\end{algorithmic}
\end{algorithm}
%
%	\end{minipage}
%\end{wrapfigure}

%
Let $\skill_{(i)}$ denote the $i$-th order statistic for $(\skill_1,\ldots,\skill_\items),$ i.e., 
$$ \skill_{(1)} \leq \skill_{(2)} \leq \ldots \leq  \skill_{(n)},  $$
then Lemma \ref{lemma_help_func_reward_max} implies that  $f_\gamma\big(v;\{\skill_{(n)}\}\big)\leq f_\gamma\big(\skill_{(n)};\{\skill_{(n)}\}\big) $ for any $v \in [0,\skill_{(n)}]$ and the smallest decrease of $f_\gamma\big(\cdot;\{\skill_{(n)}\} \big)$ over the discrete set $\{\skill_{(1)},\ldots, \skill_{(\items-1)}   \}$ is either for $\skill_{(\items-1)}$ or for $\skill_{(1)}.$

% 
%Note that $f$ describes the behavior of the expected utility
%
With this, Algorithm \ref{algo_reward_maxi} successively builds a set $S$ which will maximize the expected utility in (\ref{defi_expected_utility}) for a given score parameter $\skill=(\skill_1,\ldots,\skill_\items)$.
First, the scores are sorted in order to find the arms with the highest scores, as by Lemma \ref{lemma_highest_score_in_optsubset} these are always element of the maximizing subset.
If more than $(\pullnumber-1)$ elements have the same highest score,  a randomly chosen  $\pullnumber$-sized set of these is returned, since the expected utility among all possible $\pullnumber$-sized subsets of these is the same by Lemma \ref{lemma_help_func_reward_max} or Lemma \ref{lemma_highest_score_in_optsubset}. 

Otherwise, an active index set $A$ is initialized containing all indices for which it is not decided yet, if they are part of the maximizing set $S$ eventually.
As by Lemma \ref{lemma_highest_score_in_optsubset} the expected utility decreases from that point on by enlarging the set $S,$ the algorithm determines the arm with the smallest decrease for the expected utility, where ties are broken arbitrary by two possible candidates.

Since the expected utility of the currently set $S$ is identical to $f_\gamma(0;S)$ only the arms with the smallest resp.\ highest score parameter in $A$ have to be checked by the implication after Lemma \ref{lemma_highest_score_in_optsubset}.
It can be shown that the algorithm has worst complexity of $O( \pullnumber \, \items \, \log(\items))$ if an efficient sorting algorithm is used in the initial step.

\section{Further experiments for the Pre-Bandit problem}\label{sec:further_exp}

In this section, we provide further experiments on synthetic data for the two variants of the Pre-Bandit problem.

\begin{figure} 
	%	[h]
	\centering
	\subfigure
	{\includegraphics[width=0.48\linewidth]{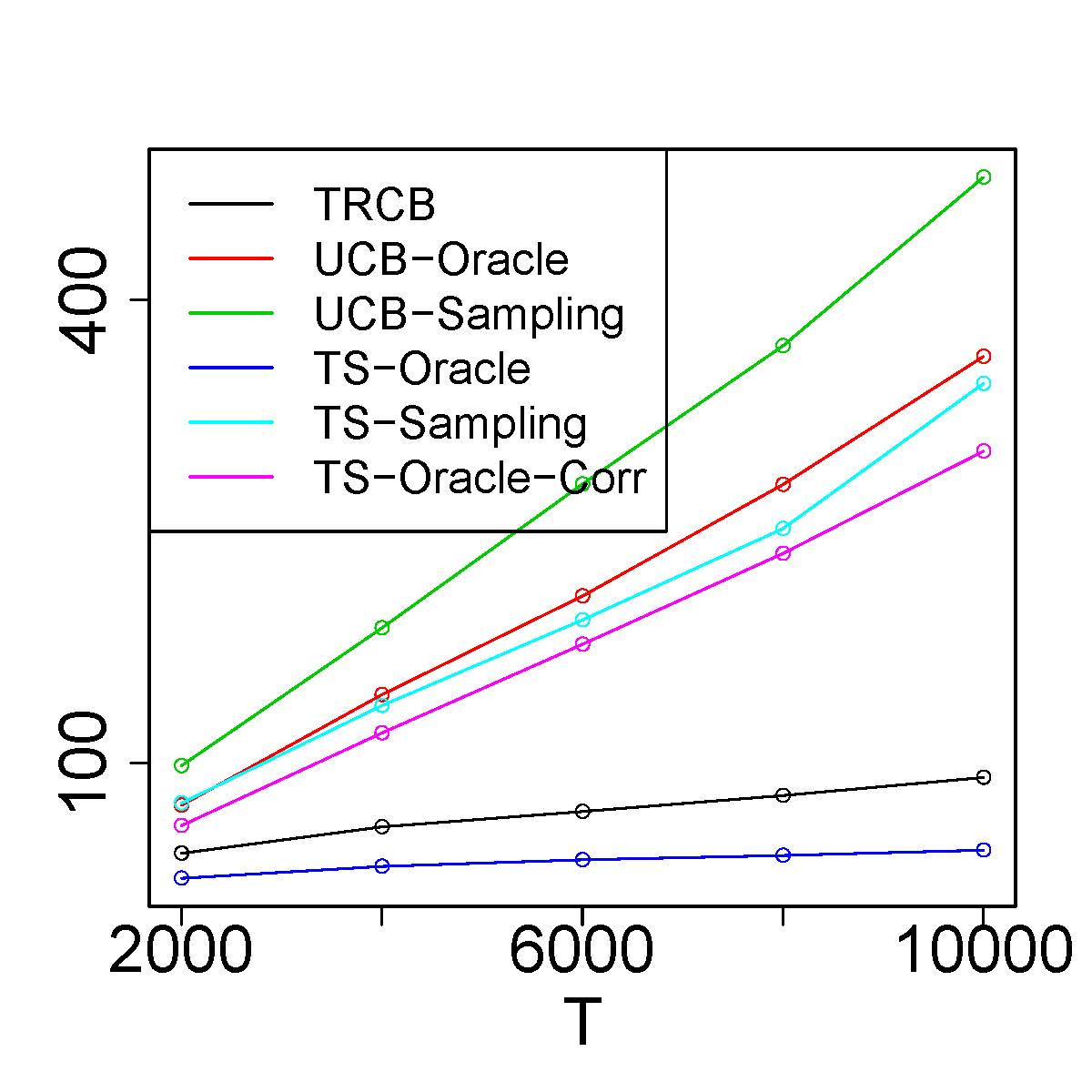}}
	\subfigure
	{\includegraphics[width=0.48\linewidth]{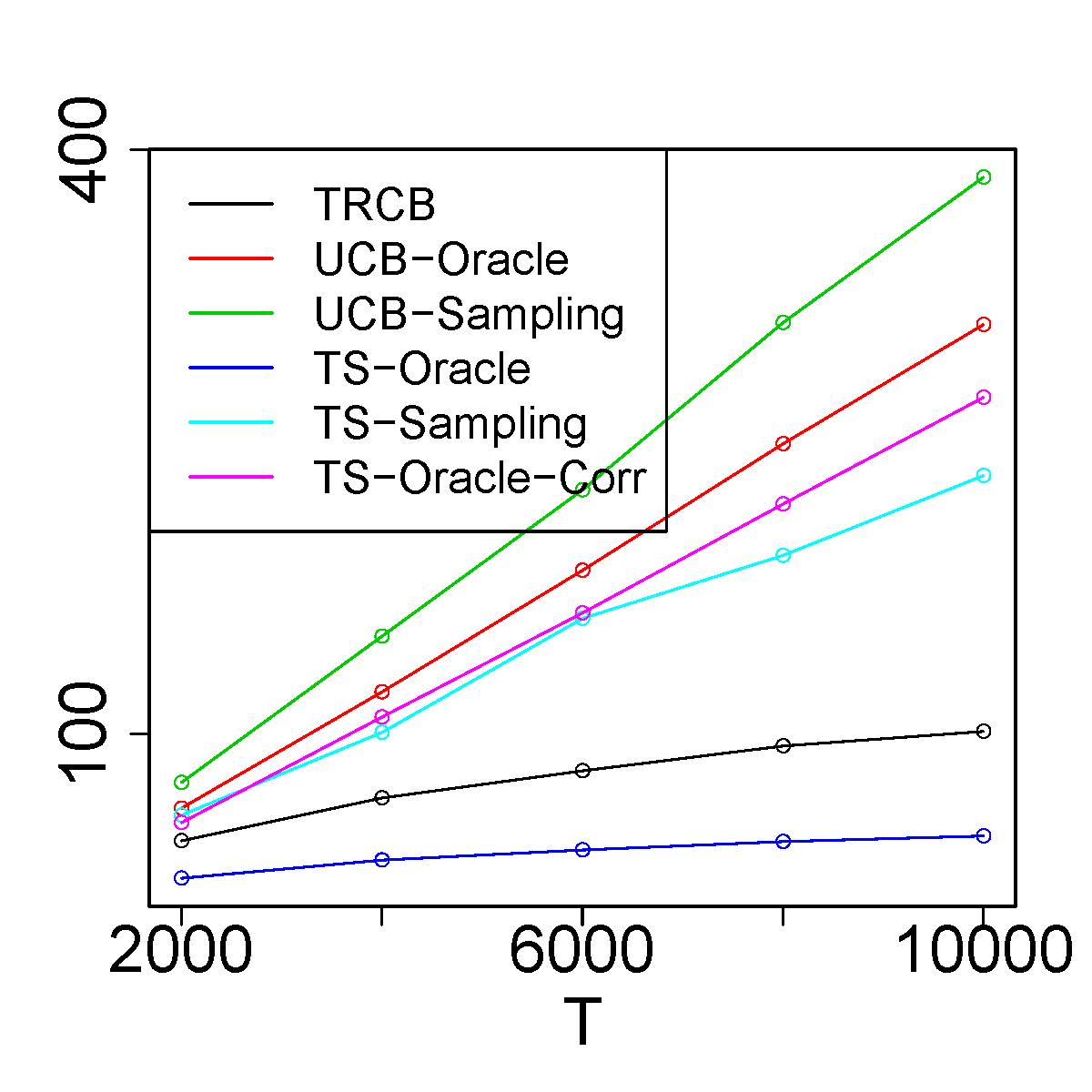}}
	\caption{Mean cumulative regret for 1000 runs of randomly generated restricted PB instances for $(\items,\pullnumber)=(20,4)$ (left) and $(\items,\pullnumber)=(30,5)$ (right).}
	\label{fig:regre_other_scenarios_restricted_pre}
\end{figure} 

\paragraph{Restricted Pre-Bandit problem (larger number of arms)} First, we present two additional scenarios of the simulation study in Section \ref{sec_experiments} for the restricted Pre-Bandit problem with larger numbers of arms $n$ and different preselection sizes $\pullnumber.$
In particular, we investigate  the performance of the following algorithms, which were also analyzed in Section \ref{sec_experiments}, for the restricted Pre-Bandit problem:
\begin{itemize}
	\item TRCB: The TRCB algorithm in Algorithm \ref{algo_explore_exploit} with $\Cshrink = 7 \cdot 10^{-5} $ and $\skillmin=0.02$ (here as a parameter of the algorithm).
	\item UCB-Oracle:	UCB-type algorithm of \citet{agrawal2016near} with knowledge of the best arm in advance and revenues are set to be the estimated score parameters (in short $r = \hat \skill$).
	\item UCB-Sampling:	UCB-type algorithm of \citet{agrawal2016near} without knowledge of the best arm in advance (sampled with MNL probability among the three best) and $r = \hat \skill$.
	\item TS-Oracle:	The Thompson sampling algorithm of \citet{agrawal2017thompson} (Algorithm 1)  with knowledge of the best arm in advance and $r = \hat \skill$.
	\item TS-Sampling:	The Thompson sampling algorithm of \citet{agrawal2017thompson} (Algorithm 1)  without knowledge of the best arm in advance (sampled with MNL probability among the three best) and $r = \hat \skill$.
	\item TS-Oracle-Corr: Correlated Thompson sampling algorithm of \citet{agrawal2017thompson} (Algorithm 2) with knowledge of the best arm in advance and $r = \hat \skill$.

\end{itemize}
The left picture in Figure \ref{fig:regre_other_scenarios_restricted_pre} provides the findings for the case $\items=20$ and $\pullnumber=4,$ while the right picture illustrates our results for $\items=30$ and $\pullnumber=5.$ Both scenarios are considered for the time horizons $\timehorizon\in \{ i \cdot 2000\}_{i=1}^{5}$ and the score parameters are drawn randomly from 
%the unit interval
the $\items$-simplex  without any restrictions on $\skillmin$ and with $\gamma=1.$

\begin{figure} 
	%	[h]
	\centering
	\subfigure
	{\includegraphics[width=0.48\linewidth]{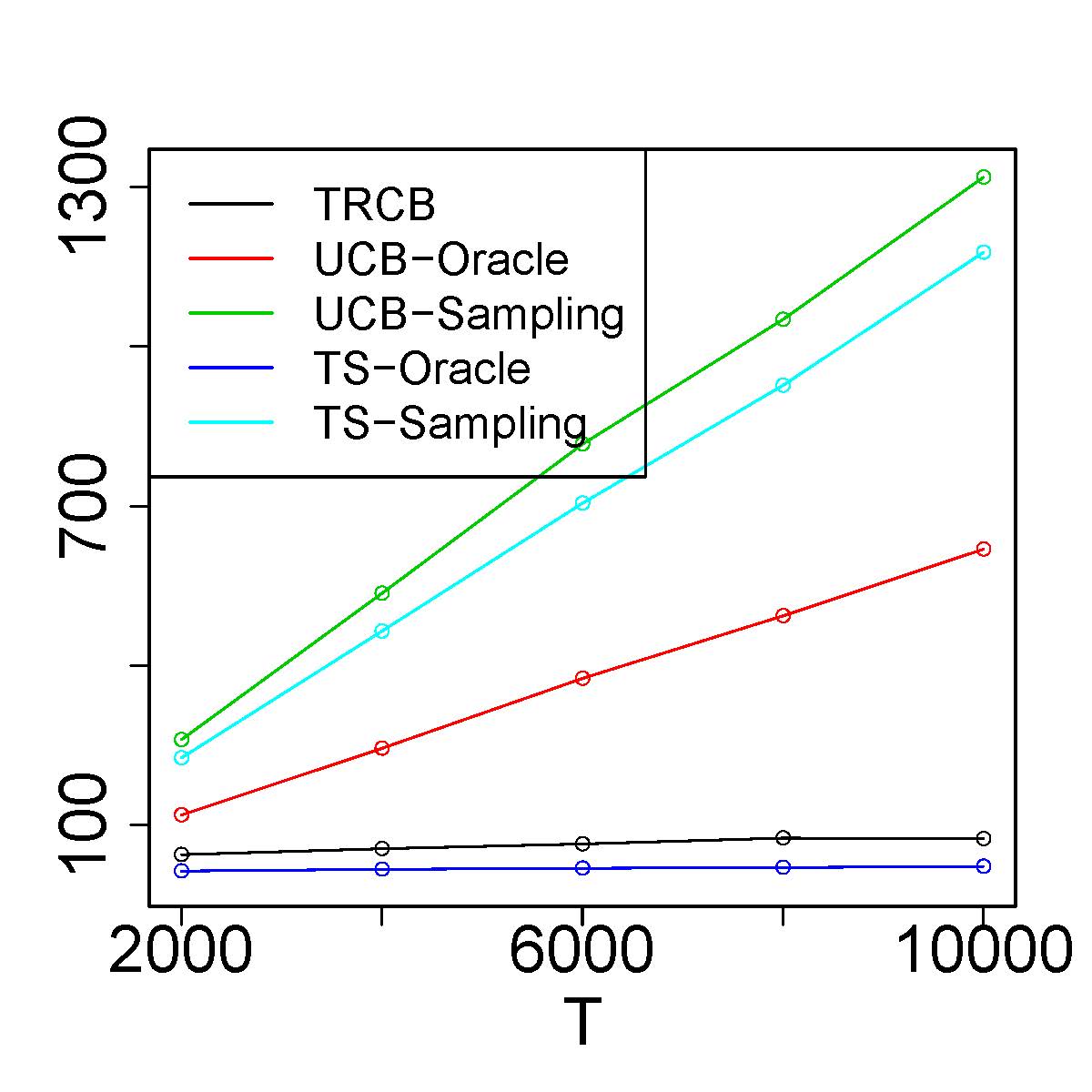}}
	\subfigure
	{\includegraphics[width=0.48\linewidth]{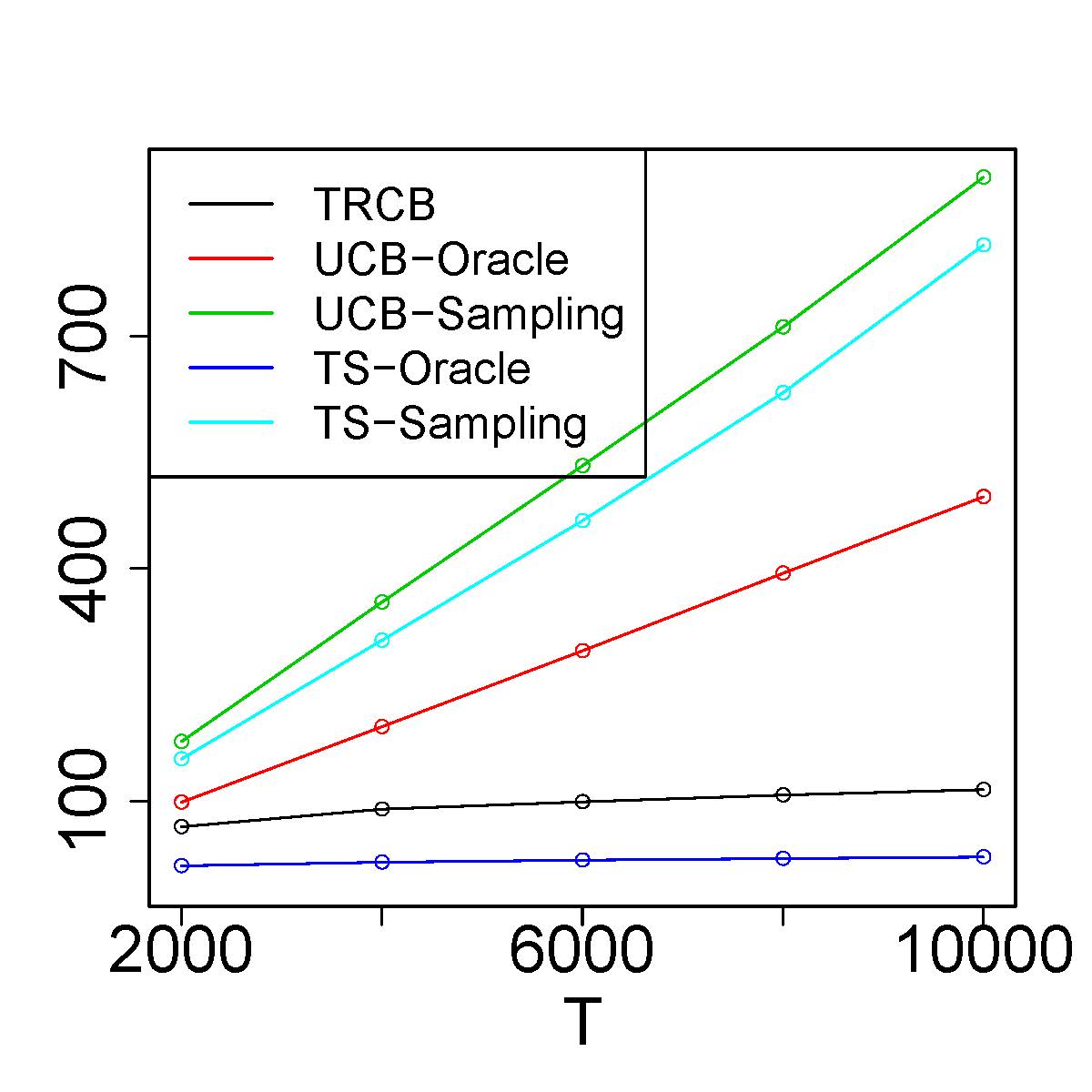}}
	\caption{Mean cumulative regret for 1000 runs of randomly generated restricted PB instances for $(\items,\pullnumber)=(10,3)$ (left) and $(\items,\pullnumber)=(20,4)$ (right) and $\gamma=1/20.$}
	\label{fig:regre_other_scenarios_restricted_pre_small_gamma}
\end{figure}

%\begin{table}[h]
%	\centering
%	\caption{Empirical standard deviations of the cumulative regret for the different time horizon steps for the scenarios $(\items,\pullnumber)=(20,4)$ and $(\items,\pullnumber)=(30,5).$} 
%	\label{table_std_dev_restricted_pre_bandits}
%	\resizebox{0.9\textwidth}{!}{
%		\begin{tabular}{rlllll||lllll}
%			%	\hline 
%			& \multicolumn{5}{|c||}{ $(\items,\pullnumber)=(20,4)$} & \multicolumn{5}{c}{ $(\items,\pullnumber)=(30,5)$} \\
%			\hline	
%			%	& \multicolumn{5}{c||}{ $\timehorizon$} & \multicolumn{5}{c}{ $\timehorizon$} \\
%			\multicolumn{1}{c|}{ 	$\timehorizon$}   & 2000 & 4000 & 6000 & 8000 & 10000 & 2000 & 4000 & 6000 & 8000 & 10000 \\ 
%			\hline
%			\multicolumn{1}{c|}{ TRCB}   & 21.47 & 32.93 & 43.73 & 57.51 & 66.36 & 20.54 & 34.81 & 40.84 & 54.91 & 58.88 \\
%			\multicolumn{1}{c|}{ UCB-Oracle}   & 43.90 & 79.27 & 119.19 & 165.42 & 202.54  & 34.45 & 65.40 & 102.52 & 143.31 & 172.84 \\ 
%			\multicolumn{1}{c|}{ UCB-Sampling}   & 98.31 & 187.52 & 280.59 & 370.10 & 479.20   & 75.21 & 150.24 & 225.29 & 311.10 & 385.72 \\ 
%			\multicolumn{1}{c|}{ TS-Oracle}   & 7.74 & 10.01 & 11.37 & 13.42 & 14.06 & 6.91 & 9.48 & 11.17 & 13.43 & 13.91 \\ 
%			\multicolumn{1}{c|}{ TS-Sampling}   & 86.11 & 161.01 & 235.16 & 329.54 & 429.52 & 53.66 & 101.35 & 175.86 & 246.12 & 284.47 \\ 
%			\multicolumn{1}{c|}{ TS-Oracle-Corr}   & 21.84 & 43.65 & 63.73 & 87.99 & 111.97 & 18.01 & 38.74 & 55.92 & 80.59 & 97.91 \\  
%			\hline
%	\end{tabular}}
%\end{table}

\begin{table}[h]
	\centering
	\caption{Empirical standard deviations of the cumulative regret for the different time horizon steps for the scenarios $(\items,\pullnumber)=(20,4)$ and $(\items,\pullnumber)=(30,5).$} 
	\label{table_std_dev_restricted_pre_bandits}
	\resizebox{0.47\textwidth}{!}{
		\begin{tabular}{rlllll}
			%	\hline 
			& \multicolumn{5}{|c}{ $(\items,\pullnumber)=(20,4)$}  \\
			\hline	
			%	& \multicolumn{5}{c||}{ $\timehorizon$} & \multicolumn{5}{c}{ $\timehorizon$} \\
			\multicolumn{1}{c|}{ 	$\timehorizon$}   & 2000 & 4000 & 6000 & 8000 & 10000 \\ 
			\hline
			\multicolumn{1}{c|}{ TRCB}   & 21.47 & 32.93 & 43.73 & 57.51 & 66.36 \\
			%			 20.54 & 34.81 & 40.84 & 54.91 & 58.88 \\
			\multicolumn{1}{c|}{ UCB-Oracle}   & 43.90 & 79.27 & 119.19 & 165.42 & 202.54  \\
			%			 34.45 & 65.40 & 102.52 & 143.31 & 172.84 \\ 
			\multicolumn{1}{c|}{ UCB-Sampling}   & 98.31 & 187.52 & 280.59 & 370.10 & 479.20 \\
			%			  & 75.21 & 150.24 & 225.29 & 311.10 & 385.72 \\ 
			\multicolumn{1}{c|}{ TS-Oracle}   & 7.74 & 10.01 & 11.37 & 13.42 & 14.06 \\
			%			 6.91 & 9.48 & 11.17 & 13.43 & 13.91 \\ 
			\multicolumn{1}{c|}{ TS-Sampling}   & 86.11 & 161.01 & 235.16 & 329.54 & 429.52 \\
			%			 53.66 & 101.35 & 175.86 & 246.12 & 284.47 \\ 
			\multicolumn{1}{c|}{ TS-Oracle-Corr}   & 21.84 & 43.65 & 63.73 & 87.99 & 111.97 \\
			%			18.01 & 38.74 & 55.92 & 80.59 & 97.91 \\  
			\hline
			\hline 			
			& \multicolumn{5}{|c}{ $(\items,\pullnumber)=(30,5)$} \\			\hline 		
			%			\multicolumn{1}{c|}{ 	$\timehorizon$}   & 2000 & 4000 & 6000 & 8000 & 10000 \\ 
			\multicolumn{1}{c|}{ TRCB}   &
			20.54 & 34.81 & 40.84 & 54.91 & 58.88 \\
			\multicolumn{1}{c|}{ UCB-Oracle}  &
			%			 & 43.90 & 79.27 & 119.19 & 165.42 & 202.54  \\
			34.45 & 65.40 & 102.52 & 143.31 & 172.84 \\ 
			\multicolumn{1}{c|}{ UCB-Sampling} &
			%			  & 98.31 & 187.52 & 280.59 & 370.10 & 479.20 \\
			75.21 & 150.24 & 225.29 & 311.10 & 385.72 \\ 
			\multicolumn{1}{c|}{ TS-Oracle}   &
			%			& 7.74 & 10.01 & 11.37 & 13.42 & 14.06 \\
			6.91 & 9.48 & 11.17 & 13.43 & 13.91 \\ 
			\multicolumn{1}{c|}{ TS-Sampling}  &
			%			 & 86.11 & 161.01 & 235.16 & 329.54 & 429.52 \\
			53.66 & 101.35 & 175.86 & 246.12 & 284.47 \\ 
			\multicolumn{1}{c|}{ TS-Oracle-Corr}   &
			%			& 21.84 & 43.65 & 63.73 & 87.99 & 111.97 \\
			18.01 & 38.74 & 55.92 & 80.59 & 97.91 \\  
	\end{tabular}}
\end{table}

The findings are similarly as for the case $\items=10$ and $\pullnumber=3,$ that is only the Thompson Sampling algorithm with knowledge of the best arm apriori (TS-Oracle) outperforms TRCB, while the other algorithms are outperformed by TRCB.
Furthermore, we report the empirical standard deviations of the considered algorithms for each time horizon in both scenarios in Table \ref{table_std_dev_restricted_pre_bandits}. 
Only TS-Oracle has a throughout smaller standard deviation than TRCB, while all the others have variations of a higher magnitude than TRCB. 

\begin{figure}
		 [H]
	\centering
	\subfigure
	{\includegraphics[width=0.48\linewidth]{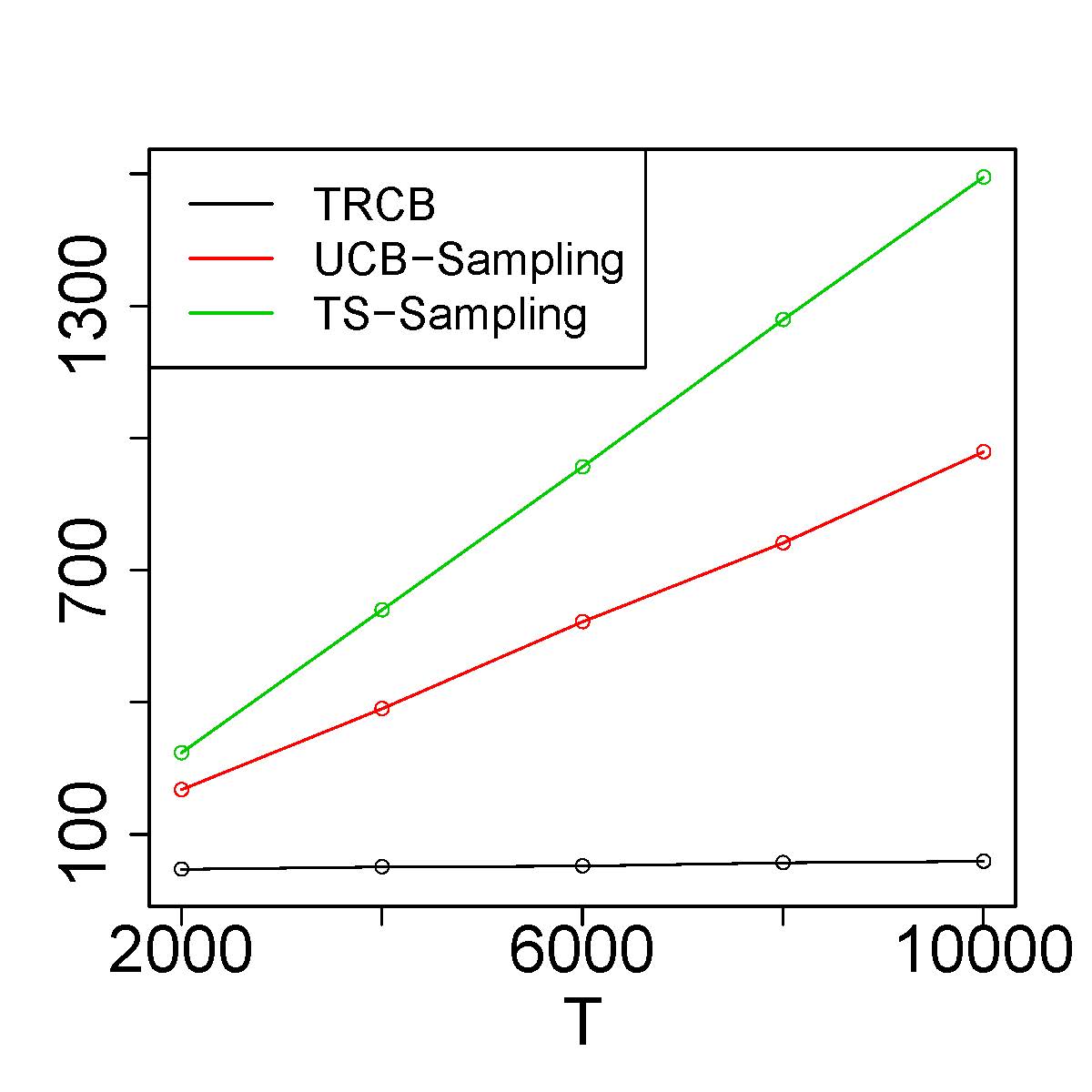}}
	\subfigure
	{\includegraphics[width=0.48\linewidth]{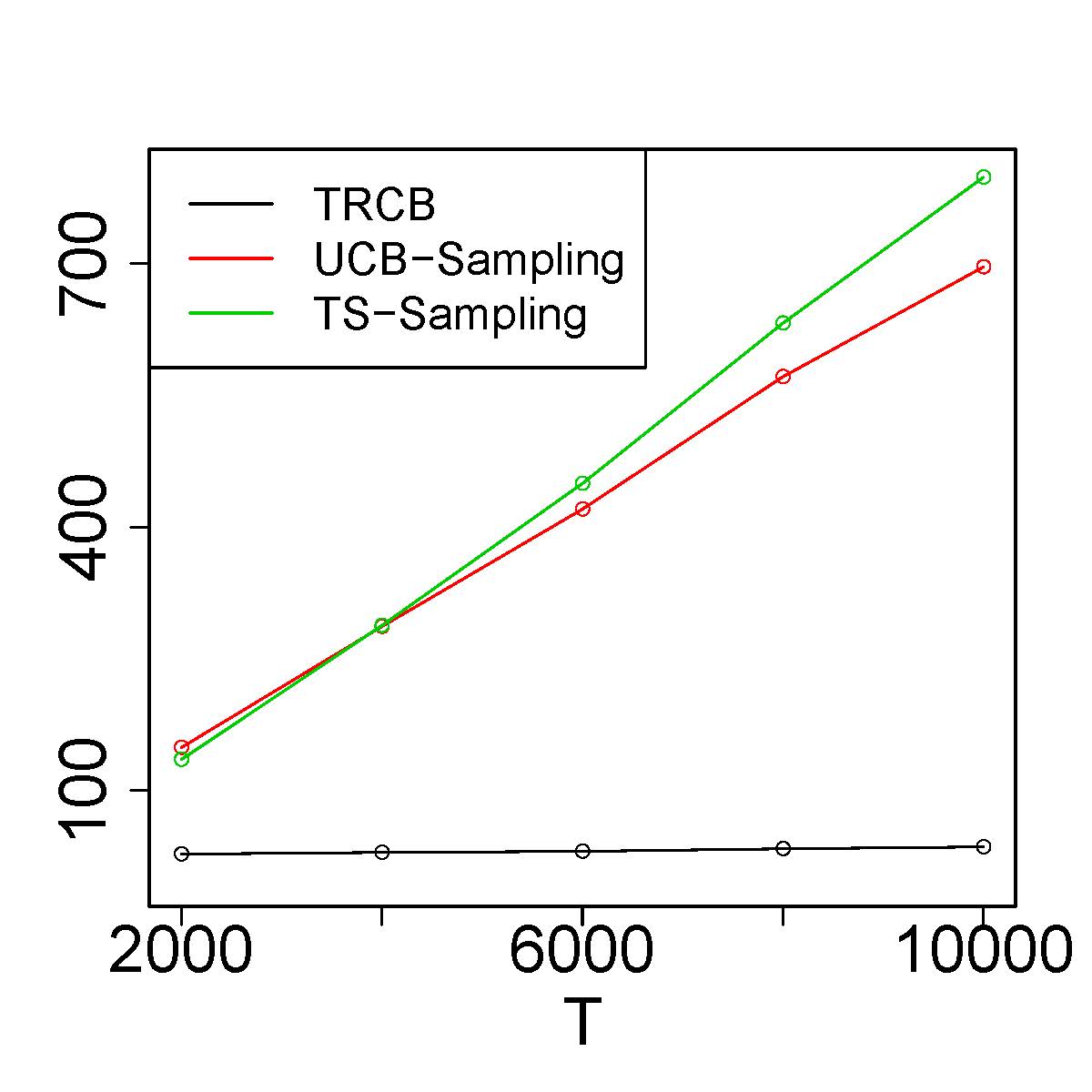}}
	\caption{Mean cumulative regret for 1000 runs of randomly generated restricted PB instances for $(\items,\pullnumber)=(10,3)$ (left) and $(\items,\pullnumber)=(20,4)$ (right) and $\gamma=20.$}
	\label{fig:regre_other_scenarios_restricted_pre_large_gamma}
\end{figure}

\paragraph{Restricted Pre-Bandit problem (Varying degree of preciseness)} Next, we consider two additional scenarios, in which we initially set $\gamma = 1/20$ such that  the most preferred subsets consists throughout of the top-$\pullnumber$ arms and 
The results for $\gamma = 1/20$  are depicted in Figure \ref{fig:regre_other_scenarios_restricted_pre_small_gamma} for the cases $(\items,\pullnumber)=(10,3)$ and  $(\items,\pullnumber)=(20,4)$ for the algorithms described above.
Note that TS-Oracle-Corr could not be compared as it sampled negative values for the score parameters, which lead to numerical issues regarding the evaluation of the utility function. 
Again the findings are in line with the observations we have made in the simulations before, i.e., only TS-Oracle is able to outperform our algorithm TRCB due to its advantage of knowing the best arm. 
In particular, this demonstrates that our algorithm performs well for scenarios where top-$\pullnumber$ subsets are the desired outcome for a user.

In addition, we consider the case $\gamma = 20$ such that the most preferred subsets are basically all subsets which contain the arm with the highest score (cf.\ Example 2 in the main paper).
Figure \ref{fig:regre_other_scenarios_restricted_pre_large_gamma} illustrates the results for the cases $(\items,\pullnumber)=(10,3)$ and  $(\items,\pullnumber)=(20,4),$ where we do not included the algorithms which have prior knowledge of the best arm, as these naturally have throughout a regret of zero.
This experiment indicates that the considered DAS algorithms depend too much on the assumption that the no-choice option corresponds to the highest scored arm as also remarked in Section \ref{sec_experiments}.
%that our algorithm performs well for scenarios where top-$\pullnumber$ subsets are the desired outcome for a user.
% 
%

%
%Again the findings are in line with the observations we have made in the simulations before, i.e., only TS-Oracle is able to outperform our algorithm TRCB due to its advantage of knowing the best arm. 
%

%
\paragraph{Flexible Pre-Bandit problem} In addition to the simulations in Section \ref{sec_experiments}, we investigate the empirical regret growth over time for larger  numbers of arms $\items$ for our CBR algorithm for the flexible Pre-Bandit problem.
We consider two variants of the CBR-algorithm:
\begin{itemize}
	\item CBR: The CBR algorithm with 
		$\sigma(x)= (1 \wedge x) 1_{[0,\infty)}(x).$

	\item CBR-As: The CBR algorithm with
		$\sigma(x) = \frac{1}{\pi} \arctan\Big( \frac{x-\nicefrac{1}{2}}{(1-x)^\rho x^\rho}  \Big)+ \frac12 $ and $\rho=2.$
\end{itemize}
Figure \ref{fig:regretanalysisracingalgo} illustrates the results of our simulations for both CBR algorithm variants over 500 repetitions, respectively, with $\items \in \{60,120,240\},$ over the time horizons $\timehorizon \in \{ i \cdot 2000  \}_{i=1}^5$ and the score parameters were drawn randomly from the unit interval and with $\gamma=1.$ 

\begin{figure}  
	%	[H]
	\centering
	
	{\includegraphics[width=0.94\linewidth]{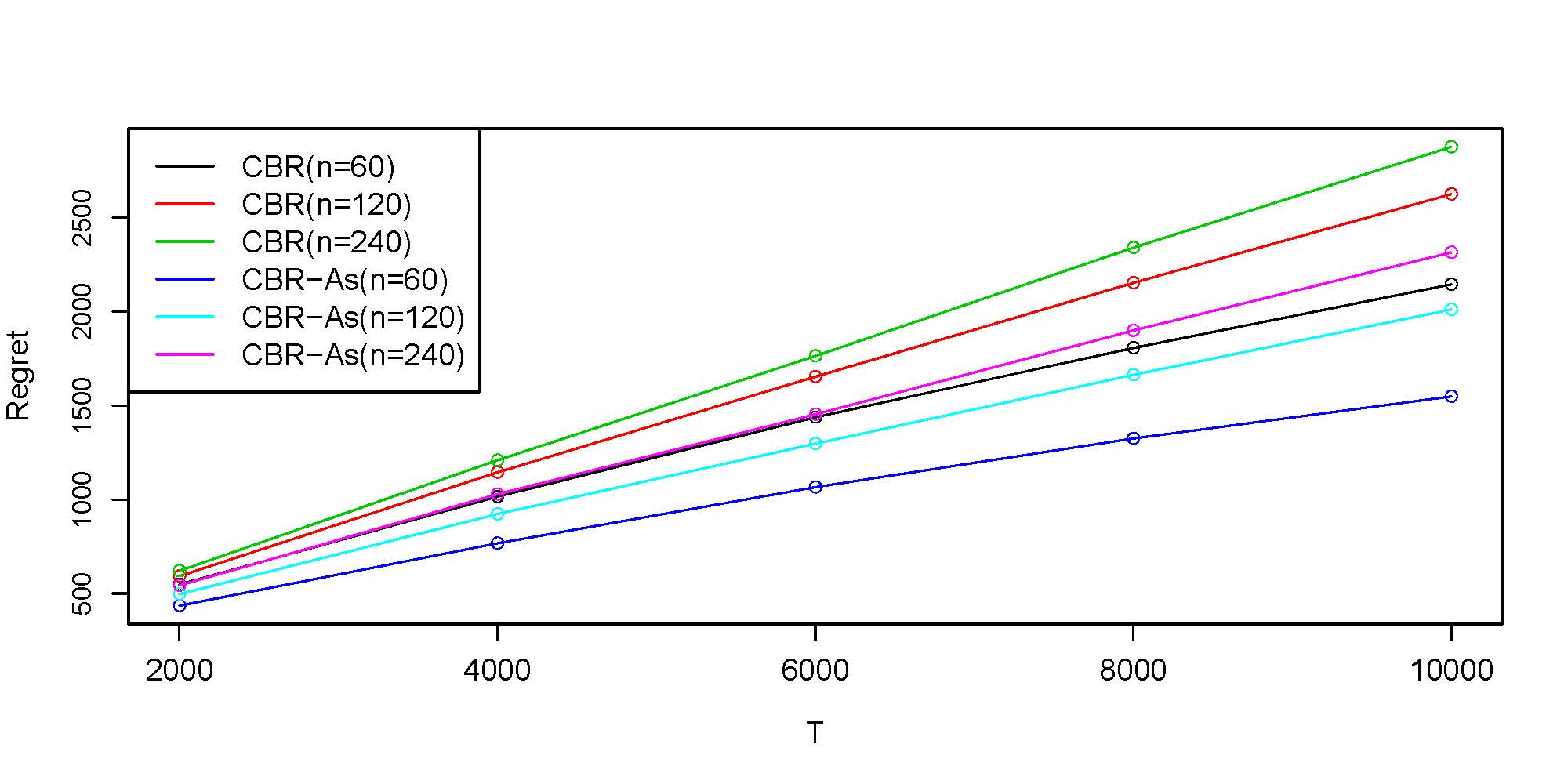}}
	\caption{
		%		Left: S-Curved function of CBR-As. Right: 
		Mean cumulative regret of the variants of the CBR algorithm for 500 runs of randomly generated flexible Pre-Bandit instances for $\items \in \{60,120,240\}.$  }
	\label{fig:regretanalysisracingalgo}
\end{figure}

It is clearly visible that  CBR-As outperforms CBR due to the more sophisticated choice of the S-curved function $\sigma.$ 
Thus, it is reasonable to believe that the performance of CBR can be significantly improved by an appropriate choice of $\sigma.$
Note that the Double Thompson Sampling considered in Section \ref{sec_experiments} was not competitive in these scenarios and is therefore omitted.

\end{document}